\documentclass{article}

\usepackage[final,main,nonatbib]{neurips_2026}
\usepackage[utf8]{inputenc} 
\usepackage[T1]{fontenc}    

\usepackage{graphicx} 
\usepackage{calc} 
\usepackage{enumitem} 
\usepackage{amsmath,amssymb,amsfonts,amsthm,bm}
\usepackage{mathrsfs}
\usepackage{hyperref}
\usepackage{xcolor}
\usepackage{csquotes}
\usepackage{tikz}
\usetikzlibrary{positioning, arrows.meta, calc}
\usepackage{subcaption}
\usepackage{multirow}

\newtheorem{theorem}{Theorem}
\newtheorem{lemma}{Lemma}
\newtheorem{corollary}{Corollary}
\newtheorem{definition}{Definition}
\newtheorem{proposition}{Proposition}

\newtheorem{assumption}{Assumption}

\usepackage[numbers, sort&compress]{natbib}

\usepackage{longtable,booktabs,array}
\newcommand{\x}{\mathbf{x}}
\newcommand{\z}{\mathbf{z}}
\newcommand{\y}{\mathbf{y}}
\newcommand{\balpha}{\bm{\alpha}}
\newcommand{\h}{\mathbf{h}}

\newcommand{\var}{\mathrm{Var}}
\newcommand{\cov}{\mathrm{Cov}}
\newcommand{\rk}{\mathrm{rank}}

\newcommand{\ind}{\perp\!\!\!\perp}

\newcommand{\ATE}{\mathrm{ATE}}
\newcommand{\GTE}{\mathrm{GTE}}
\newcommand{\MPLE}{\mathrm{MPLE}}

\title{Causal Inference for Sequential Settings under Interference and Latent Confounding}
\author{%
  Phevos Paschalidis \\
  EECS, LIDS, CSAIL\\
  MIT\\
  \texttt{phevosp@mit.edu} \\
  \And
  \textbf{Constantinos Daskalakis} \\
  EECS, CSAIL\\
  MIT\\
  \texttt{costis@mit.edu} \\
  \And
  \textbf{Devavrat Shah} \\
  EECS, IDSS, LIDS, SDSC\\
  MIT\\
  \texttt{devavrat@mit.edu} \\
}
\begin{document}

\maketitle

\begin{abstract}

    We study causal inference under outcome interference for sequential, observational settings. 
    Specifically, we consider settings where the binary outcomes over $N$ units are Markovian across $T$ time steps.
    At each time step, the outcomes of $N$ units have dependencies captured through an Ising model; 
    each outcome is also impacted through an external field capturing the effects of its treatment as well as latent confounders. 
    Similar to panel data literature, these latent confounders are modeled to have a low-rank factor
    structure. 
    Our data is 
    a single sample from this high-dimensional distribution. 
    To estimate causal quantities of interest, we provide a computationally efficient method based on Maximum Pseudo-Likelihood Estimation (MPLE)
    for learning the model parameters. Under mild assumptions, we establish non-asymptotic consistency for parameter estimation
    and show this translates to faithful estimation of causal quantities of interest after sampling from the learned model.
    We demonstrate the efficacy of the method through synthetic experiments as well as a real-world case-study 
    investigating causal effects of vaccine rates on COVID-19 death rates within US counties nationwide.
\end{abstract}

\section{Introduction}

A key assumption underpinning much of the rich literature on causal inference is the Stable Unit Treatment Value Assumption (SUTVA), which states that the outcome of a unit is causally dependent only on its own features and treatment \cite{Imbens_2015,ding2024first}. For many applications, however,
there is interference naturally present leading to a violation of SUTVA.

As a motivating example, consider a population of $N$ individuals, some of whom receive a vaccine. Under the SUTVA assumption, each individual $i$'s binary outcome $x_i$ (whether or not they get sick) is dependent only on individual characteristics $\alpha_i$ (affecting their predisposition to getting sick) and their binary treatment $z_i$ (whether they receive the vaccine). It is therefore possible to estimate the population level causal effect of the vaccine by averaging the independent outcomes of different individuals who did or did not receive it. In reality, however, each individual's likelihood of contracting the disease is also dependent on its prevalence in the population they come in contact with. 

A growing body of literature has begun exploring causal frameworks for settings that violate SUTVA, for example  \cite{Hudgens_2008,Aronow_2017,Savje_2021,tchetgen2021auto,bhattacharya2025causal,Cattaneo_2025,Kandiros_2025,Agarwal_2023}. A well-established approach is to assume that each individual's outcome is affected by its own and others' treatments according to an underlying network \cite{Jagadeesan_2020,Ugander_2023,Kandiros_2025,Cattaneo_2025, Agarwal_2023}. 
That is, $x_i = f(\alpha_i,z_i,\z_{N(i)})$ where $N(i)$ are the neighbors of $i$ and $f$ is an unknown potential outcome function (which might be randomized). 

An arguably more complete model of interference is that of outcome interference, which allows $x_i = f(\alpha_i, z_i,\balpha_{-i},\z_{-i},\x_{-i})$, 
with the additional assumption that the joint $\x$ satisfies certain conditional independence properties capturing e.g.~network dependencies~\cite{tchetgen2021auto,bhattacharya2025causal}: in the example discussed earlier, 
each individual's likelihood of getting sick is directly dependent on their own predisposition, their own vaccination status, 
and the realized outcomes of their friends and family, and therefore dependent on the predispositions, vaccinations, 
and outcomes of the entire population.

Existing work in outcome interference considers non-sequential settings \cite{tchetgen2021auto,bhattacharya2025causal}, whereas it is also useful to consider time dependence since, for example, individuals may receive the vaccine at different times and get sick at different times. In this case, we will have dependence between adjacent time steps as well. 
Furthermore, outcomes and interventions are impacted by confounders in observational settings, which may be latent.
This is also unexplored in existing work.
%
In summary, we have a setting where outcomes across multiple units are observed in a sequential setting, interfere with each other, and
are impacted by interventions that are confounded by potentially latent variables. 


\textbf{Contributions.} We introduce a model to capture the setting of sequential outcomes over multiple units under interference and latent confounding. The latent confounding across time-steps and units is modeled to have low-rank factor structure similar to the panel data 
literature, cf. \cite{amjad_2018,agarwal2026synthetic,abadie2003economic,abadie2010synthetic,abadie2024doubly,agarwal2023causal}. Under our model, the latent confounding
manifests through the so-called unit and time specific external field parameters associated with the outcome variable. 

We provide a computationally efficient method based on Maximum Pseudo-Likelihood Estimation (MPLE) to learn the model parameters
from a {\em single} observation of data, where recall that an observation is a sample from a distribution of outcomes for all units and time steps with intricate dependencies.
This, subsequently, enables estimation of causal quantities of interest via sampling from the underlying distribution. 
For example, we can estimate the global average treatment effect which compares the average outcome over all nodes and times under complete intervention versus no intervention; more generally, we can compare the average outcomes 
under any two choices of intervention.
Synthetic data experiments and a real-world case study investigating the effect of vaccination on COVID-19 death rates
of US counties establish the practical utility of the method. We also compare to a state-of-the-art approach that does not model interference \cite{agarwal2023causal}.

From the perspective of learning high-dimensional distributions from a single sample, our work extends a growing body of work~\cite{chatterjee2007estimation,bhattacharya2018inference,daskalakis2019regression,ghosal2020joint,dagan2021learning,Kandiros_2021} to the sequential setting with latent variables. From the perspective of causal inference, this work provides provable guarantees for causal estimation in a reasonably generic setting where one is given access to observational data with network interference, intertemporal dependencies, and latent confounding. While there is a rapidly growing literature on causal inference under network interference or sequential observation with latent confounding, there do not exist frameworks that offer guarantees when all complexities are co-present. 


\subsection{Related Works} \label{subsec:related_works}

\textbf{Causal Inference under Outcome Interference.} Following a surge of interest in causal inference under interference \cite{Hudgens_2008,Manski_2013,vanderLaan_2014,Aronow_2017}, Tchetgen et al. proposed a model for outcome interference by modeling outcomes 
as a single draw from a Markov random field (MRF) which encodes conditional independence properties of 
the joint distribution of outcomes \cite{tchetgen2021auto}. The authors propose MPLE as a solution approach but do not provide guarantees on parameter estimation. 
Outcome interference was subsequently explored in \cite{sherman2018identification,bhattacharya2020causal}. The most relevant paper to our work is \cite{bhattacharya2025causal}, which explores $\sqrt{n}$-consistent methods for causal effect estimation using the MRF model. Their results include guarantees on MPLE parameter estimation and mean-field inference in a non-sequential setting where confounding is from observed covariates sampled i.i.d from an underlying distribution.

\textbf{Causal Inference for Sequential Settings and Latent Confounding.} 
Our approach to latent confounders is inspired by the factor model and synthetic controls literature \cite{amjad_2018,agarwal2026synthetic,agarwal2023causal,abadie2003economic,abadie2024doubly,abadie2010synthetic,GarciaBulle_2022}. Unobserved confounders are unidentifiable in general \cite{Imbens_2015,ding2024first}, but by considering a fixed population over time and modeling latent variables as a combination of unit-specific 
factors and time-specific shocks, it is possible to identify (some of the) causal quantities of interest.
Under such settings, synthetic controls (and synthetic interventions) recreate
the counterfactual of what would have happened had a specific unit not been treated by 
constructing synthetic outcomes 
as a linear combination of outcomes of units that did not undergo the treatment \cite{amjad_2018,agarwal2026synthetic,abadie2003economic,abadie2024doubly,abadie2010synthetic}.
Recently, this line of work has been extended to account for treatment interference of the form $x_i = f(\alpha_i,z_i,\z_{N(i)})$ \cite{Agarwal_2023}.
All of these methods allow for counterfactual estimation of a specific type, the effect on a specific unit of always or never having the intervention, whereas our model can consider arbitrary interventional patterns.

\textbf{Ising Models, Exponential Family, and Latent Confounding.} 
%
%
This work extends a line of literature on computationally efficient methods for estimating high-dimensional Ising models 
from a single joint sample~\cite{chatterjee2007estimation,bhattacharya2018inference,daskalakis2019regression,Dagan_2019,ghosal2020joint,Daskalakis_2020,dagan2021learning,Kandiros_2021,Daskalakis_2025} by 
allowing for sequential settings with latent variables. Prior works established parsimonious 
conditions under which model parameters can be learned using a single sample under generalizations of the classical 
Dobrushin's uniqueness condition \cite{dobruschin1968description}. This collection of works builds upon use of the Maximum Pseudo-Likelihood \cite{besag1974spatial,besag1975statistical}. In the context of generic exponential families, computationally efficient methods for learning have been developed 
using different loss functions, c.f.~\cite{shah2021learning, shah2021computationally, Shah_2023}. This line of approach
establishes its statistical properties through its connection to maximum likelihood estimation for a re-parameterized model class. 
%
%
This was further extended to account for unobserved confounding with a generic intervention pattern in \cite{Shah_2023}. 
Like this work, it connects the effect of unobserved confounding to unobserved external field. The results in \cite{Shah_2023}
differ from this work in that they do not allow for interference. 


\section{Setup} \label{sec:model}

We consider a network of $N$ units exposed to binary interventions over $T \geq 2$ time steps. 
We denote the binary outcome of unit $i$ at time $t$ by $x_i^{(t)} \in \{-1,1\}$. Its intervention 
is $z_i^{(t)} \in \{-1,1\}$. In addition to its outcome and intervention, we also assume that each 
unit $i$ at time $t$ is associated with a latent variable, $\alpha_i^{(t)} \in \mathbb{R}$, 
which can influence both $x_i^{(t)}$ and $z_i^{(t)}$. We are interested in a setting where our observed data 
is only a \textit{single} realization of $\x \in \{-1,1\}^{NT}$ given $\z \in \{-1,1\}^{NT}$. Similar to panel data literature, $\z$ and $\x$ both depend on latent confounders, but $\z$ are not assigned in response to realized $\x$.

\noindent{\bf Model.} We introduce the following model to capture the temporal and spatial dependence for outcomes 
$\x^{(t)} \in \{-1,1\}^N$ over $t \in [T]$ given interventions $\z^{(t)}$ over $t \in [T]$ and latent
confounders $\alpha_i^{(t)} \in \mathbb{R}, i \in [N], t \in [T]$.  Precisely, we have
\begin{gather}
    p(\x\,|\,\z) = \prod_{t=1}^T p(\x^{(t)}|\z^{(t)},\x^{(t-1)}),
        \qquad\text{where} \nonumber \\
        p\big(\x^{(t)}\,|\,\z^{(t)},\x^{(t-1)}\big) 
        \propto \exp\left(\sum_{i=1}^N \alpha_i^{(t)}x_i^{(t)} + \beta \sum_{i=1}^N x_i^{(t)}z_i^{(t)} + \xi \sum_{i < j} \gamma_{ij} x_i^{(t)} x_j^{(t)} + \eta \sum_{i=1}^N x_i^{(t)}x_i^{(t-1)}\right),
        \label{eq:model}
\end{gather}
for all $t \in [T]$. We assume $\Gamma = [\gamma_{ij}] \in \mathbb{R}^{N \times N}$ is a non-zero known symmetric matrix with zero diagonal capturing the interaction between outcomes at each time. Then 
$\theta = (A:=\{\alpha_i^{(t)}\}_{i\in[N],t\in[T]}, \beta, \xi, \eta)$ summarize the unknown parameters, where $A$ captures latent counfounders, $\beta$ the direct effect of $z_i^{(t)}$ on $x_i^{(t)}$, $\xi$ the strength of spatial dependence, and $\eta$ the strength of temporal dependence.
The model thus implicitly assumes that the treatment effect and temporal dependence are homogeneous.
If $\xi=0$, then \eqref{eq:model} reduces to a more standard logistic regression model where the outcomes of each unit are independent. If $\xi > 0$ and $\Gamma$ represents a fully connected graph, the outcome of unit $i$ is dependent on the outcomes of all other units and, through those outcomes, the full vector of interventions and all latent features. These dependencies are propagated along the edges of the network: $x_i^{(t)} \ind \x_{-i}^{(t)} \,|\, \x^{(t)}_{N(i)}$ where $N(i) = \{j : \gamma_{ij} \neq 0\}$ are the neighbors of~$i$.

\noindent{\bf Causal Estimands.} The traditional causal estimand in the non-sequential setting under 
no interference is the average treatment effect, defined as
\begin{equation*}
    \ATE = \frac{1}{N}\sum_{i=1}^N \big(\mathbb{E}[ x_i\,|\,\mathrm{do}(z_i = 1)] - \mathbb{E}[ x_i\,|\,\mathrm{do}(z_i = -1)]\big),
\end{equation*}
where `do' indicates externally setting the treatment while leaving all other aspects fixed.
With the introduction of spatial and temporal dependence, the outcome of unit $i$ at time $t$ is dependent 
on the entire treatment assignment for all units at all time steps. Therefore,  causal estimands of interest 
take the form of more general aggregated treatment effects: for any two intervention patterns 
$\z_0, \z_1 \in \{-1,1\}^{NT}$, define the Generalized Treatment Effect (GTE):
\begin{equation}
    \GTE(\z_1, \z_0) = \frac{1}{NT}\sum_{i=1}^N \sum_{t=1}^T \big(\mathbb{E}[ x_i^{(t)}\,|\,\mathrm{do}(\z = \z_1)] - \mathbb{E}[ x_i^{(t)}\,|\,\mathrm{do}(\z = \z_0)]\big).
    \label{eq:agg_treat_effect}
\end{equation}
For example, a particularly well-studied causal estimand 
in the context of interference is the global average treatment effect (GATE), which is effectively $\GTE({\bf 1}, {\bf -1})$. 
In our setting, we denote GTE for specific $\theta$ by writing $\GTE_\theta(\z_1, \z_0) = \frac{1}{NT}\sum_{i=1}^N \sum_{t=1}^T \big(\mathbb{E}_\theta[ x_i^{(t)}\,|\,\z=\z_1] - \mathbb{E}_\theta[ x_i^{(t)}\,|\,\z=\z_0]\big)$.

\textbf{Conditions for Identification, Estimation.} The following identification result establishes sufficient conditions under which the causal estimand can be learned from observational data. We write in terms of $\theta^\ast$ which denotes the true parameters of the model. Since $\GTE_{\theta^\ast}(\z_1,\z_0)$ is a function of $\theta^\ast$, it suffices to show that $\theta^\ast$ can be learned uniquely from \eqref{eq:model}.
\begin{proposition}
    Let $Z \in \{-1,1\}^{N \times T}$ be the matrix representation of the observed interventions. For any $\Theta \subset \{(A,\beta,\xi,\eta): A \in \mathbb{R}^{N \times T}, \beta, \xi, \eta \in \mathbb{R}\}$ such that $\theta^\ast \in \Theta$, if there exists a $\varkappa > 0$ with
    \begin{equation}
        \|(A-A^\ast)+(\beta-\beta^\ast)Z\|_F^2
        \geq
        \varkappa\left(
        \|A-A^\ast\|_F^2
        +
        \|(\beta-\beta^\ast)Z\|_F^2
        \right)
        \qquad \forall \theta \in \Theta,
        \label{eq:identifiability}
    \end{equation}
    then $\theta^\ast$ is identifiable in $\Theta$ from the observational distribution. Equivalently, we have
    \begin{equation*}
        p_\theta(\cdot|\z) = p_{\theta^\ast}(\cdot|\z) \,\,\Rightarrow\,\, \theta = \theta^\ast \,\,\Rightarrow\,\, \GTE_{\theta}(\z_1,\z_0) = \GTE_{\theta^\ast}(\z_1,\z_0)
    \end{equation*}
    for all $\theta \in \Theta$ and any $\z_0,\z_1 \in \{-1,1\}^{NT}$.
    \label{prop:identifiability}
\end{proposition}
The proof is deferred to Appendix \ref{app:caus_model}. 
Intuitively, if \eqref{eq:identifiability} were violated, then there exists a parameter $\theta \neq \theta^\ast$ such that $A + \beta Z = A^\ast + \beta^\ast Z$. Thus, the observational distributions over outcomes are identical but the causal effects are not. Some restriction of the parameter space is necessary since \eqref{eq:identifiability} cannot hold if $A$ can be arbitrary, but our identification result does not specify $\Theta$. 

We also specify conditions for finite-sample estimation used in Theorems \ref{thm:param_est} and \ref{thm:causal.est}. Assumptions \ref{ass:low_rank} and \ref{ass:bounded_width} constrain the parameter space (explicitly define $\Theta$) through a low-rank structure on $A^\ast$ and by bounding the effects of interference, latents, and interventions. Assumption \ref{ass:exc_interv} strengthens the identifiability condition of \eqref{eq:identifiability} to facilitate non-asymptotic guarantees.
\begin{assumption}
    The matrix $A^\ast := \{{\alpha^\ast}_{i}^{(t)}\}_{i,t} \in \mathbb{R}^{N \times T}$ of latent confounders has rank at most $k$, for some $k \ll N,T$. Equivalently, we can write
    \begin{equation*}
        A^\ast = U^\ast (V^\ast)^\top,
    \end{equation*}
    where $U^\ast \in \mathbb{R}^{N \times k}$ is a matrix of $k$-dimensional unit specific latent features and  $V^\ast \in \mathbb{R}^{T \times k}$ is a matrix of $k$-dimensional time-specific latent shocks. 
    \label{ass:low_rank}
\end{assumption}

The low rank assumption is common in the synthetic controls and factor model literature \cite{abadie2010synthetic,abadie2024doubly,agarwal2026synthetic,abadie2003economic,amjad_2018}. 
Low-rank matrices have been shown to naturally arise in modern datasets and emerge from ``well-behaved'' generative models (e.g., Lipschitz factor functions) \cite{udell2019big}. 
\begin{assumption}
    Let $\|\Gamma\|_\infty = \max_{i\in[N]} \sum_{j \in [N]} |\gamma_{ij}| = 1$.
    We assume that $|\xi^\ast| \leq 1-\rho$ and $\max_{i\in[N],t\in[T]} |{\alpha^\ast}_i^{(t)}| \leq B$, $|\beta^\ast| \leq B$, and $|\eta^\ast| \leq B$ for some constants $B \geq 1$ and $0 < \rho < 1$. 
    \label{ass:bounded_width}
\end{assumption}
Assumption \ref{ass:bounded_width} is based on Dobrushin's uniqueness condition, which implies a number of desirable properties including concentration of measure, efficient inference, and correlation decay \cite{chatterjee2005concentration,Daskalakis_2017,hayes2006simple,kunsch1982decay,Follmer_1982}. 
Specifically, the condition $\|\xi^\ast\Gamma\|_\infty < 1$ is sufficient for the conditional distribution $\x^{(t)} \,|\, \z^{(t)},\x^{(t-1)}$ to satisfy Definition \ref{def:dobrushins} below (see Lemma 2.6 of \cite{Dagan_2019}). Assumption \ref{ass:bounded_width} also bounds the maximum influence of the latents, intervention, and temporal dependence. 
\begin{definition}[Dobrushin's Uniqueness Condition]
    Let $\mu$ be a measure on $\{-1,1\}^N$. Dobrushin's interaction matrix \(C = (C_{ik})_{i,k \in [N]}\) is
    \begin{equation*}
        C_{ik} = 
        \sup_{\substack{\x_{[N]\setminus \{i,k\}} \\ x_{k},x_k' }} 
        d_{\mathrm{TV}}\Big(
        \mu_{x_i|\x_{-i}}\big(\cdot|\, \x_{[N] \setminus \{i,k\}},x_k\big), 
        \mu_{x_i|\x_{-i}}\big(\cdot|\, \x_{[N] \setminus \{i,k\}},x_k'\big)
        \Big).
    \end{equation*}
    Dobrushin's coefficient is $c = \max_{1 \leq i \leq N} \sum_{k \neq i} C_{ik}$. The measure $\mu$ satisfies Dobrushin's uniqueness condition if $c < 1$.
    \label{def:dobrushins}
\end{definition}

The final assumption chooses $\varkappa$ in \eqref{eq:identifiability}. Note that given Assumptions \ref{ass:low_rank} and \ref{ass:bounded_width}, our parameter space is $\Theta = \{(A,\beta,\xi,\eta): A \in [-B,B]^{N \times T}, |\xi| \leq 1 - \rho, |\beta| \leq B, |\eta| \leq B, \rk(A) \leq k\}$.
\begin{assumption}
    Recall that $Z \in \{-1,1\}^{N \times T}$ is the matrix representation of $\z$. Then, for any fixed $\theta = (A,\beta,\xi,\eta) \in \Theta$,
    \begin{equation*}
        \|(A-A^\ast) + (\beta-\beta^\ast) Z\|_F^2 
        \geq e^{-cB}\left(
        \|A-A^\ast\|_F^2
        +
        \|(\beta-\beta^\ast)Z\|_F^2
        \right),
    \end{equation*}
    for some constant $c$ independent of $\theta$ and with $B$ from Assumption \ref{ass:bounded_width}.
    \label{ass:exc_interv}
\end{assumption}

\section{Main Results: Learning Algorithm, Guarantees} \label{sec:theoretical_results}

We present an algorithm for learning model parameters. The algorithm is based on Maximum Pseudo-Likelihood Estimation (MPLE). We provide non-asymptotic consistency guarantees to establish
correctness by building on insights from recent works \cite{dagan2021learning, Kandiros_2021}. Finally, we discuss its implication
in terms of estimating the causal estimand defined in \eqref{eq:agg_treat_effect}. 


\noindent{\bf Maximum Pseudo-Likelihood Estimation (MPLE).} 
The maximum pseudo-likelihood is maximum likelihood estimation (MLE) applied to a conditional mean-field like approximation of true likelihood. 
Specifically, given $p \geq 1$ variables $Y_1,\dots, Y_p$ and their $n \geq 1$
observations  $y^{(i)}_1,\dots, y^{(i)}_p$ for $i \in [n]$, to estimate parameter $\phi$ from a parametric family of distribution
over potential choices of $\Phi$, one can use the MPLE,
\begin{align*}
    \phi^{\MPLE} & \in \arg\max_{\phi \in \Phi} \sum_{i=1}^n \sum_{j=1}^p \log p_{\phi}(Y_j = y^{(i)}_j | Y_{-j} = y^{(i)}_{-j}), 
\end{align*}
where $\log p_{\phi}(Y_j = \cdot | Y_{-j} = \cdot)$ represents the conditional distribution of $Y_j$ given all other $p-1$ variables 
per the distribution with respect to parameter $\phi$.

We apply MPLE to our setting by specializing to the model of~\eqref{eq:model}. The parameter of interest is 
$\theta \in \Theta$,
which we estimate using the conditional distribution of $\x | \z$. Specifically, we apply
MPLE sequentially over $t \in [T]$, i.e. while computing the conditional likelihood of $x^{(t)}_i$, we condition on 
$x^{(t)}_{-i}, x^{(t-1)}_i, z^{(t)}_i$, but not variables corresponding to time s for $s > t$. In that sense, our
estimator can be viewed as {\em sequential} MPLE. Precisely, given a single sample observation $\x, \z \in \{-1,1\}^{NT}$, 
\begin{align}\label{eq:mple.seq}
\hat{\theta} & \in \arg\min_{\theta \in \Theta} \varphi(\theta;\x,\z), \quad \mbox{where} \quad 
\varphi(\theta;\x,\z)  = - \sum_{t=1}^T \sum_{i=1}^N \log p_\theta(x_i^{(t)}|\x^{(t)}_{-i},\z_i^{(t)},\x_i^{(t-1)}).
\end{align}
The benefit of MPLE in our setting is that $\varphi$ is a convex function in $\theta$ and is easy to evaluate unlike MLE (since MPLE has a simple partition function). $\Theta$ is non-convex because of the low-rank restriction on $A$, but practical methods for low-rank optimization are  well-developed \cite{recht2010guaranteed,jain2013low,chi2019nonconvex}. We establish the following guarantees. Let
\begin{equation}
    \|\theta-\theta'\|_\star = \left(\frac{\|A - A'\|_F^2}{NT} + (\beta-\beta')^2 + (\xi-\xi')^2 + (\eta-\eta')^2\right).
    \label{eq:starred_norm}
\end{equation}
\begin{theorem}
Given a single observation $(\x,\z)$ from \eqref{eq:model} with parameters 
$\textstyle \theta^\ast = (\{A^\ast = [{\alpha^\ast}_i^{(t)} : {i\in [N], t \in[T]}]\}, \beta^\ast, \xi^\ast, \eta^\ast)$, 
let $\hat{\theta} = (\hat{A}, \hat{\beta}, \hat{\xi}, \hat{\eta})$ be the parameter estimate obtained as per \eqref{eq:mple.seq}. 
Let Assumptions \ref{ass:low_rank}-\ref{ass:exc_interv} hold. 
    Then, there exists a constant $C(\rho,B) = \rho^{-4}\exp(O(B))$ 
    such that for any $\delta \in (0,1/2)$, with probability at least $1-\delta$, we have
    \begin{align}\label{eq:thm1.main}
            \|\hat{\theta} - \theta^\ast\|_\star 
            & \leq C(\rho,B)\left(\frac{k(N + T)\log T + \log \frac{1}{\delta}}{T\|\Gamma\|_F^2}\right). 
        \end{align}
    \label{thm:param_est}
\end{theorem}

We defer a full proof of Theorem \ref{thm:param_est} to Appendix \ref{app:thm_proof}. Our dependence on $1/\|\Gamma\|_F^2$ is unavoidable as per Theorem 3 of \cite{Kandiros_2021}, which extends naturally to our setting; intuitively, $\|\Gamma\|_F^2$ acts as our `effective sample size' over the $N$ units. Our dependence on $k(N+T)\log T$ is from the log entropy of entry-wise bounded matrices of dimension $N \times T$ with rank $k$. Corollary \ref{cor:main} offers sufficient conditions for the error to vanish asymptotically.
\begin{corollary}
    We assume that $N \geq T$ and that $k$ is a constant. If $T\|\Gamma\|_F^2 = \omega(N \log T)$, then $\lim_{N,T \to \infty} \|\hat{\theta} - \theta^\ast\|_\star  = 0$.
    \label{cor:main}
\end{corollary}
A natural setting where $T\|\Gamma\|_F^2 = \omega(N\log T)$ holds is when $T = \omega(1)$ and $\|\Gamma\|_F = \Omega(\sqrt{N})$, e.g. if $\Gamma$ is connected with bounded maximum degree.

\noindent{\bf Estimating Causal Estimand.} 
The primary causal estimand of interest is $\GTE_{\theta^\ast}(\z_1, \z_0)$ for any two distinct interventions $\z_0, \z_1 \in \{-1,1\}^{NT}$.
Since $\theta^\ast$ is unknown, we use $\GTE_{\hat{\theta}}(\z_1, \z_0)$ which requires calculating $\mathbb{E}_{\hat{\theta}}[x_i^{(t)}| \z]$ for all $i \in [N]$, $t \in [T]$, and $\z \in \{\z_0,\z_1\}$.
We use a \textit{sequential} Gibbs approach which samples $\x^{(t)} | \z^{(t)}, \x^{(t-1)}$ via Gibbs sampling for each $t \in [T]$ sequentially \cite{geman1984stochastic,Wainwright_2008}.
Since the Gibbs sampler mixes fast for Ising models under Dobrushin's condition, sequential Gibbs permits efficient calculation of $\mathbb{E}_{\theta}[x_i^{(t)}| \z]$ for any $i \in [N], t\in [T], \z \in \{-1,1\}^{NT}$, and $\theta \in \Theta$ \cite{hayes2006simple}. Inference is computationally hard in the worst-case for models violating Dobrushin's condition \cite{mossel2009hardness}.

Given estimation of $\GTE_{\hat{\theta}}(\z_1, \z_0)$, it remains to control the perturbation error resulting from using $\hat{\theta}$ instead of $\theta^\ast$.
Theorem \ref{thm:causal.est} leverages correlation decay for models under Dobrushin's condition to argue that local errors in estimating $p_{\theta}(x_i^{(t)}=1| \z^{(t)}, \x_{-i}^{(t)}, \x^{(-t)})$ using $\theta' \neq \theta$ do not propagate spatially and temporally \cite{dobruschin1968description,Follmer_1982}. Assumption \ref{ass:bounded_width} is thus the unifying condition enabling end-to-end realization of parameter estimation, efficient inference, and perturbation bounds.
\begin{theorem}\label{thm:causal.est}
    For any $\z_0, \z_1 \in \{-1, 1\}^{NT}$ and model parameters $\theta$ and $\theta'$ such that $|\eta| + |\xi| \leq 1 - \kappa$ and $|\eta'| + |\xi'| \leq 1 - \kappa$ for some $0 < \kappa < 1$, there exists a constant $c>0$ such that
    \begin{equation*}
        (\GTE_{\theta}(\z_1, \z_0) - \GTE_{\theta'}(\z_1, \z_0))^2 \leq c\kappa^{-2}\|\theta-\theta'\|_\star.
    \end{equation*}
\end{theorem}
The proof is deferred to Appendix \ref{app:thm_proof_2}. Theorems \ref{thm:param_est} and \ref{thm:causal.est} imply the following. Generally, it is not guaranteed that $\hat{\theta}$ will satisfy the condition $|\hat{\eta}| + |\hat{\xi}| \leq 1 - \kappa$ after parameter learning, though this can be imposed as an additional constraint on parameter space $\Theta$ during optimization.
\begin{corollary}
    Let $\hat{\theta}$ be the MPLE estimate of $\theta^\ast$.
    Under Assumptions \ref{ass:low_rank}-\ref{ass:exc_interv} and if $|\hat{\eta}| + |\hat{\xi}| \leq 1 - \kappa$ and $|\eta^\ast| + |\xi^\ast| \leq 1 - \kappa$ for some $0 < \kappa < 1$, there exists a constant $C(\rho,B,\kappa) = \kappa^{-2}\rho^{-4}\exp(O(B))$ such that for any $\delta \in (0,1/2)$, with probability at least $1-\delta$, we have
    \begin{equation*}
         (\GTE_{\hat{\theta}}(\z_1, \z_0) - \GTE_{\theta^\ast}(\z_1, \z_0))^2
         \leq C(\rho,B,\kappa)\left(\frac{k(N+T)\log T + \log \frac{1}{\delta}}{T\|\Gamma\|_F^2}\right).
    \end{equation*}
    \label{cor:causal.est.cor}
\end{corollary}

\section{Empirical Results} \label{sec:empirical_results}

\subsection{Synthetic Experiments} \label{subsec:synth}

We conduct experiments to validate that our method estimates parameters and causal estimands with reasonable data efficiency\footnote{All code is available at the link https://github.com/phevosp/MPLE-Causal-Inference.}. We compare to variations of our approach that assume no interference or no latent confounding as well as a state-of-the-art causal inference method \cite{agarwal2023causal}.

\textbf{Setting.} We set $N=500, T=50, k=3$. To generate interventions, we construct a propensity matrix from randomly sampled latent features and sample each $z_i^{(t)}$ independently. We re-use the same latent features to construct $A^\ast$, which ensures latent confounding. We let $\Gamma$ be Erdos-Renyi with $p=0.01$, normalized so $\|\Gamma\|_\infty = 1$. Values of $\beta^\ast$, $\xi^\ast$, and $\eta^\ast$ vary. Given all relevant parameters, a single sample from the trajectory is generated via sequential Gibbs with $B=100$ local updates for each $i\in[N], t\in [T]$. 
Appendix \ref{app:exp_settings} details the complete experimental setting.



\textbf{Methods.} We compute the MPLE parameters $\hat{\theta}$ by parameterizing $A = UV^\top$ and minimizing 
\begin{equation}
    \mathcal{L}(\theta,\lambda) = \varphi(\theta) + \lambda(\|U\|_F^2 + \|V\|_F^2), 
    \label{eq:regularized_mple}
\end{equation}
where $\varphi$ is from \eqref{eq:mple.seq}. This is a non-convex objective in $U$ and $V$, so we alternate taking steps where $U$ or $V$ is fixed (see \cite{jain2013low}). We select $\lambda$ using cross-validation (see Appendix \ref{app:exp_details}). We similarly fit $\hat{\theta}_{\xi=0}$ and $\hat{\theta}_{A=0}$ which differ from $\hat{\theta}$ in restricting $\xi=0$ and $A=0$, respectively. The former is equivalent to a traditional logistic regression model (no interference); the latter assumes no confounding.

We also implement the Synthetic Nearest Neighbors (SNN) algorithm of \cite{agarwal2023causal}, which generalizes synthetic controls. SNN constructs $M^{(1)}, M^{(-1)} \in \{-1,1\}^{N \times T}$ as matrices of outcomes for treated and untreated unit-times with missing data (for each $i\in[N],t\in[T]$ one of $M_{it}^{(1)}$ or $M_{it}^{(-1)}$ is observed). It then estimates the underlying potential outcome matrices $P^{(1)}$ and $P^{(-1)}$, which are assumed to be low-rank. SNN does not model interference or time-dependence but allows treatment effect to be inhomogeneous.

\textbf{Results for MPLE.} Table \ref{tab:parameter_gate_recovery} summarizes the estimation of parameters and causal estimands for $\hat{\theta}$, $\hat{\theta}_{\xi=0}$, and $\hat{\theta}_{A=0}$ when $\beta^\ast=-0.3$, $\xi^\ast = 0.8$, and $\eta^\ast=0.3$. Relative to $\hat{\theta}_{\xi=0}$ and $\hat{\theta}_{A=0}$, $\hat{\theta}$ improves estimation of $\GTE({\bf 1},-{\bf 1})$ by 92\% and 91\% (in absolute error), respectively. 

\begin{table}[hbt]
    \centering
    \begin{tabular}{l|ccccc}
        & $\beta$ & $\xi$ & $\eta$ & $A$ RMSE & $\GTE({\bf 1},{\bf -1})$ \\
        \hline

        $\theta^\ast$
            & $-0.300$
            & $0.800$
            & $0.300$
            & $0.000$
            & $-0.705 \pm 0.010$ \\

        $\hat{\theta}$
            & $-0.279 \pm 0.003$
            & $0.799 \pm 0.019$
            & $0.293 \pm 0.002$
            & $0.322 \pm 0.003$
            & $-0.689 \pm 0.015$ \\

        $\hat{\theta}_{\xi=0}$
            & $-0.279 \pm 0.003$
            & $0.000 \pm 0.000$
            & $0.295 \pm 0.002$
            & $0.325 \pm 0.003$
            & $-0.513 \pm 0.005$ \\
            
        $\hat{\theta}_{A=0}$
            & $-0.154 \pm 0.004$
            & $0.712 \pm 0.026$
            & $0.222 \pm 0.005$
            & $0.750 \pm 0.000$
            & $-0.533 \pm 0.023$ \\

        \hline
    \end{tabular}
    \caption{Parameter and $\GTE$ recovery. Entries report means across 10 trials with standard error. For the latent field, the reported value is the root mean square error (RMSE). We estimate $\GTE$ in each trial by averaging eight trajectories each generated with $B=100$ local updates.}
    \vspace{-1em}
    \label{tab:parameter_gate_recovery}
\end{table}
Recall that Theorem \ref{thm:param_est} bounds the MSE of $\hat{A}$ and squared error of $\hat{\beta}$, $\hat{\xi}$, $\hat{\eta}$ by $O((N+T)\log T/T\|\Gamma\|_F^2)$ and Theorem \ref{thm:causal.est} controls $\GTE_{\hat{\theta}}$ similarly if $|\xi^\ast| + |\eta^\ast| < 1$.
In this case, $\|\Gamma\|_F^2 = 9.13 \pm 0.71$ and $|\xi^\ast| + |\eta^\ast| > 1$ but the MPLE solution $\hat{\theta}$ still recovers the parameters and causal effect. 
Table \ref{tab:parameter_gate_recovery} also suggests latent field recovery is harder than scalar estimation, which is intuitive but not captured by our theoretical results.
Nevertheless, recovering $A^\ast$ well enough to capture latent confounding and identify $\beta^\ast$ permits accurate $\GTE$ estimation whereas setting $A=0$ fails even though ${\alpha^\ast}_i^{(t)}$ is zero in expectation. $\hat{\theta}$ is strictly more general than $\hat{\theta}_{\xi=0}$ and, as shown in Table \ref{tab:parameter_gate_recovery_no_interference}, recovers its performance when the data satisfy no interference.

\textbf{Comparison to SNN.} Table \ref{tab:gte-comparison} compares estimation of $\GTE({\bf 1}, {\bf -1})$ by $\hat{\theta}$ and SNN for many values of $\beta^\ast$ and $\xi^\ast$. $\eta^\ast = 0.1$ is fixed.
SNN underestimates GTE under interference (when $\xi^\ast = 0.8$) and when treatment effects are large (performance worsens as $\beta^\ast$ increases). SNN is also infeasible in large data settings (e.g., the real world setting of Section \ref{subsec:real-world-results}). Other experiments demonstrate that in small data settings (e.g., $N=200$, $T=20$) with no interference and small $\beta^\ast$, SNN more accurately estimates GTE than $\hat{\theta}$. SNN may also be preferred when treatment effects are inhomogeneous.

\begin{table}[hbt]
    \centering
    \begin{tabular}{ccccc}
        $\beta^\ast$ & $\xi^\ast$ & Truth GTE & MPLE GTE & SNN GTE \\
        \hline
        0.10 & 0.0 & 0.160 & 0.177 & 0.139 \\
        0.10 & 0.8 & 0.225 & 0.275 & 0.138 \\
        0.15 & 0.0 & 0.236 & 0.266 & 0.182 \\
        0.15 & 0.8 & 0.336 & 0.375 & 0.178 \\
        0.20 & 0.0 & 0.318 & 0.355 & 0.219 \\
        0.20 & 0.8 & 0.434 & 0.478 & 0.216 \\
        \hline
    \end{tabular}
    \caption{Comparison of the true, MPLE-estimated, and SNN-estimated global treatment effect. We estimate $\GTE({\bf 1}, {\bf -1})$ for ground-truth and MPLE parameters by averaging eight trajectories each generated with $B=100$ local updates. $\GTE({\bf 1}, {\bf -1})$ for SNN are calculated from $P^{(1)}$ and $P^{(-1)}$.}
    \label{tab:gte-comparison}
\end{table}

\subsection{COVID-19 Vaccination Case Study} \label{subsec:real-world-results}

We investigate the causal effects of the COVID-19 vaccine on death rates among US counties. We first validate that the observed interventional pattern permits identification via MPLE whereas SNN is infeasible due to problem size and missingness pattern. Then, we argue that our model is sufficiently expressive to capture COVID-19 outcome data. Finally, we estimate causal effects of interest and demonstrate the role of interference. 

\textbf{Setting.} We compile COVID-19 per-county death data from the NYT and vaccination data from the CDC and Bansal \cite{NYTimes2021covid19data,CDCCovidDataTracker,BansalLabVaccineTracking}. 
Our outcomes are 1 if county $i$ had more than 2 deaths per 100,000 residents during week $t$. 
Our interventions are 1 if the county was more than 30\% vaccinated with a two week lag such that, e.g. vaccine rates from Sep. 1-7 affect death rates from Sep. 14-21. 
Thresholds are chosen to ensure diversity in $\x$ and $\z$.
The interaction graph $\Gamma$ is constructed using a distance kernel function that keeps the eight nearest neighbors and weighs via exponential decay; our edge weights are $w_{ij} = e^{-d(i,j)/\bar{d}}$ where $d(i,j)$ is the distance between the centroids of $i$ and $j$ and $\bar{d}$ is the median distance in the graph. We normalize so $\|\Gamma\|_\infty = 1$.
Our final data are the outcomes and interventions of $N=3,014$ mainland US counties with a population greater than $2,000$ from March 1st, 2020 until May 15th, 2022 ($T=115$). The first county becomes vaccinated at time $t=50$.

\textbf{Hybrid Experiments.}
The interventional pattern of COVID-19 vaccination may not satisfy Assumption \ref{ass:exc_interv} since once a county crosses the 30\% vaccination threshold it remains vaccinated (see Figure \ref{fig:app_interv}). To validate that the data permit identifiability, we conduct hybrid experiments where $\z$ and $\Gamma$ are from the data but we generate $\x$ according to our model. Specifically, we do singular value decomposition on $Z \in \mathbb{R}^{N\times T}$ and keep the features associated with the top 1 singular value (77\% of energy). We supplement with random features to define $A^\ast$ with rank $k=5$. Thus, we have low rank biases with partial confounding. Full details are in Appendix \ref{app:exp_settings}.


Table \ref{tab:parameter_gate_recovery_hybrid} summarizes the parameter and causal effect estimation results of $\hat{\theta}$, $\hat{\theta}_{\xi=0}$, and $\hat{\theta}_{A=0}$.
The large setting and missingness pattern are infeasible for SNN.
Our results mirror the purely synthetic experiments (Table \ref{tab:parameter_gate_recovery}), demonstrating that the interventional distribution permits identification when the model is correct.
The results also reinforce the importance of capturing latent confounding while estimating $A^\ast$, since $\hat{\theta}$ and $\hat{\theta}_{A=0}$ have similar RMSE of $A$ but only the former learns $\beta^\ast$ and $\GTE$.

\begin{table}[hbt]
    \centering
    \begin{tabular}{l|ccccc}
        & $\beta$ & $\xi$ & $\eta$ & $A$ RMSE & $\GTE({\bf 1}_{t\geq 50},{\bf -1})$ \\
        \hline

        $\theta^\ast$
            & $-0.300$
            & $0.800$
            & $0.300$
            & $0.000$
            & $-0.562 \pm 0.001$ \\

        $\hat{\theta}$
            & $-0.310 \pm 0.002$
            & $0.840 \pm 0.001$
            & $0.299 \pm 0.001$
            & $0.388 \pm 0.074$
            & $-0.554 \pm 0.002$ \\

        $\hat{\theta}_{\xi=0}$
            & $-0.345 \pm 0.017$
            & $0.000 \pm 0.000$
            & $0.306 \pm 0.001$
            & $0.244 \pm 0.009$
            & $-0.410 \pm 0.019$ \\

        $\hat{\theta}_{A=0}$
            & $-0.097 \pm 0.001$
            & $0.813 \pm 0.004$
            & $0.278 \pm 0.002$
            & $0.4 \pm 0.000$
            & $-0.303 \pm 0.001$ \\

        \hline
    \end{tabular}
    \caption{Parameter and $\GTE$ recovery for hybrid experiments. Entries report means across 10 trials with standard error. For the latent field, the reported value is the RMSE. We estimate $\GTE$ in each trial by averaging eight trajectories each generated with $B=100$ local updates. Note that ${\bf 1}_{t \geq 50}$ denotes the intervention where all counties are vaccinated starting at $t=50$.}
    \label{tab:parameter_gate_recovery_hybrid}
\end{table}
\textbf{Test Set Recovery.} To validate that our model can capture the COVID-19 outcome data, we evaluate recovery of an unseen test set. The construction of a test set is complicated by unit and time-specific latents and dependence between units.
We partition the graph and time horizon into 6 components, $C_1,\dots,C_6$ and $T_1,\dots,T_6$, and define
$
    \mathcal{T}_{test} = \bigcup_{j=1}^6 \{x_i^{(t)}: i \in C_j, t \in T_j\}
$ (16.67\% of data).
This allows us to observe enough data from each unit and time step to reconstruct the external field despite missingness. We also condition on any $x_i^{(t)}$ that is directly connected spatio-temporally to an $x_j^{(r)} \in \mathcal{T}_{test}$ to prevent data leakage due to interference. See Figure \ref{fig:test-set-viz} for a visualization. On the remaining outcomes ($80.6\%$ of data), we fit $\hat{\theta}$ and $\hat{\theta}_{\xi=0}$ by minimizing \eqref{eq:regularized_mple} using cross-validation for choice of rank and $\lambda$ (see Appendix \ref{app:exp_details} for details). After learning model parameters, we use sequential Gibbs (32 samples; $B=100$) to recreate all outcomes under the observed interventions.

Table \ref{tab:test_statistics} reports the absolute error (AE) in estimating the average outcome and the marginal Brier score (see Equation \eqref{eq:marginal_brier} in Appendix \ref{app:supp_results} for definition). We compare to baselines which, using the training data, predict $x_i^{(t)}$ to be the average outcome for time $t$, the average outcome for unit $i$, or the previous outcome for unit $i$.
We repeat the same process for the hybrid setting and present analogous results in Table \ref{tab:test_statistics_hybrid}.

\begin{table}[hbt]
    \centering
    \begin{tabular}{l|ccc}
         & Train AE ($t \geq 50$) & Test AE ($t \geq 50$) & Test Brier Score ($t \geq 50$)\\ \hline 
         
       $\hat{\theta}$ 
        & $0.001$
        & $0.033$
        & $0.192$
        \\
        
       $\hat{\theta}_{\xi=0}$ 
       & $0.002$
       & $0.048$
       & $0.192$
       \\

       Time Avg.
       & $0.000$
       & $0.076$
       & $0.201$
       \\

      Unit Avg.
       & $0.047$
       & $0.028$
       & $0.228$
    \\
       Persistence
       & $0.061$
       & $0.124$
       & $0.405$
       \\
       \hline
    \end{tabular}
    \caption{Absolute error for average outcome estimation and marginal Brier score for real world data. The train set comprises 156,865 outcomes for $t \geq 50$; the test set comprises 31,488 for $t \geq 50$. The observed average outcome is $-0.282$ for $t \geq 50$ in the training set and $-0.206$ for the testing set.}
    \label{tab:test_statistics}
\end{table}

\begin{table}[hbt]
    \centering
    \begin{tabular}{l|ccc}
         & Train AE ($t \geq 50$) & Test AE ($t \geq 50$) & Test Brier Score ($t \geq 50$) \\ \hline 
         
       $\hat{\theta}$ 
        & $0.007$
        & $0.010$
        & $0.222$
        \\
        
       $\hat{\theta}_{\xi=0}$ 
       & $0.006$
       & $0.019$
       & $0.224$
       \\

       Time Avg.
       & $0.000$
       & $0.036$
       & $0.250$
       \\

      Unit Avg.
       & $0.084$
       & $0.098$
       & $0.251$
    \\
       Persistence
       & $0.006$
       & $0.041$
       & $0.476$
       \\
       \hline
    \end{tabular}
    \caption{Absolute error for average outcome estimation and marginal Brier score for hybrid data. The train and test sets are identical to the real-world setting. The observed average outcome for $t \geq 50$ is $-0.041$ in the training set and $-0.007$ in the testing set. Results average 32 MCMC samples with $B=100$ local updates.}
    \label{tab:test_statistics_hybrid}
\end{table}

The results indicate that the model is expressive enough to capture the real world data since the absolute error is small, superior to baseline methods, and on the same order of magnitude as the hybrid case. Furthermore, $\hat{\theta}$ outperforms $\hat{\theta}_{\xi=0}$ for $t \geq 50$.
Finally, we note that $\hat{\xi} = 1.11$ for $\hat{\theta}$ fit on the COVID data which violates Dobrushin's condition. We still expect fast mixing in practice and validate by comparing results when $B=100$ and $B=500$ (Table \ref{tab:app_test_stats_B_comp}).

\textbf{Estimating Causal Effect of COVID-19 Vaccine.} To estimate causal effects, we fit $\hat{\theta}$ and $\hat{\theta}_{\xi=0}$ on the full data after re-selecting hyperparameters using cross validation. We use sequential Gibbs (8 samples; $B=100$ local updates) to estimate counterfactuals. 
Figure \ref{fig:caus_comparison} compares the observed data with the estimated evolution of the average COVID-19 outcome for two counterfactuals: a `no interventions' policy where no county exceeds 30\% vaccination and an `all interventions' policy where all counties are vaccinated at $t=50$. For $t < 50$, all scenarios are identical.
Table \ref{tab:caus_comparison} summarizes the estimated causal effect for $t \geq 50$.
Table \ref{app:sensitivity_analysis} in Appendix \ref{app:supp_results} ensures that our results are robust to choice of hyperparameters.
For further supplemental results, see Figures \ref{fig:app_unit_comparison} and \ref{fig:app_observed_recreation} in Appendix \ref{app:supp_results}.

\begin{figure}[hbt]
    \centering
    \begin{subfigure}{0.48\linewidth}
        \centering
        \includegraphics[width=\linewidth]{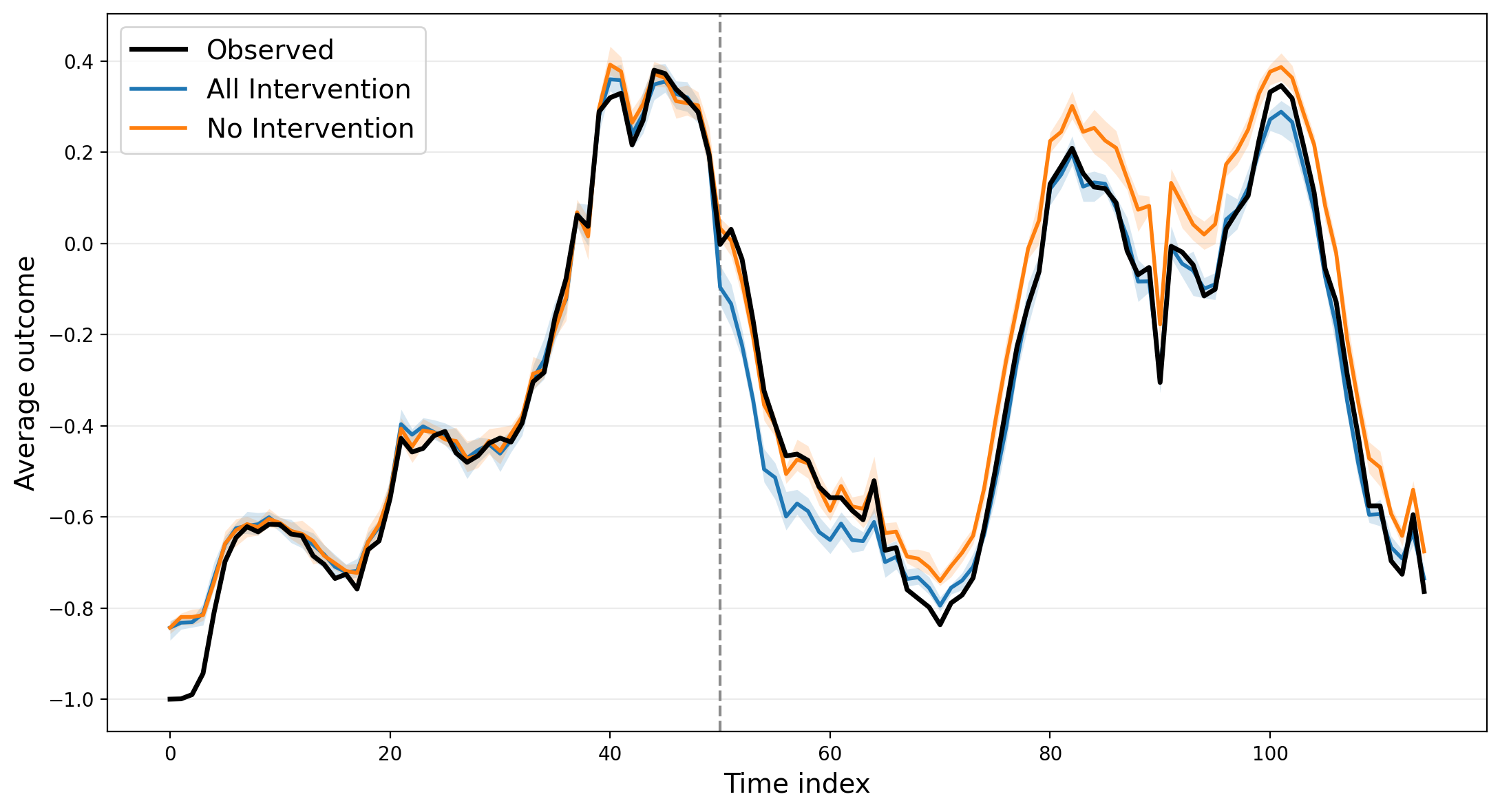}
        \caption{Interference.}
        \label{fig:interference}
    \end{subfigure}
    \hfill
    \begin{subfigure}{0.48\linewidth}
        \centering
        \includegraphics[width=\linewidth]{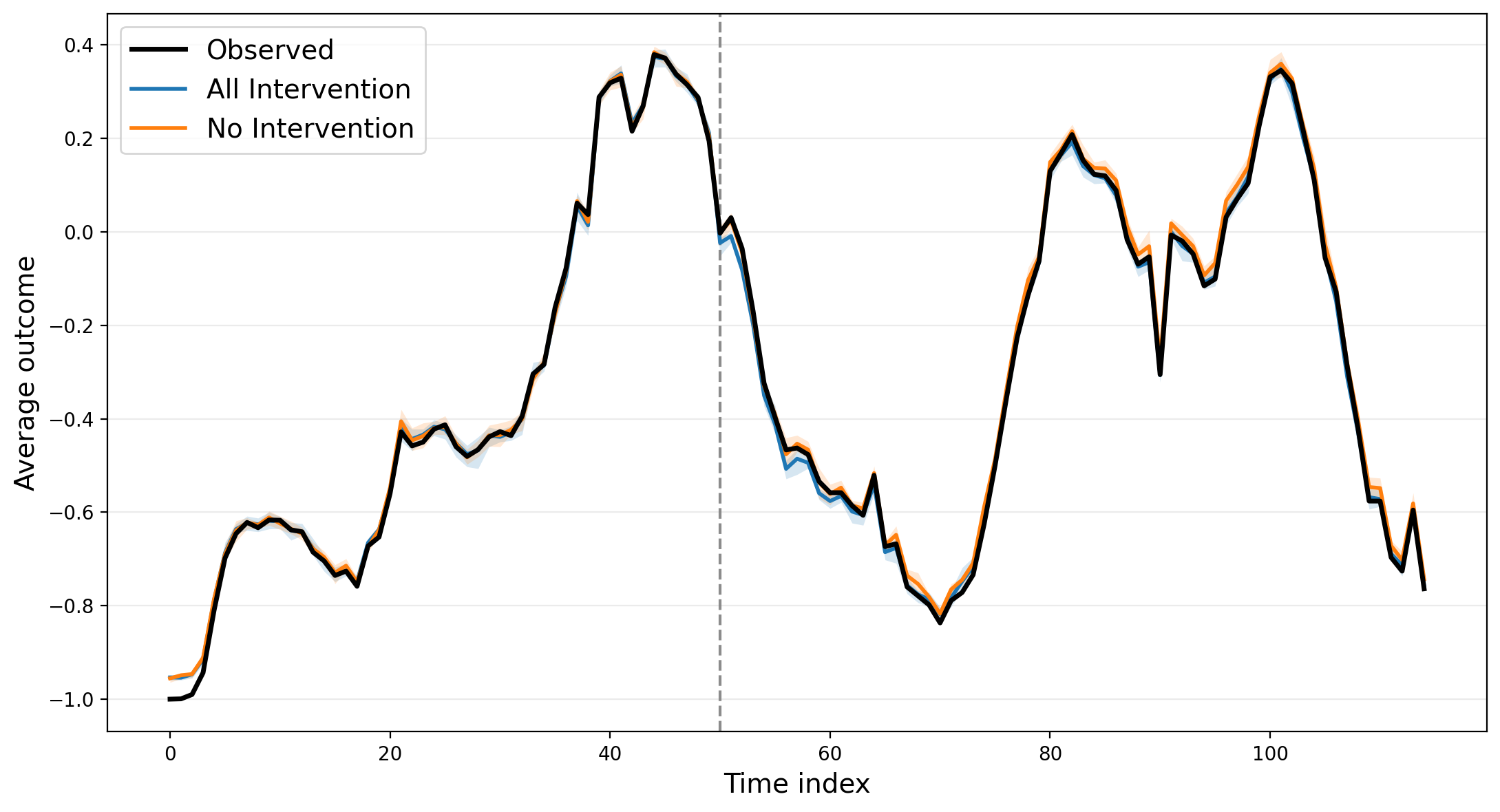}
        \caption{No Interference.}
        \label{fig:no_interference}
    \end{subfigure}
    \caption{Average outcomes over time under counterfactual scenarios.}
    \label{fig:caus_comparison}
\end{figure}

\begin{table}[hbtp]
    \centering
    \begin{tabular}{l|cccc}
    \hline
         & `No Interventions' ($\z_0$) & `All Interventions' ($\z_1$) & $\GTE(\z_1,\z_0)$ \\
    \hline
        $\hat{\theta}$ & $-0.191$ & $-0.297$ & $-0.106$ \\
        $\hat{\theta}_{\xi=0}$ & $-0.249$ & $-0.269$ & $-0.020$ \\
    \hline
    \end{tabular}
    \caption{Average outcomes for $t \geq 50$ under counterfactual scenarios. Estimates average 8 samples generated using sequential Gibbs with $B=100$ updates.}
    \label{tab:caus_comparison}
    \vspace{-0.5em}
\end{table}

Under both models, the `no interventions' counterfactual closely follows the observed outcomes initially (when most counties are unvaccinated) before predicting worse (more positive) outcomes. The `all interventions' counterfactual predicts better outcomes initially, before converging to the observed outcomes when most counties become vaccinated. Thus, both models conclude that vaccination decreased COVID-19 death rates, which is consistent with previous population-level evidence \cite{suthar2022public,mclaughlin2022county}. The estimated total effect is substantially larger when interference is modeled which aligns with existing evidence of indirect vaccine protection \cite{milman2021community, eyre2022effect}.

\section{Limitations} \label{sec:limitations}

Our theoretical results are under Dobrushin's uniqueness condition, which may not hold for all real-world experiments. In particular, some settings may exhibit strong outcome interference. Given our interest in causal inference, it is impossible to understand the performance of our model in estimating effects of unobserved interventions on real world data. In our experimental setting, we validate our approach by testing its ability to estimate unseen data under the implemented interventional policy.

\section{Conclusion}

We investigate causal inference for sequential settings with interference and latent confounding. We provide guarantees on parameter estimation via a computationally efficient learning approach of sequential Maximum Pseudo-Likelihood Estimation (MPLE). This enables estimation of a diverse set of causal quantities of interest. In addition to theoretical guarantees, we support our method with experiments on synthetic data and via a real-world case study investigating the effects of COVID-19 vaccines on death rates in US counties nationwide. Our work offers a new approach for causal inference in a complex setting, enabling a better understanding of causal relationships between treatments and outcomes. The establishment of causal relationships can inform important policy decisions.

\section*{Acknowledgments} Phevos Paschalidis was supported by a Paris Kanellakis fellowship. Constantinos Daskalakis was supported by a Simons Investigator Award, a Simons Collaboration
on Algorithmic Fairness, ONR MURI grant N00014-25-1-2116, ONR grant N00014-25-1-2296. This work was also supported in part by ONR MURI grant number N00014-24-1-2687.
The authors have no conflict of interests to report.

\bibliographystyle{ieeetr}
\bibliography{bib}

\appendix

\section{AI Statement}

In this work, we used generative AI to assist in creating or editing software code and to double check mathematical proof details. We have not used generative AI to generate synthetic datasets, help develop frameworks, formulate claims, write proofs, propose or refine hypotheses, design feedback on methodology or experiments, assist with translation, or interpret results.

\section{Causal Model} \label{app:caus_model}

\setcounter{figure}{0}
\renewcommand{\thefigure}{B\arabic{figure}} 

\subsection{Representations of Causal Model}

Figure \ref{fig:graph_model_i_t} represents the latent confounding for a single $i$ and $t$. Our consideration of latent variables in \eqref{eq:model} is fully general for binary $z_i^{(t)}$ and $x_i^{(t)}$. Figure \ref{fig:graph_model} offers a graphical representation for the distribution over two units and time steps.
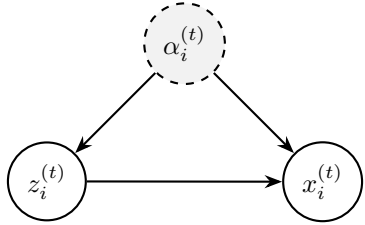
\begin{figure}[hbt]
    \centering
    \begin{tikzpicture}[
        node distance=1.5cm,
        thick,
        state/.style={circle, draw, minimum size=1cm},
        latent/.style={circle, draw, dashed, minimum size=1cm, fill=gray!10}
    ]
    
        \node[latent] (u) {$\alpha_i^{(t)}$};
        \node[state, below left=of u] (z) {$z_i^{(t)}$};
        \node[state, below right=of u] (x) {$x_i^{(t)}$};
    
        \draw[->, >={Stealth}] (u) -- (z);
        \draw[->, >={Stealth}] (u) -- (x);
        \draw[->, >={Stealth}] (z) -- (x);
    \end{tikzpicture}
    \caption{Our causal model for a single unit $i$ at time $t$.}
    \label{fig:graph_model_i_t}
\end{figure}

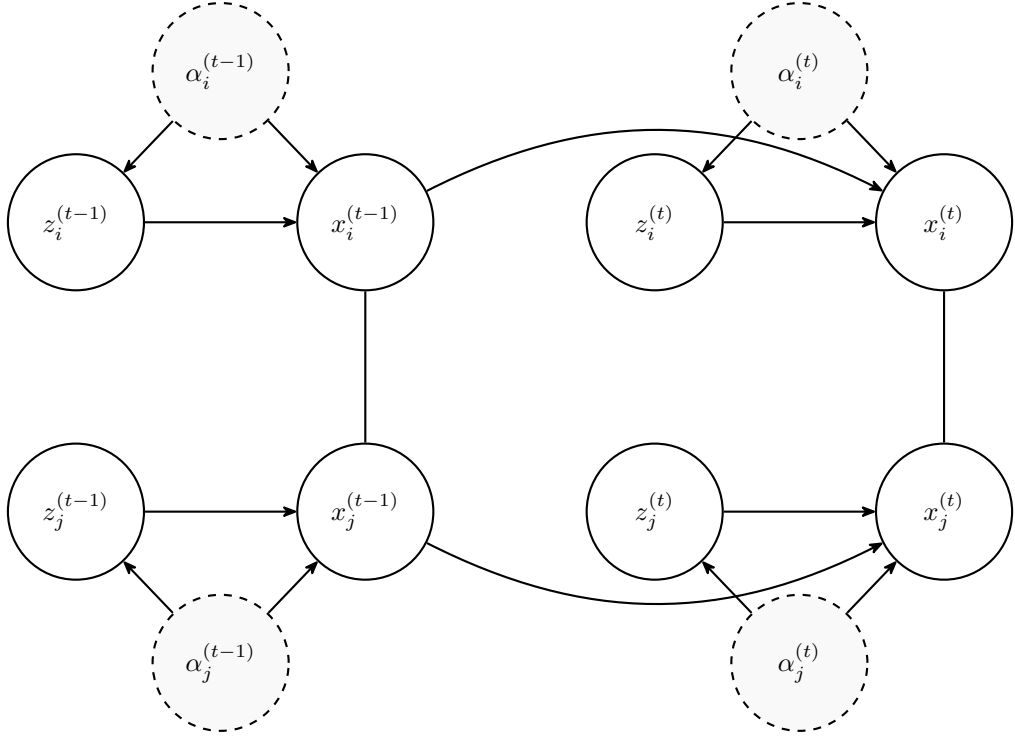
\begin{figure}[hbt]
    \centering
    \begin{tikzpicture}[
        node distance=2cm,
        thick,
        state/.style={circle, draw, minimum size=1.8cm},
        latent/.style={circle, draw, dashed, minimum size=1.8cm, fill=gray!5},
        >={Stealth[round]}
    ]
    
        \node[state] (zit-1) {$z_i^{(t-1)}$};
        \node[state] (xit-1) [right=of zit-1] {$x_i^{(t-1)}$};
        \node[latent] (uit-1) at ($(zit-1)!0.5!(xit-1) + (0,2)$) {$\alpha_i^{(t-1)}$};
    
        \node[state] (zjt-1) [below=2cm of zit-1] {$z_j^{(t-1)}$};
        \node[state] (xjt-1) [below=2cm of xit-1] {$x_j^{(t-1)}$};
        \node[latent] (ujt-1) at ($(zjt-1)!0.5!(xjt-1) + (0,-2)$) {$\alpha_j^{(t-1)}$};
    
        \foreach \s in {i,j} {
            \draw[->] (u\s t-1) -- (z\s t-1);
            \draw[->] (u\s t-1) -- (x\s t-1);
            \draw[->] (z\s t-1) -- (x\s t-1);
        }
        \draw[thick] (xit-1) -- (xjt-1);
    
        \node[state] (zi) [right=2cm of xit-1] {$z_i^{(t)}$};
        \node[state] (xi) [right=of zi] {$x_i^{(t)}$};
        \node[latent] (ui) at ($(zi)!0.5!(xi) + (0,2)$) {$\alpha_i^{(t)}$};
    
        \node[state] (zj) [below=2cm of zi] {$z_j^{(t)}$};
        \node[state] (xj) [below=2cm of xi] {$x_j^{(t)}$};
        \node[latent] (uj) at ($(zj)!0.5!(xj) + (0,-2)$) {$\alpha_j^{(t)}$};
    
        \foreach \s in {i,j} {
            \draw[->] (u\s) -- (z\s);
            \draw[->] (u\s) -- (x\s);
            \draw[->] (z\s) -- (x\s);
        }
        \draw[thick] (xi) -- (xj);
    
        \draw[->] (xit-1) [bend left=27] to (xi);
        \draw[->] (xjt-1) [bend right=27] to (xj);
    
    \end{tikzpicture}
    \caption{Our causal model for two units across two time steps, coupled spatially and temporally.}
    \label{fig:graph_model}
\end{figure}

\subsection{Identifiability of Causal Model} \label{app:ident_proof}

\setcounter{proposition}{0}
\begin{proposition}
    Let $Z \in \{-1,1\}^{N \times T}$ be the matrix representation of the observed interventions. For any $\Theta \subset \{(A,\beta,\xi,\eta): A \in \mathbb{R}^{N \times T}, \beta, \xi, \eta \in \mathbb{R}\}$ such that $\theta^\ast \in \Theta$, if there exists a $\varkappa > 0$ with
    \begin{equation*}
        \|(A-A^\ast)+(\beta-\beta^\ast)Z\|_F^2
        \geq
        \varkappa\left(
        \|A-A^\ast\|_F^2
        +
        \|(\beta-\beta^\ast)Z\|_F^2
        \right)
        \qquad \forall \theta \in \Theta,
    \end{equation*}
    then $\theta^\ast$ is identifiable from the conditional observational distribution of \eqref{eq:model} with respect to $\Theta$. Equivalently, for $\theta \in \Theta$ we have $p_\theta(\cdot|\z) = p_{\theta^\ast}(\cdot|\z) \Rightarrow \theta = \theta^\ast$ where $p_\theta(\x|\z)$ is from \eqref{eq:model}.
\end{proposition}

\begin{proof}[Proof of Proposition \ref{prop:identifiability}]
Fix $\theta=(A,\beta,\xi,\eta)\in\Theta$ and suppose that
\begin{equation*}
    p_\theta(\,\cdot\mid \z)
    =
    p_{\theta^\ast}(\,\cdot\mid \z).
\end{equation*}
We show that $\theta=\theta^\ast$.

Because all parameters are finite, every outcome trajectory has strictly
positive probability under the model. Equality of the joint distributions
therefore implies equality of the corresponding conditional transition
distributions. Thus, for every $t\in[T]$ and every
$\y^{(t-1)},\y^{(t)}\in\{-1,1\}^N$, with $\y^{(t-1)}$ in the support of
$\x^{(t-1)}$, we have
\begin{equation*}
    p_\theta\big(
        \x^{(t)}=\y^{(t)}
        \mid \z^{(t)}, \x^{(t-1)}=\y^{(t-1)}
    \big)
    =
    p_{\theta^\ast}\big(
        \x^{(t)}=\y^{(t)}
        \mid \z^{(t)}, \x^{(t-1)}=\y^{(t-1)}
    \big).
\end{equation*}
Consequently,
\begin{equation*}
    \alpha_i^{(t)}-\alpha_i^{\ast(t)}
    +(\beta-\beta^\ast)z_i^{(t)}
    +(\xi-\xi^\ast)\sum_{j\neq i}\gamma_{ij}y_j^{(t)}
    +(\eta-\eta^\ast)y_i^{(t-1)}
    = 0 
\end{equation*}
for every admissible $i$, $t$, $\y^{(t)}$, and $\y^{(t-1)}$.

We assume $\Gamma \neq 0$. There exist $i\neq j$ such that
$\gamma_{ij}\neq 0$. Consider two configurations
$\y_{-i}^{(t)}$ that differ only in the value of $y_j^{(t)}$.
Subtracting the corresponding equalities gives $2(\xi - \xi^\ast)\gamma_{ij} = 0$ so $\xi=\xi^\ast$.

Next, consider any $t\geq 2$. Since the transition distributions have full
support, every $\y^{(t-1)}\in\{-1,1\}^N$ is a possible value of
$\x^{(t-1)}$. Consider two such previous states that differ only in the
value of $y_i^{(t-1)}$. Using $\xi=\xi^\ast$ and subtracting the
corresponding equalities gives $2(\eta - \eta^\ast) = 0$ so $\eta=\eta^\ast$.

Substituting $\xi=\xi^\ast$ and $\eta=\eta^\ast$ into the equality above
yields
\begin{equation*}
    \alpha_i^{(t)}-\alpha_i^{\ast(t)}
    +(\beta-\beta^\ast)z_i^{(t)}
    =0
\end{equation*}
for every $i\in[N]$ and $t\in[T]$. Equivalently,
\begin{equation*}
    (A-A^\ast)+(\beta-\beta^\ast)Z=0.
\end{equation*}
Applying the identifiability condition \eqref{eq:identifiability}, we obtain
\begin{equation*}
    0
    =
    \left\|(A-A^\ast)+(\beta-\beta^\ast)Z\right\|_F^2
    \geq
    \varkappa\left(
        \|A-A^\ast\|_F^2
        +
        \|(\beta-\beta^\ast)Z\|_F^2
    \right).
\end{equation*}
Since $\varkappa>0$, both terms on the right-hand side must be zero. Hence,
\begin{equation*}
    A=A^\ast
    \qquad\text{and}\qquad
    (\beta-\beta^\ast)Z=0 \,\,\Rightarrow\,\, \beta = \beta^\ast.
\end{equation*}
We conclude that
\begin{equation*}
    \theta
    =
    (A,\beta,\xi,\eta)
    =
    (A^\ast,\beta^\ast,\xi^\ast,\eta^\ast)
    =
    \theta^\ast.
\end{equation*}
\end{proof}

\section{Proof of Theorem \ref{thm:param_est}} \label{app:thm_proof}

In this section, we prove Theorem \ref{thm:param_est}, restated below. Recall the definition of $\|\cdot\|_\star$ from \eqref{eq:starred_norm}.
\setcounter{theorem}{0}
\begin{theorem}
Given a single observation $(\x,\z)$ from \eqref{eq:model} with parameters 
$\textstyle \theta^\ast = (\{A^\ast = [{\alpha^\ast}_i^{(t)} : {i\in [N], t \in[T]}]\}, \beta^\ast, \xi^\ast, \eta^\ast)$, 
let $\hat{\theta} = (\hat{A}, \hat{\beta}, \hat{\xi}, \hat{\eta})$ be the parameter estimate obtained as per \eqref{eq:mple.seq}. 
Let Assumptions \ref{ass:low_rank}-\ref{ass:exc_interv} hold. 
    Then, there exists a constant $C(\rho,B) = \rho^{-4}\exp(O(B))$ 
    such that for any $\delta \in (0,1/2)$, with probability at least $1-\delta$, we have
    \begin{align*}
            \|\hat{\theta} - \theta^\ast\|_\star 
            & \leq C(\rho,B)\left(\frac{k(N + T)\log T + \log \frac{1}{\delta}}{T\|\Gamma\|_F^2}\right). 
    \end{align*}
\end{theorem}
\setcounter{theorem}{2}

The proof is divided into three steps: a curvature result which states that for any fixed $\theta$ sufficiently far from $\theta^\ast$, we have $\varphi(\theta) > \varphi(\theta^\ast)$ with sufficiently high probability; an identifiability statement proving that we can separate the effects of the external field $A$ from the interventions and the temporal couplings; and a union bound argument that extends the results of the first steps to all $\theta \in \Theta$.

We show the parts in Sections \ref{subsec:rand_dist}, \ref{subsec:det_dist}, and \ref{subsec:eps_net}. First, we define a few key terms, formalize the goals, and offer a sketch of the proofs. Let $\tilde{\theta} = \theta - \theta^\ast$. We also define the $N$-dimensional vector
\begin{equation}
    \h^{(t)} = [h_i^{(t)}]_{i \in [N]}, \qquad h_i^{(t)} = \alpha_{i}^{(t)} +  \beta z_i^{(t)} + \eta x_i^{(t-1)},
    \label{eq:h_def}
\end{equation}
which serves as an external field for the conditional distribution $\x^{(t)}|\z^{(t)},\x^{(t-1)}$. We similarly define $\tilde{\h}^{(t)}$ and ${\h^\ast}^{(t)}$. We also define two notions of `distance':
\begin{align}
    d(\theta) &= \|\tilde{\theta}\|_\star \cdot T\|\Gamma\|_F^2 = \left(\frac{\|\tilde{A}\|_F^2}{NT} + \tilde{\beta}^2 + \tilde{\xi}^2 + \tilde{\eta}^2\right)T\|\Gamma\|_F^2;
    \label{eq:d_dist_def} \\
    r(\theta) &= \sum_{t=1}^T \left\|\mathbb{E}\left[\tilde{\xi}\Gamma \x^{(t)} + \tilde{\h}^{(t)}\Big|\x^{(t-1)},\z^{(t)}\right]\right\|^2_2 + T\|\tilde{\xi}\Gamma\|^2_F.
    \label{eq:r_dist_def}
\end{align}
Note that $d(\theta)$ is deterministic and dependent only on the parameters of our model, whereas $r(\theta)$ is a random variable dependent on the data. It is relatively easy to see that a bound 
\begin{equation*}
    d(\hat{\theta}) \leq C(\rho,B)\left(k(N+T)\log T + \log 1/\delta\right)
\end{equation*}
immediately implies Theorem \ref{thm:param_est}. Indeed, this will be our goal throughout the rest of the proof. It is helpful to first consider $r(\theta)$. In particular, the first step in our proof is to bound the undesirable event that $r(\theta)$ is large and $\varphi(\theta)$ fails to separate from $\varphi(\theta^\ast)$.

\begin{lemma}[Random Pointwise Bound]
    Let $\theta \in \Theta$ be fixed. Then, there exist constants $c_1, C_1 > 0$ such that for any fixed $u_1 > 0$, we have
    \begin{equation*}
        \mathbb{P}\Big(\left\{\varphi(\theta) < \varphi(\theta^\ast) + e^{-c_1B}u_1^2\right\} \cap \{u_1^2 \leq r(\theta)\}\Big) \leq C_1\log NT \exp(-\rho^4 e^{-c_1B}u_1^2);
    \end{equation*}
    i.e. the probability that both $r(\theta) \geq u_1^2$ and $\varphi(\theta) < \varphi(\theta^\ast) + e^{-c_1B}u_1^2$ hold is at most $C_1\log NT\exp(-\rho^4 e^{-c_1B}u_1^2).$
    \label{lem:pointwise_random}
\end{lemma}
We dedicate Section \ref{subsec:rand_dist} to this result. The proof follows by rewriting the sequential MPLE as
\begin{equation*}
    \varphi(\theta) = \sum_{t=1}^T \varphi_t(\theta), 
    \quad\text{where}\quad
    \varphi_t(\theta) = - \sum_{i=1}^N \log p_\theta(\x_i^{(t)}|\x^{(t)}_{-i},\z_i^{(t)},\x_i^{(t-1)}).
\end{equation*}
By a Taylor's expansion, for every $t \in [T]$ there exist $w \in [0,1], \theta' = \theta^\ast + w(\theta - \theta^\ast)$ such that
\begin{equation*}
    \varphi_t(\theta) - \varphi_t(\theta^\ast) 
    =\tilde{\theta}^\top\nabla \varphi_t(\theta^\ast)
    + \frac{1}{2}\tilde{\theta}^\top \nabla^2\varphi_t(\theta')\,\tilde{\theta}.
\end{equation*}
Both $\tilde{\theta}^\top\nabla \varphi_t(\theta^\ast)$ and $\tilde{\theta}^\top \nabla^2\varphi_t(\theta')\,\tilde{\theta}$ are random quantities dependent on $\x^{(t)}|\z^{(t)},\x^{(t-1)}$ which is an Ising model satisfying Dobrushin's uniqueness condition. Therefore, for fixed $t$, concentration and anti-concentration results for $\varphi_t(\theta) - \varphi_t(\theta^\ast)$ are known in the literature \cite{dagan2021learning,Kandiros_2021}. We extend these results to the sequential setting with a careful martingale argument and show that
\begin{gather*}
    \sum_{t=1}^T \mathbb{E}\left[\varphi_t(\theta) - \varphi_t(\theta^\ast) \,|\,\z^{(t)},\x^{(t-1)}\right] = \Omega(r(\theta)),\\
    \left|\sum_{t=1}^T \left( \varphi_t(\theta) - \varphi_t(\theta^\ast) - \mathbb{E}\left[\varphi_t(\theta) - \varphi_t(\theta^\ast) \,|\,\z^{(t)},\x^{(t-1)}\right] \right) \right| = O(\sqrt{r(\theta)}),
\end{gather*}
after which the result follows naturally. 

The next step proves $r(\theta) \geq e^{-cB}d(\theta)$ probabilistically and uses Lemma \ref{lem:pointwise_random} to conclude that large $d(\theta)$ implies high-probability separation of $\varphi(\theta)$ from $\varphi(\theta^\ast)$. To relate $d(\theta)$ and $r(\theta)$, we argue that the ``features'' $A$, $\z$, and $\x$ are well-conditioned with high probability such that we can separate the effects of $\alpha_i^{(t)}$, $\beta z_i^{(t)}$, and $\eta x_i^{(t-1)}$. It is at this step that we use Assumption \ref{ass:exc_interv} and argue that $\x$ cannot be ``hidden'' by $A$ or $\z$ with high probability. Lemma \ref{lem:pointwise_deterministic} formalizes the result.

\begin{lemma}[Deterministic Pointwise Bound]
    Let $\theta \in \Theta$ be fixed. Then, there exist constants $c_2, C_2 > 0$ such that for any fixed $u_2 > 0$ such that $u_2^2 \leq d(\theta)$, we have
    \begin{equation*}
        \mathbb{P}\Big(\varphi(\theta) \geq \varphi(\theta^\ast) + e^{-c_2B}u_2^2\Big)
        \geq 1 - C_2\log NT \exp(-\rho^4 e^{-c_2B}u_2^2).
    \end{equation*}
    \label{lem:pointwise_deterministic}
\end{lemma}
Finally, in Section \ref{subsec:eps_net}, we show Lemma \ref{lem:d_hat_theta}. Letting $R$ denote the RHS of (\ref{eq:d_hat_bound}), the result follows by defining an $\epsilon$-net over $\Theta = \{\theta = (A,\beta,\xi,\eta): A \in [-B,B]^{N \times T}, |\beta|,|\xi|,|\eta| \leq B, \rk(A) \leq k\}$ and proving that $\varphi(\theta_\epsilon) > \varphi(\theta^\ast)$ for all $\theta_\epsilon$ in the net such that $d(\theta_\epsilon) \geq R$. Then, we extend the bound to all $\theta$ with $d(\theta) \geq R$ using the Lipschitzness of $\varphi(\theta)$ and appropriate choice of $\epsilon$ in the net. The result follows by the contrapositive. Theorem \ref{thm:param_est} is a direct consequence of Lemma \ref{lem:d_hat_theta}.

\begin{lemma}[Bound on $d(\hat{\theta})$]
    Let $\hat{\theta}$ be the MPLE estimate of the true parameters. Then, there exists a constant $C(\rho,B)=\rho^{-4}\exp(O(B))$ such that, with probability $1-\delta$, we have
    \begin{equation}
        d(\hat{\theta}) \leq C(\rho,B)\left(k(N + T)\log T + \log \frac{1}{\delta}\right).
        \label{eq:d_hat_bound}
    \end{equation}
    \label{lem:d_hat_theta}
\end{lemma}

\begin{proof}[Proof of Theorem \ref{thm:param_est}]
    By Lemma \ref{lem:d_hat_theta}, we have
    \begin{equation*}
        d(\hat{\theta}) \leq C(\rho,B)\left(k(N + T)\log T + \log \frac{1}{\delta}\right),
    \end{equation*}
    with probability at least $1-\delta$. Recalling the definition of $d(\theta)$ in \eqref{eq:d_dist_def} and $\|\cdot\|_\star$ in \eqref{eq:starred_norm}, we have
    \begin{equation*}
        \|\hat{\theta} - \theta^\ast\|_\star \cdot T\|\Gamma\|_F^2 \leq C(\rho,B)\left(k(N + T)\log T + \log \frac{1}{\delta}\right),
    \end{equation*}
    on the same event. Dividing both sides by $T\|\Gamma\|_F^2$ gives the desired result.
\end{proof}

\subsection{Pointwise Random Separation} \label{subsec:rand_dist}

We prove Lemma \ref{lem:pointwise_random}, restated below.
\setcounter{lemma}{0}
\begin{lemma}[Random Pointwise Bound]
    Let $\theta \in \Theta$ be fixed. Then, there exist constants $c_1, C_1 > 0$ such that for any fixed $u_1 > 0$, we have
    \begin{equation*}
        \mathbb{P}\Big(\left\{\varphi(\theta) < \varphi(\theta^\ast) + e^{-c_1B}u_1^2\right\} \cap \{u_1^2 \leq r(\theta)\}\Big) \leq C_1\log NT \exp(-\rho^4 e^{-c_1B}u_1^2);
    \end{equation*}
    i.e., the probability that both $r(\theta) \geq u_1^2$ and $\varphi(\theta) < \varphi(\theta^\ast) + e^{-c_1B}u_1^2$ hold is at most $C_1\log NT\exp(-\rho^4 e^{-c_1B}u_1^2).$
\end{lemma}
\setcounter{lemma}{3}

We rewrite our sequential MPLE as
\begin{equation*}
    \varphi(\theta) = \sum_{t=1}^T \varphi_t(\theta), 
    \quad\text{where}\quad
    \varphi_t(\theta) = - \sum_{i=1}^N \log p_\theta(\x_i^{(t)}|\x^{(t)}_{-i},\z_i^{(t)},\x_i^{(t-1)}).
\end{equation*}
By a Taylor expansion, for each fixed $\theta$ and $t \in [T]$ there exist $w \in [0,1], \theta' = \theta^\ast + w(\theta - \theta^\ast)$ such that
\begin{equation}
    \varphi_t(\theta) - \varphi_t(\theta^\ast) 
    =\tilde{\theta}^\top\nabla \varphi_t(\theta^\ast)
    + \frac{1}{2}\tilde{\theta}^\top \nabla^2\varphi_t(\theta')\,\tilde{\theta}.
    \label{eq:taylor}
\end{equation}
Recalling the definition of $h$ in \eqref{eq:h_def}, we write after some calculations
\begin{equation}
    \tilde{\theta}^\top\nabla \varphi_t(\theta^\ast) = 
    \sum_{i=1}^N \big[\tanh(\xi^\ast \Gamma_i \x^{(t)} + {h^\ast}_i^{(t)}) - x_i^{(t)}\big]
    (\tilde{h}_i^{(t)} + \tilde{\xi} \Gamma_i \x^{(t)}) = S_t,
    \label{eq:score}
\end{equation}
and, for $\|v\|_\infty$ representing the maximum absolute entry for a vector $v$ and for some constant $c$,
\begin{align}
    \frac{1}{2}\tilde{\theta}^\top \nabla^2\varphi_t(\theta')\,\tilde{\theta} &=  
    \frac{1}{2}\sum_{i=1}^N\mathrm{sech}^2(\xi' \Gamma_i \x^{(t)} + {h'}_i^{(t)})
    (\tilde{h}_i^{(t)} + \tilde{\xi} \Gamma_i \x^{(t)})^2 \nonumber \\
    &\geq \frac{1}{2}\mathrm{sech}^2(\|\xi' \Gamma \x^{(t)}\|_\infty + \|{\h'}^{(t)}\|_\infty)
    \sum_{i=1}^N(\tilde{h}_i^{(t)} + \tilde{\xi} \Gamma_i \x^{(t)})^2 \nonumber \\
    &\geq e^{-cB}\left\|\tilde{\xi} \Gamma \x^{(t)} + \tilde{\h}^{(t)}\right\|_2^2 = H_t,
    \label{eq:hessian}
\end{align}
where the last inequality follows since $ \|\xi' \Gamma \x^{(t)}\|_\infty$ and $\|{\h'}^{(t)}\|_\infty$  are $O(B)$.
We want to argue that $\sum_{t=1}^T S_t + H_t$ is large with high probability. We define $\mathcal{G}_t = \sigma(\x^{(1)},\dots,\x^{(t)},\z)$ and
\begin{equation}
    Q^2_t = \left\|\mathbb{E}\left[\tilde{\xi}\Gamma \x^{(t)} + \tilde{\h}^{(t)}\Big|\x^{(t-1)},\z^{(t)}\right]\right\|_2^2 + \|\tilde{\xi}\Gamma\|_F^2,
    \label{eq:q_t_def}
\end{equation}
which is $\mathcal{G}_{t-1}$-measurable; note $\sum_{t=1}^T Q_t^2 = r(\theta)$ where $r(\theta)$ is defined in \eqref{eq:r_dist_def}. The conditional distributions of $S_t$ and $H_t$ are functions of $\x^{(t)} | \mathcal{G}_{t-1}$ which is an Ising model satisfying Dobrushin's uniqueness condition. Previous work has established that the conditional mean $\mathbb{E}[S_t + H_t | \mathcal{G}_{t-1}]$ is of order $Q_t^2$ and that $S_t$ and $H_t$ concentrate around their conditional means with radius $Q_t$. To prove Lemma \ref{lem:pointwise_random}, we extend these arguments to the sequential setting. Lemma \ref{lem:cond_means_lb} argues that $\sum_{t=1}^T \mathbb{E}[S_t + H_t | \mathcal{G}_{t-1}] \geq \sum_{t=1}^T Q_t^2 = r(\theta)$; Lemma \ref{lem:martingale_conc} shows that $\sum_t (S_t + H_t)$ concentrates at radius $\sqrt{r(\theta)}$. The proofs are deferred to Appendix \ref{app:aux_lemmas}. 
The proof of Lemma \ref{lem:pointwise_random} synthesizes these results.
\begin{lemma} Let $\theta \in \Theta$ be fixed. There exists a constant $c_4 > 0$ such that
    \begin{equation*}
        \sum_{t=1}^T \mathbb{E}[S_t + H_t | \mathcal{G}_{t-1}] \geq e^{-c_4B} \cdot r(\theta).
    \end{equation*}
    \label{lem:cond_means_lb}
\end{lemma}

\begin{lemma}
    Let $\theta \in \Theta$ be fixed. There exist constants $c_5,c_5',C_5 > 0$ such that for all $u_5 \geq 0$,
    \begin{multline*}
        \mathbb{P}\left(
        \left|\sum_{t=1}^T \Big(S_t + H_t - \mathbb{E}[S_t + H_t | \mathcal{G}_{t-1}] \Big) \right|
        \le
        \rho^{-2}e^{c_5B}\max\left\{
        u_5\sqrt{r(\theta)},
        u_5^2
        \right\}
        \right) \\
        \geq 1-C_5\log NTe^{-c_5'u_5^2}.
    \end{multline*}
    \label{lem:martingale_conc}
\end{lemma}

\begin{proof}[Proof of Lemma \ref{lem:pointwise_random}]
    Given $u_1 > 0$, let
    \begin{equation*}
        \mathcal{E}_\varphi = \big\{\varphi(\theta) - \varphi(\theta^\ast) \geq e^{-c_1B}u_1^2\big\}
        \qquad\text{and}\qquad
        \mathcal{E}_r = \{u_1^2 \leq r(\theta)\}.
    \end{equation*}
    We aim to show that $\mathbb{P}\left(\mathcal{E}_\varphi^c \cap \mathcal{E}_r\right) \leq C_1\log NT \exp(-\rho^{4}e^{-c_1B}u_1^2)$. 
    
    As in \eqref{eq:taylor}, we write
    \begin{multline*}
        \varphi(\theta) - \varphi(\theta^\ast) = \sum_{t=1}^T \varphi_t(\theta) - \varphi_t(\theta^\ast) 
        = \sum_{t=1}^T \left(\tilde{\theta}^\top\nabla \varphi_t(\theta^\ast)
        + \frac{1}{2}\tilde{\theta}^\top \nabla^2\varphi_t(\theta')\,\tilde{\theta}\right) 
        \\ \geq \sum_{t=1}^T S_t + H_t
        = \sum_{t=1}^T \mathbb{E}[S_t + H_t | \mathcal{G}_{t-1}] + \sum_{t=1}^T \Big(S_t + H_t - \mathbb{E}[S_t + H_t | \mathcal{G}_{t-1}]\Big).
    \end{multline*}
    By Lemmas \ref{lem:cond_means_lb} and \ref{lem:martingale_conc}, for any fixed $u_5 > 0$ with probability at least $1 - C_5\log NTe^{-c_5'u_5^2}$,
    we have 
    \begin{equation}
        \varphi(\theta) - \varphi(\theta^\ast) \geq e^{-c_4B} \cdot r(\theta) - \rho^{-2}e^{c_5B}\max\left\{
        u_5\sqrt{r(\theta)},
        u_5^2
        \right\}.
        \label{eq:varphi_theta_diff}
    \end{equation}
    Choose $u_5 = \rho^2 e^{-(c_4+c_5)B}u_1/2$ and denote by $\mathcal{E}_{4\cap5}$ the event of \eqref{eq:varphi_theta_diff} with this choice of $u_5$. We claim
    \begin{equation*}
        \mathcal{E}_r \cap \mathcal{E}_{4\cap 5} \subseteq \mathcal{E}_{\varphi}.
    \end{equation*}
    It suffices to prove that when $u_1^2 \leq r(\theta)$, we have
    \begin{equation*}
        \rho^{-2}e^{c_5B}\max\left\{
        u_5\sqrt{r(\theta)},
        u_5^2
        \right\}
        \leq \frac{1}{2}e^{-c_4B}r(\theta)
    \end{equation*}
    for our choice of $u_5$. First,
    \begin{equation*}
        \rho^{-2}e^{c_5B}u_5\sqrt{r(\theta)} \leq \frac{1}{2}e^{-c_4B} u_1 \sqrt{r(\theta)} \leq \frac{1}{2}e^{-c_4B}r(\theta).
    \end{equation*}
    Then, since $B \geq 1$ and $\rho \leq 1$,
    \begin{equation*}
        \rho^{-2}e^{c_5B}u_5^2 = \frac{1}{4}\rho^{2}e^{-(2c_4+c_5)B}u_1^2 \leq \frac{1}{4}e^{-c_4B}r(\theta).
    \end{equation*}
    Therefore, \eqref{eq:varphi_theta_diff} reduces to
    \begin{equation*}
        \varphi(\theta) - \varphi(\theta^\ast) \geq \frac{1}{2}e^{-c_4B} r(\theta) \geq \frac{1}{2}e^{-c_4B}u_1^2 \geq e^{-c_1B}u_1^2,
    \end{equation*}
    where $c_1$ is sufficiently large. Thus, $\mathcal{E}_r \cap \mathcal{E}_{4\cap 5} \subseteq \mathcal{E}_{\varphi}$. Then, it is also the case that $\mathcal{E}_r \cap \mathcal{E}_{\varphi}^c \subseteq \mathcal{E}_{4\cap 5}^c$. Recalling our choice of $u_5$ in defining $\mathcal{E}_{4\cap5}$, we have
    \begin{equation*}
        \mathbb{P}(\mathcal{E}_{4\cap5}^c) \leq C_5\log NT \exp\left(-\frac{c_5'\rho^4 e^{-2(c_4+c_5)B}}{4} u_1^2 \right) \leq C_1\log NT \exp(-\rho^4 e^{-c_1B}u_1^2),
    \end{equation*}
    after selecting $C_1$ and increasing $c_1$ if necessary. Therefore,
    \begin{equation*}
        \mathbb{P}(\mathcal{E}_r \cap \mathcal{E}_{\varphi}^c) \leq \mathbb{P}(\mathcal{E}_{4\cap5}^c) \leq C_1\log NT \exp(-\rho^4 e^{-c_1B}u_1^2).
    \end{equation*}
\end{proof}
\subsection{Pointwise Deterministic Separation} \label{subsec:det_dist}

We now turn towards proving Lemma \ref{lem:pointwise_deterministic}, restated below. This translates our bound on the separation between $\varphi(\theta)$ and $\varphi(\theta^\ast)$ into one that more immediately implies parameter identification.
\setcounter{lemma}{1}
\begin{lemma}[Deterministic Pointwise Bound]
   Let $\theta \in \Theta$ be fixed. Then, there exist constants $c_2, C_2 > 0$ such that for any fixed $u_2 > 0$ such that $u_2^2 \leq d(\theta)$, we have
    \begin{equation*}
        \mathbb{P}\Big(\varphi(\theta) \geq \varphi(\theta^\ast) + e^{-c_2B}u_2^2\Big)
        \geq 1 - C_2\log NT \exp(-\rho^4 e^{-c_2B}u_2^2).
    \end{equation*}
\end{lemma}
\setcounter{lemma}{5}

The result follows naturally from Lemma \ref{lem:pointwise_random} and Lemma \ref{lem:rand_to_det} below.

\begin{lemma}
    Let $d(\theta)$ and $r(\theta)$ be as defined in \eqref{eq:d_dist_def} and \eqref{eq:r_dist_def}. Then, there exist constants $c_6, C_6 > 0$ such that
    \begin{equation*}
        r(\theta) \geq e^{-c_6B}\,d(\theta),
    \end{equation*}
    with probability at least $1 - C_6 \exp(-\rho e^{-c_6B}\,d(\theta))$.
    \label{lem:rand_to_det}
\end{lemma}

\begin{proof}[Proof of Lemma \ref{lem:pointwise_deterministic}]
    Let $u_1^2 = e^{-c_6B}d(\theta)$. Lemma \ref{lem:pointwise_random} states
    \begin{multline}
        \mathbb{P}\Big(\left\{\varphi(\theta) \geq \varphi(\theta^\ast) + e^{-(c_1+c_6)B}d(\theta)\right\} 
        \cup \{r(\theta) < e^{-c_6B}d(\theta)\}\Big)  \\
        \geq 1 -  C_1\log NT \exp(-\rho^4 e^{-(c_1+c_6)B}d(\theta)).
        \label{eq:lemma1event}
    \end{multline}
    By a union bound,
    \begin{multline*}
        \mathbb{P}\Big(\varphi(\theta) \geq \varphi(\theta^\ast) + e^{-(c_1+c_6)B}d(\theta)\Big) \geq 
        \\ \mathbb{P}\Big(\left\{\varphi(\theta) \geq \varphi(\theta^\ast) + e^{-(c_1+c_6)B}d(\theta)\right\} 
        \cup \{r(\theta) < e^{-c_6B}d(\theta)\}\Big) - \mathbb{P}\big(r(\theta) < e^{-c_6B}d(\theta)\big).
    \end{multline*}
    Substituting in bounds from \eqref{eq:lemma1event} and Lemma \ref{lem:rand_to_det} gives
    \begin{equation*}
        \mathbb{P}\Big(\varphi(\theta) \geq \varphi(\theta^\ast) + e^{-(c_1+c_6)B}d(\theta)\Big) 
        \geq 1 - (C_1+C_6)\log NT\exp(-\rho^4 e^{-(c_1+c_6)B}d(\theta)).
    \end{equation*}
    Letting $c_2,C_2>0$ be suitably large and dependent on $c_1,C_1,c_6,C_6$ gives the result since $\rho \leq 1$ and replacing $d(\theta)$ by $u_2^2 \leq d(\theta)$ only weakens the lower bound and the probability statement.
\end{proof}

We dedicate the rest of the section to proving Lemma \ref{lem:rand_to_det}. Recalling \eqref{eq:d_dist_def}, define $d_h$ and $d_\xi$ as 
\begin{equation*} 
    d(\theta) = \underbrace{\left(\frac{\|\tilde A\|_F^2}{TN} + \tilde{\beta}^2  + \tilde{\eta}^2\right)T\|\Gamma\|_F^2}_{d_h} + \underbrace{T\tilde{\xi}^2\|\Gamma\|_F^2}_{d_\xi}.
\end{equation*}
with $r(\theta)$ already defined as
\begin{equation*}
    r(\theta) = \sum_{t=1}^T \left\|\mathbb{E}\left[\tilde{\xi}\Gamma \x^{(t)} + \tilde{\h}^{(t)}\Big|\x^{(t-1)},\z^{(t)}\right]\right\|^2_2 + T\|\tilde{\xi}\Gamma\|^2_F.
\end{equation*}
To argue that $d(\theta)$, which decouples the effects of $A$, $\beta$, and $\eta$, is a lower bound on $r(\theta)$, we need to show that the three external-field features $\alpha_i^{(t)}$, $z_i^{(t)}$, and $x_i^{(t-1)}$ are well-conditioned along the observed trajectory. The result is formalized in Lemma \ref{lem:ext_field_identifiability} and uses similar techniques to the proof of Lemma \ref{lem:pointwise_random}. We defer the proof to Appendix \ref{app:aux_lemmas}. The proof of Lemma \ref{lem:rand_to_det} then follows by a case analysis.

\begin{lemma}
    Let $\theta\in\Theta$ be fixed and recall $A$ is the $N \times T$ dimensional matrix that comprises $\alpha_i^{(t)}$.
    Then there exist constants $c_7,C_7>0$ such that, with probability at least
    \begin{equation*}
        1-C_7\exp\big(-\rho e^{-c_7B} d_h\big),
    \end{equation*}
    we have
    \begin{equation*}
        \sum_{t=1}^T \|\tilde{\h}^{(t)}\|_2^2
        \geq e^{-c_7B}\Big(\|\tilde{A}\|_F^2 + TN(\tilde{\beta}^2 + \tilde{\eta}^2) \Big)
        = e^{-c_7B}d_h \cdot \frac{N}{\|\Gamma\|_F^2}.
    \end{equation*}
    \label{lem:ext_field_identifiability}
\end{lemma}

\begin{proof}[Proof of Lemma \ref{lem:rand_to_det}]
    Recall that our goal is to prove
    \begin{equation*}
        r(\theta) \geq e^{-c_6B}d(\theta)
    \end{equation*}
    with probability at least $1 - C_6 \exp(-\rho e^{-c_6B}d(\theta))$ for suitable constants $c_6, C_6 > 0$. We choose a constant $c$ sufficiently large (dependent on $c_7$) and divide into cases.

    \textit{Case 1: $d_\xi \geq e^{-cB} d_h$.} Then, we already have
    \begin{equation*}
        r(\theta) \geq d_\xi \geq e^{-c_6'B}d(\theta),
    \end{equation*}
    deterministically for $c_6'$ sufficiently large by the assumptions of our case analysis.

    \textit{Case 2: $d_h \geq e^{cB}d_\xi$.} Choose $c$ large enough such that $e^{-cB} \leq \frac{1}{4}e^{-c_7B}$, where $c_7$ is the constant from Lemma \ref{lem:ext_field_identifiability}. Furthermore, note that for each row $i$, the normalization $\|\Gamma\|_\infty=1$ implies $\sum_j |\gamma_{ij}|\le 1$, so $|\Gamma_i\x^{(t)}|\le 1$ because $\x^{(t)}\in\{-1,1\}^N$. Hence $\|\mathbb{E}[\Gamma\x^{(t)}|\x^{(t-1)},\z^{(t)}]\|_2^2\le N$ for every $t$. Using the elementary inequality $\|a+b\|_2^2\ge \frac12\|b\|_2^2-\|a\|_2^2$, on the event of Lemma \ref{lem:ext_field_identifiability} it is the case that
    \begin{multline*}
        r(\theta) \geq \sum_{t=1}^T \left\|\mathbb{E}\left[\tilde{\xi}\Gamma \x^{(t)} + \tilde{\h}^{(t)}\Big|\x^{(t-1)},\z^{(t)}\right]\right\|^2_2 
        \geq \frac{1}{2} \sum_{t=1}^T \|\tilde \h^{(t)}\|_2^2 - \sum_{t=1}^T \|\mathbb{E}[\tilde{\xi} \Gamma \x^{(t)}|\x^{(t-1)},\z^{(t)}]\|_2^2
        \\ \geq \frac{N}{\|\Gamma\|_F^2}\left(\frac{1}{2}e^{-c_7B} d_h - d_\xi\right).
    \end{multline*}
    By the case assumption and our choice of $c$, the term in the parenthesis is at least $\frac{1}{4}e^{-c_7B} d_h$. Since $\|\Gamma\|_F^2 \leq N$ given our assumption $\|\Gamma\|_\infty = 1$, this implies that
    \begin{equation*}
        r(\theta) \geq \frac{1}{4}e^{-c_7B} d_h \geq e^{-c_6''B}d(\theta),
    \end{equation*}
    on the event of Lemma \ref{lem:ext_field_identifiability} for appropriate choice of $c_6''$. Let $c_6 = \max\{c_6',c_6''\}$. The success probability of this event is at least $1 - C_6 \exp(-\rho e^{-c_6B}d(\theta))$ again since $d_h$ dominates $d(\theta)$ in this case.
\end{proof}

\subsection{Bounding $d(\hat{\theta})$} \label{subsec:eps_net}

In this section, we complete the proof of Theorem \ref{thm:param_est} by translating the pointwise bound of Lemma \ref{lem:pointwise_deterministic} into a uniform statement over a set of all parameters in 
\begin{equation*}
    \Theta = \{\theta = (A,\beta,\xi,\eta): A \in [-B,B]^{N \times T}, |\xi| \leq 1-\rho, |\beta|,|\eta| \leq B, \rk(A) \leq k\}.
\end{equation*}
The only substantive difference from a standard finite-dimensional $\epsilon$-net argument is that the matrix component $A$ lives in an $NT$-dimensional ambient space. We handle this using the low-rank constraint $\rk(A)\le k$, which reduces the covering entropy to order $k(N+T)$. This will allow us to conclude via the contrapositive (since $\varphi(\hat{\theta}) \leq \varphi(\theta^\ast)$ by definition) that $\hat{\theta}$ is close to $\theta^\ast$ with high probability. In the rest of this section, we focus on the proof of Lemma \ref{lem:d_hat_theta} from which the proof of Theorem \ref{thm:param_est} follows immediately. We begin by proving a Lipschitzness result.

\begin{lemma}
    Let $\theta_1=(A_1,\beta_1,\xi_1,\eta_1), \theta_2=(A_2,\beta_2,\xi_2,\eta_2) \in \Theta$. Then,
    \begin{equation*}
        |\varphi(\theta_1) - \varphi(\theta_2)| \leq 2\Big(\|A_1-A_2\|_{1} + NT\big(|\beta_1-\beta_2|+|\xi_1-\xi_2|+|\eta_1-\eta_2|\big)\Big).
    \end{equation*}
    \label{lem:lipschitz}
\end{lemma}

\begin{proof}[Proof of Lemma \ref{lem:lipschitz}]
    Define $f(t)=\varphi(\theta_2+t(\theta_1-\theta_2))$. By the mean value theorem,
    \begin{equation*}
        |\varphi(\theta_1)-\varphi(\theta_2)|
        \le \sup_{t\in[0,1]}\left| \frac{d}{dt} f(t)\right|
        = \sup_{t\in[0,1]}\left| \langle \theta_1-\theta_2,\nabla\varphi(\theta_2+t(\theta_1-\theta_2))\rangle\right|.
    \end{equation*}
    From the explicit score formula \eqref{eq:score}, each derivative with respect to an entry $\alpha_i^{(t)}$ has absolute value at most $2$, and the derivatives with respect to $\beta,\xi,\eta$ are each bounded by $2NT$ because $|z_i^{(t)}|,|x_i^{(t-1)}|,|\Gamma_i \x^{(t)}|\le 1$. The stated bound follows immediately.
\end{proof}

We next generalize our $d(\theta)$ definition in \eqref{eq:d_dist_def} and define the deterministic distance between two arbitrary parameter values:
\begin{equation}
    d(\theta,\theta')
    :=
    T\|\Gamma\|_F^2\left(
        \frac{\|A-A'\|_F^2}{TN}
        +(\beta-\beta')^2
        +(\xi-\xi')^2
        +(\eta-\eta')^2
    \right).
    \label{eq:dist_det_gen_new}
\end{equation}
In particular, $d(\theta,\theta^\ast)=d(\theta)$. Since \eqref{eq:dist_det_gen_new} is simply a weighted Euclidean distance, $\sqrt{d(\theta,\theta')}$ obeys the triangle inequality. We are now ready to prove Lemma \ref{lem:d_hat_theta}.

\setcounter{lemma}{2}
\begin{lemma}[Bound on $d(\hat{\theta})$]
    Let $\hat{\theta}$ be the MPLE estimate of the true parameters. Recall the definition of $d(\theta)$ in (\ref{eq:d_dist_def}). Then, there exists a constant $C(B)$ exponentially dependent on $B$ such that, with probability $1-\delta$, we have
    \begin{equation*}
        d(\hat{\theta}) \leq C(\rho,B)\left(k(N+T)\log T + \log \frac{1}{\delta}\right).
    \end{equation*}
\end{lemma}
\setcounter{lemma}{8}

\begin{proof}[Proof of Lemma \ref{lem:d_hat_theta}]
    Let
    \begin{equation*}
        R := C(\rho,B)\Big(k(N+T)\log T + \log \frac{1}{\delta}\Big),
    \end{equation*}
    where $C(\rho,B)$ contains a sufficiently large singly-exponential function of $B$ to be chosen below. We will show that, with probability at least $1-\delta$,
    \begin{equation*}
        \varphi(\theta)>\varphi(\theta^\ast)
        \qquad\text{for every }\theta\in\Theta\text{ such that }d(\theta)\ge R.
    \end{equation*}
    Since $\hat\theta$ minimizes $\varphi$ over $\Theta$, this immediately implies $d(\hat\theta)\le R$.

    \smallskip
    \textit{Step 1: Constructing the net.}
    Let $\mathcal N_A(\varepsilon_A)$ be an $\varepsilon_A$-net, in Frobenius norm, of the set
    \begin{equation*}
        \{A\in[-B,B]^{N\times T}:\operatorname{rank}(A)\le k\}.
    \end{equation*}
    By a standard covering-number for rank $k$ matrices (see Lemma 3.1 of \cite{candes2011tight}), we can bound
    \begin{equation}
        \log |\mathcal N_A(\varepsilon_A)| \le C\,k(N+T)\log\!\left(\frac{CB\sqrt{NT}}{\varepsilon_A}\right).
        \label{eq:rank_k_entropy}
    \end{equation}
    Likewise, let $\mathcal N_s(\varepsilon_s)$ be an $\varepsilon_s$-net of $[-B,B] \times [-1+\rho,1-\rho] \times [-B,B]$ in Euclidean norm, so that
    \begin{equation}
        \log |\mathcal N_s(\varepsilon_s)| \le C\log\!\left(\frac{CB}{\varepsilon_s}\right).
        \label{eq:standard_net_entropy}
    \end{equation}
    We define the product net
    \begin{equation*}
        \mathcal N(\varepsilon_A,\varepsilon_s)
        := \mathcal N_A(\varepsilon_A)\times \mathcal N_s(\varepsilon_s)\subset\Theta.
    \end{equation*}
    For every $\theta=(A,\beta,\xi,\eta)\in\Theta$, there exists $\theta_\varepsilon=(A_\varepsilon,\beta_\varepsilon,\xi_\varepsilon,\eta_\varepsilon)\in\mathcal N(\varepsilon_A,\varepsilon_s)$ such that
    \begin{equation*}
        \|A-A_\varepsilon\|_F\le \varepsilon_A,
        \qquad
        \|(\beta,\xi,\eta)-(\beta_\varepsilon,\xi_\varepsilon,\eta_\varepsilon)\|_1\le \varepsilon_s.
    \end{equation*}
    It will suffice to consider $\epsilon_A = \sqrt{N/T}$ and $\epsilon_s = 1/T$ for which \eqref{eq:rank_k_entropy} and \eqref{eq:standard_net_entropy} give
    \begin{equation}
        \log\left(|\mathcal{N}(\varepsilon_A, \varepsilon_s)|\right) \leq Ck(N+T)\log \left(CBT\right) + C\log(CBT).
        \label{eq:product_net_entropy}
    \end{equation}

    \smallskip
    \textit{Step 2: Pointwise separation on the net.}
    By Lemma \ref{lem:pointwise_deterministic}, for each fixed $\theta_\varepsilon\in\Theta$ with $d(\theta_\varepsilon)\ge R/4$,
    \begin{equation*}
        \varphi(\theta_\varepsilon)\ge \varphi(\theta^\ast)+e^{-c_2B}d(\theta_\varepsilon)
        \geq \varphi(\theta^\ast)+e^{-c_2B}R/4,
    \end{equation*}
    except on an event of probability at most
    \begin{equation*}
        C_2\log NT \exp(-\rho^4e^{-c_2B}R/4).
    \end{equation*}
    Taking a union bound over $\mathcal N(\varepsilon_A,\varepsilon_s)$, this fails for some net point only with probability at most
    \begin{equation}
        C_2|\mathcal N(\varepsilon_A,\varepsilon_s)|\log NT\exp(-\rho^{4}e^{-c_2B}R/4).
        \label{eq:eps_net_prob_bound}
    \end{equation}
    From \eqref{eq:product_net_entropy},
    \begin{multline*}
        \log\left(C_2|\mathcal{N}(\varepsilon_A, \varepsilon_s)|\log NT\right) \leq Ck(N+T)\log \left(CBT\right) + C\log(CBT) + \log \log NT \\ \leq e^{cB}k(N+T)\log T,
    \end{multline*}
    for some constant $c$. For the dependence on $B$ chosen big enough in the expression of $R$,
    \begin{equation*}
        \exp(-\rho^{4}e^{-c_2B}R/4) \leq \exp(-e^{cB}k(N+T)\log T - \log \frac{1}{\delta}).
    \end{equation*}
    Then, (\ref{eq:eps_net_prob_bound}) reduces to 
    \begin{equation*}
        C_2|\mathcal N(\varepsilon_A,\varepsilon_s)|\log NT\exp(-\rho^{4}e^{-c_2B}R/4) \leq \delta,
    \end{equation*}
    so we have
    \begin{equation}
        \varphi(\theta_\varepsilon) \geq \varphi(\theta^\ast)+e^{-c_2B}R/4
        \label{eq:good_event_step_2}
    \end{equation}
    for all $\theta_\varepsilon \in \Theta$ such that $d(\theta_\epsilon) \geq R/4$ with probability at least $1 - \delta$.

    \smallskip
    \textit{Step 3: Passing from the net to all of $\Theta$.}
    Because $\sqrt{d(\cdot,\cdot)}$ observes the triangle inequality, for any $\theta\in\Theta$ and nearest net point $\theta_\varepsilon$ we have
    \begin{equation*}
        \sqrt{d(\theta_\varepsilon)}
        \ge \sqrt{d(\theta)}-\sqrt{d(\theta,\theta_\varepsilon)}.
    \end{equation*}
    Let $C(B) \geq 4$, enlarging if necessary. Then, if $d(\theta)\ge R$, since $\epsilon_A = \sqrt{N/T}$, $\epsilon_s = 1/T$ we have
    \begin{equation*}
        d(\theta,\theta_\varepsilon)
        \le
        T\|\Gamma\|_F^2\left(\frac{\varepsilon_A^2}{TN}+\varepsilon_s^2\right)
        \le \|\Gamma\|_F^2 \leq N \leq \frac{R}{4}.
    \end{equation*}
    Thus, $d(\theta_\varepsilon)\ge \frac{R}{4}$ and so on the good event from Step 2 (Equation \eqref{eq:good_event_step_2}) it follows that
    \begin{equation*}
        \varphi(\theta_\varepsilon)\ge \varphi(\theta^\ast)+e^{-c_2B}R/4.
    \end{equation*}
    By Lemma \ref{lem:lipschitz},
    \begin{equation*}
        |\varphi(\theta)-\varphi(\theta_\varepsilon)|
        \le 2\Big(\|A-A_\varepsilon\|_1 + NT\big(|\beta-\beta_\varepsilon|+|\xi-\xi_\varepsilon|+|\eta-\eta_\varepsilon|\big)\Big).
    \end{equation*}
    Let $C(\rho,B) \geq 50e^{c_2B}$, enlarging if necessary. Using $\|A-A_\varepsilon\|_1\le \sqrt{NT}\|A-A_\varepsilon\|_F\le \sqrt{NT}\,\varepsilon_A$ and $\varepsilon_A = \sqrt{N/T}$, $\varepsilon_s=1/T$, we have
    \begin{equation*}
        |\varphi(\theta)-\varphi(\theta_\varepsilon)|\le 5N \leq e^{-c_2B}R/8.
    \end{equation*}
    Therefore
    \begin{equation*}
        \varphi(\theta)
        \ge \varphi(\theta_\varepsilon)-|\varphi(\theta)-\varphi(\theta_\varepsilon)|
        \ge \varphi(\theta^\ast)+e^{-c_2B}R/8
        > \varphi(\theta^\ast).
    \end{equation*}

    We have thus shown that every $\theta\in\Theta$ satisfying $d(\theta)\ge R$ has strictly larger pseudo-likelihood than $\theta^\ast$ on an event of probability $1-\delta$. Since $\hat\theta$ minimizes $\varphi$ over $\Theta$, we have $\varphi(\hat{\theta}) \leq \varphi(\theta^\ast)$. Therefore, $d(\hat\theta)\le R$ by the contrapositive.
\end{proof}

\section{Auxiliary Lemmas} \label{app:aux_lemmas}

We prove the auxiliary lemmas (Lemmas \ref{lem:cond_means_lb}, \ref{lem:martingale_conc}, and \ref{lem:ext_field_identifiability}) used in the proof of Theorem \ref{thm:param_est}. Lemmas \ref{lem:cond_means_lb} and \ref{lem:martingale_conc} are used to prove Lemma \ref{lem:pointwise_random}. Lemma \ref{lem:ext_field_identifiability} is used in the proof of Lemma \ref{lem:pointwise_deterministic}. Since Lemmas \ref{lem:pointwise_random} and \ref{lem:ext_field_identifiability} follow from similar arguments, we begin with a proof overview.

\subsection{Proof Overview} \label{subsec:proof_overview}

Viewed abstractly, Lemmas \ref{lem:pointwise_random} and \ref{lem:ext_field_identifiability} aim to provide a lower bound on $\sum_{t=1}^T Y_t$, where $Y_t = \varphi_t(\theta) - \varphi_t(\theta^\ast)$ in Lemma \ref{lem:pointwise_random} and $Y_t = \|\tilde{\h}^{(t)}\|_2^2$ in Lemma \ref{lem:ext_field_identifiability}. We choose an appropriate filtration and write
\begin{equation*}
    \sum_{t=1}^T Y_t = \sum_{t=1}^T \mathbb{E}[Y_t | \mathcal{F}_{t-1}] + \sum_{t=1}^T \big(Y_t - \mathbb{E}[Y_t | \mathcal{F}_{t-1}]\big) = \sum_{t=1}^T \mathbb{E}[Y_t | \mathcal{F}_{t-1}] + \sum_{t=1}^T D_t,
\end{equation*}
where $D_t = Y_t - \mathbb{E}[Y_t | \mathcal{F}_{t-1}]$.
Since $Y_t \,|\, \mathcal{F}_{t-1}$ is in both cases a function of an Ising model under Dobrushin's uniqueness condition, existing results provide a lower bound for $\mathbb{E}[Y_t|\mathcal{F}_{t-1}]$ and a concentration result for $Y_t$ around its conditional mean (equivalently, a bound on $|D_t|$), for each $t$. 

A lower bound on $\sum_{t=1}^T \mathbb{E}[Y_t\,|\,\mathcal{F}_{t-1}]$ follows immediately. To argue that $|\sum_{t=1}^T D_t|$ is sufficiently small with high probability, we prove that the conditional concentration bound for $D_t$ implies a bound on the conditional Moment Generating Function (MGF),  $\mathbb{E}[e^{\lambda D_t} \,|\, \mathcal{F}_{t-1}]$. Then, we use our filtration to bound the MGF of $\sum_{t=1}^T D_t$, which implies a tail bound on the sum. If the bound on the conditional MGF of $D_t$ is in terms of a random quantity (as will be the case for Lemma \ref{lem:pointwise_random}), we will need an additional peeling argument. 

We state the key pieces generally. We begin with the relevant anti-concentration and concentration results for Ising models under Dobrushin's. Lemmas \ref{lem:anti-concentration} and Lemma \ref{lem:ising_model_conc} are Lemmas 6 and 24 in \cite{dagan2021learning}; Lemma \ref{lem:quadr_conc} follows from the same proof as Lemma  8 in \cite{dagan2021learning} and is also stated in the proof of Lemma 14 in \cite{Kandiros_2021}. Finally, Lemma \ref{lem:affine_conc} follows from Theorem 4.3 of \cite{chatterjee2005concentration}. 

For Lemmas \ref{lem:ising_model_conc} and \ref{lem:quadr_conc}, the original statements in \cite{dagan2021learning} do not expose dependence on $\rho$ and $B$. To track this dependence, we note that Lemmas \ref{lem:ising_model_conc} and \ref{lem:quadr_conc} use Theorem 6 and 5 of \cite{dagan2021learning}, respectively, whose concentration arguments rely on the moment inequality stated in Lemma 13 of \cite{dagan2021learning} and Corollary 3.8 of \cite{Adamczak2019Note}. The constant in this moment inequality can be chosen to satisfy $C \lesssim \rho^{-2}e^{cB}$, where $\lesssim$ ignores constants independent of $\rho$ and $B$. Indeed, Theorem 3.2 of \cite{Adamczak2019Note} defines an approximate tensorization constant $C' = 2/[\beta_{\min}(1-\|A\|_2)^2]$, where $\beta_{\min}$ is the smallest single-site conditional probability and $A$ is an influence matrix. Dependence on $\beta_{\min}$ reduces to exponential dependence on $B$. By Corollary 3.3 of \cite{Adamczak2019Note}, we may take $A_{ij} = |J_{ij}|$ where $J$ is the interaction matrix of the Ising model which gives $\|A\|_2 \leq \|J\|_\infty \leq 1 - \rho$. The $C$ used in the moment inequality of Corollary 3.8 differs from $C'$ only in its dependence on $\beta_{\min}$, which remains exponential in $B$. Tracking the use of $C$ in the proof of Theorem 6 in \cite{dagan2021learning} through two applications of Lemma 13 gives the desired dependencies on $\rho$ and $B$ in Lemmas \ref{lem:ising_model_conc} and \ref{lem:quadr_conc}.

We introduce common notation for Lemmas \ref{lem:anti-concentration}-\ref{lem:affine_conc}. Let $\sigma$ be an Ising model over $\{-1,1\}^m$ with interaction matrix $J$ and external field $w$ satisfying Dobrushin's uniqueness condition. Let $\|J\|_\infty \leq 1 - \rho$ and $\|w\|_\infty \leq C_0B$ for a universal constant $C_0$.

\begin{lemma}[Lemma 6 of \cite{dagan2021learning}]
    For any vector $a \in \mathbb{R}^m$,
    \begin{equation*}
        \var(a^\top \sigma) \geq e^{-cB}\|a\|_2^2.
    \end{equation*}
    \label{lem:anti-concentration}
\end{lemma}

\begin{lemma}[Lemma 24 of \cite{dagan2021learning}]
    Let $M$ be a symmetric real matrix of dimension $m \times m$ with zeroes on the diagonal, let $b \in \mathbb{R}^m$ be a vector and let
    \begin{equation*}
        f(\sigma) = \sum_{i\in[m]} (M_i\sigma + b_i)(\sigma_i - \tanh(J_i\sigma + w_i)).
    \end{equation*}
    Then, for any $u > 0$,
    \begin{equation*}
        \Pr[|f(\sigma)| \geq u] \leq C\exp\left(-e^{-cB} \min\left(\frac{\rho^{4}u^2}{\|\mathbb{E}[M\sigma + b]\|_2^2},\frac{\rho^4u^2}{\|M\|_F^2},\frac{\rho^2 u}{\|M\|_2}\right)\right).
    \end{equation*}
    \label{lem:ising_model_conc}
\end{lemma}

\begin{lemma}[Lemma 8 of \cite{dagan2021learning}]
     Let $M$ be a symmetric real matrix of $m \times m$ with zeroes on the diagonal and let $b \in \mathbb{R}^m$ be a vector. Then, for any $u > 0$,
    \begin{multline*}
        \Pr\left[\left|\|M\sigma + b\|_2^2 - \mathbb{E}\left[\|M\sigma + b\|_2^2]\right]\right| \geq u \right]
        \\ \leq 
        C\exp\left(
            -\frac{e^{-cB}}{\|M\|_2^2}\min\left(\frac{\rho^{4} u^2}{\|M\|_F^2 + \|\mathbb{E}[M\sigma + b]\|_2^2}, \rho^2 u\right)
        \right).
    \end{multline*}
    \label{lem:quadr_conc}
\end{lemma}

\begin{lemma}[Theorem 4.3 of \cite{chatterjee2005concentration}]
    Let $a \in \mathbb{R}^m$ be a vector and $b \in \mathbb{R}$ be a scalar. Define,
    \begin{equation*}
    f(\sigma) := \sum_{i\in [m]} a_i \sigma_i + b,
    \end{equation*}
    which is an affine function of $\sigma$. Then, for any $u > 0$,
    \begin{equation*}
    \mathbb{P}\left(
        \left| f(\sigma) - \mathbb{E}[f(\sigma)] \right|
        > u
    \right)
    \le 2\exp\left(-c \cdot \frac{\rho u^2}{ \|a\|^2_2}\right).
    \end{equation*}
    \label{lem:affine_conc}
\end{lemma}

As outlined, these results will be used as subroutines to prove a lower bound for $\mathbb{E}[Y_t|\mathcal{F}_{t-1}]$ and concentration for $Y_t$ conditionally on $\mathcal{F}_{t-1}$. A bound on $|\sum_{t=1}^T D_t|$ requires first converting conditional concentration into a conditional MGF bound.

\begin{lemma}
    Let $D$ be a real-valued random variable with $\mathbb E[D\mid\mathcal F]=0$. Suppose that $A$ and $b$ are nonnegative $\mathcal F$-measurable quantities satisfying $A\ge b^2$, and suppose that, for all $u\ge0$,
    \begin{equation*}
        \mathbb P\left(|D|\ge u\mid\mathcal F\right)
        \le
        C\exp\left(
        -c\min\left\{\frac{u^2}{A},\frac{u}{b}\right\}
        \right).
    \end{equation*}
    Then there are potentially different constants $c,C>0$ such that for all $|\lambda|\le c/b$,
    \begin{equation}
        \mathbb{E}\left[e^{\lambda D}\mid\mathcal{F}\right]
        \le
        \exp\left(C\lambda^2 A\right).
        \label{eq:cond_mgf_bound}
    \end{equation}
    Alternatively, if for all $u \geq 0$ we have
    \begin{equation*}
        \mathbb P\left(|D|\ge u\mid\mathcal F\right) \leq C\exp\left(-c \cdot \frac{u^2}{A}\right),
    \end{equation*}
    then \eqref{eq:cond_mgf_bound} holds for all $\lambda \in \mathbb{R}$.
    \label{lem:tail_to_mgf}
\end{lemma}

Conditional MGF bounds for each $D_t$ imply concentration of the sum $\sum_{T=1}^T D_t$ as per Lemma \ref{lem:martingale_mgf}. The concentration holds for many choices of $\lambda$, which can be chosen to obtain the tightest bound.
\begin{lemma}
    Let $(D_t,\mathcal{F}_t)_{t=1}^T$ where $\mathcal{F}_{t} \subseteq \mathcal{F}_{t+1}$ are nested be such that for all $t=1,\dots,T$,
    \begin{equation*}
        \mathbb{E}[e^{\lambda D_t} \,|\, \mathcal{F}_{t-1}] \leq \exp(C\lambda^2 A_t),
    \end{equation*} 
    for all $\lambda \in \Lambda$, where $C > 0$ is constant, $A_t$ is $\mathcal{F}_{t-1}$-measureable, and $\Lambda$ is a deterministic interval symmetric about zero. Then, for all $u, v, \lambda > 0$ where $\lambda \in \Lambda$, we have
    \begin{equation}
        \mathbb{P}\left(\left|\sum_{t=1}^T D_t\right| \geq u, \sum_{t=1}^T A_t \leq v^2\right) \leq 2\exp(-\lambda u + C\lambda^2 v^2).
        \label{eq:martingale_mgf_bound}
    \end{equation}
    \label{lem:martingale_mgf}
\end{lemma}

Finally, if the $A_t$'s are random and a deterministic upper bound is conservative, an additional `peeling' argument is necessary. 

\begin{lemma}
    Suppose $A_t$ are random such that $1 \leq \sqrt{\sum_{t=1}^T A_t} \leq J$ for some $J \geq 1$. Suppose also that for all $u,v > 0$ we have
    \begin{equation}
        \mathbb{P}\left(\left|\sum_{t=1}^T D_t\right| \geq f(u,v), \sum_{t=1}^T A_t \leq v^2\right) \leq \exp(-c u^2),
        \label{eq:martingale_mgf_peeling_cond}
    \end{equation}
    where $f(u,v) = \max\{uv,bu^2\}$ for some $b \geq 0$. Then,
    \begin{equation*}
        \mathbb{P}\left(\left|\sum_{t=1}^T D_t \right| \geq f\left(u,\sqrt{\sum_{t=1}^T A_t}\right)\right) \leq (1 + \log J)\exp(-c u^2).
    \end{equation*}
    for a potentially different constant $c > 0$.
    \label{lem:martingale_mgf_peeling}
\end{lemma}
Note that \eqref{eq:martingale_mgf_bound} can be converted into the form of \eqref{eq:martingale_mgf_peeling_cond} by selecting $\lambda$ and choosing $u$ appropriately. The proofs of Lemmas \ref{lem:tail_to_mgf}, \ref{lem:martingale_mgf}, and \ref{lem:martingale_mgf_peeling} are deferred to Section \ref{subsec:proof_tail_to_mgf}. 

\subsection{Proof of Auxiliary Lemmas}
We now complete the proofs of Lemmas \ref{lem:cond_means_lb}, \ref{lem:martingale_conc}, and Lemma \ref{lem:ext_field_identifiability}. Lemma \ref{lem:cond_means_lb} is the anti-concentration result lower bounding $\sum_{t=1}^T \mathbb{E}[Y_t|\mathcal{F}_{t-1}]$; Lemma \ref{lem:martingale_conc} comprises the tail bound, MGF bound, and martingale arguments. Lemma \ref{lem:ext_field_identifiability} comprises all the analogous arguments, since we do not split into sub-results.

Recall from \eqref{eq:score} and \eqref{eq:hessian} that
\begin{equation*}
    S_t = \sum_{i=1}^N \big[\tanh(\xi^\ast \Gamma_i \x^{(t)} + {h^\ast}_i^{(t)}) - x_i^{(t)}\big]
    (\tilde{h}_i^{(t)} + \tilde{\xi} \Gamma_i \x^{(t)})
    \quad\text{and}\quad
    H_t = e^{-cB}\left\|\tilde{\xi} \Gamma \x^{(t)} + \tilde{\h}^{(t)}\right\|_2^2.
\end{equation*}
Recall the definitions $\mathcal{G}_t = \sigma(\x^{(1)},\dots,\x^{(t)},\z)$,
\begin{equation*}
    Q^2_t = \left\|\mathbb{E}\left[\tilde{\xi}\Gamma \x^{(t)} + \tilde{\h}^{(t)}\Big|\x^{(t-1)},\z^{(t)}\right]\right\|_2^2 + \|\tilde{\xi}\Gamma\|_F^2
\end{equation*}
from \eqref{eq:q_t_def}, and the characterization $\sum_{t=1}^T Q_t^2 = r(\theta)$ where $r(\theta)$ is defined in \eqref{eq:r_dist_def}.
        
\setcounter{lemma}{3}
\begin{lemma} Let $\theta \in \Theta$ be fixed. Then, there exists a constant $c_4 > 0$ such that
    \begin{equation*}
        \sum_{t=1}^T \mathbb{E}[S_t + H_t | \mathcal{G}_{t-1}] \geq e^{-c_4B} \cdot r(\theta).
    \end{equation*}
\end{lemma}

\begin{proof}[Proof of Lemma \ref{lem:cond_means_lb}]
    
        First, note that 
        \begin{multline*}
            \mathbb{E}\Big[\big[\tanh(\xi^\ast \Gamma_i \x^{(t)} + {h^\ast}_i^{(t)}) - x_i^{(t)}\big]
            (\tilde{h}_i^{(t)} + \tilde{\xi} \Gamma_i \x^{(t)})|\mathcal{G}_{t-1},\x_{-i}^{(t)}\Big]
            \\ = \big[\tanh(\xi^\ast \Gamma_i \x^{(t)} + {h^\ast}_i^{(t)}) - \mathbb{E}[x_i^{(t)}|\mathcal{G}_{t-1},\x_{-i}^{(t)}]\big](\tilde{h}_i^{(t)} + \tilde{\xi} \Gamma_i \x^{(t)})
            = 0,
        \end{multline*}
        where the first equality follows since $\Gamma$ has zero on the diagonal so $\Gamma_i\x^{(t)}$ is $\x^{(t)}_{-i}$ measureable and the second since $\mathbb{E}[x_i^{(t)}|\mathcal{G}_{t-1},\x_{-i}^{(t)}] = \tanh(\xi^\ast \Gamma_i \x^{(t)} + {h^\ast}_i^{(t)})$. Thus, by the Tower law,
        \begin{equation}
            \mathbb{E}[S_t | \mathcal{G}_{t-1}] = 0.
            \label{eq:score_mean}
        \end{equation}
    Consider now,
    \begin{align*}
        \mathbb{E}\left[\|\tilde{\xi}\Gamma\x^{(t)} + \tilde{\h}^{(t)}\|_2^2 \,|\, \mathcal{G}_{t-1}\right] 
        &= \sum_{i=1}^N \mathbb{E}[(\tilde{\xi}\Gamma_i\x^{(t)} + \tilde{h}_i^{(t)})^2|\mathcal{G}_{t-1}] \\
        &= \sum_{i=1}^N \left(\big(\mathbb{E}[\tilde{\xi}\Gamma_i\x^{(t)} + \tilde{h}_i^{(t)} \,|\, \mathcal{G}_{t-1}]\big)^2 +  \var\left(\tilde{\xi}\Gamma_i\x^{(t)} | \mathcal{G}_{t-1}\right)\right) \\
        &= \left\|\mathbb{E}[\tilde{\xi}\Gamma\x^{(t)} + \tilde{\h}^{(t)} \,|\, \mathcal{G}_{t-1}]\right\|_2^2 + \sum_{i=1}^N \var\left(\tilde{\xi}\Gamma_i\x^{(t)} | \mathcal{G}_{t-1}\right).
    \end{align*}
    Note that $\x^{(t)}\,|\,\mathcal{G}_{t-1}$ is an Ising model under Dobrushin's. Therefore, by Lemma \ref{lem:anti-concentration}
    \begin{equation*}
        \mathbb{E}[H_t | \mathcal{G}_{t-1}]  \geq e^{-c_4B}\left(\left\|\mathbb{E}[\tilde{\xi}\Gamma\x^{(t)} + \tilde{\h}^{(t)} \,|\, \mathcal{G}_{t-1}]\right\|_2^2 + \|\tilde{\xi}\Gamma\|_F^2\right).
    \end{equation*}
    for some constant $c_4 > 0$. Combining with \eqref{eq:score_mean} and summing over time gives the desired result by definition of $r(\theta)$ in \eqref{eq:r_dist_def}.   
\end{proof}

\begin{lemma}
    Let $\theta \in \Theta$ be fixed. There exist constants $c_5,c_5',C_5 > 0$ such that for all $u_5 \geq 0$,
    \begin{multline*}
        \mathbb{P}\left(
        \left|\sum_{t=1}^T \Big(S_t + H_t - \mathbb{E}[S_t + H_t | \mathcal{G}_{t-1}] \Big) \right|
        \le
        \rho^{-2}e^{c_5B}\max\left\{
        u_5\sqrt{r(\theta)},
        u_5^2
        \right\}
        \right) \\
        \geq 1-C_5\log NTe^{-c_5'u_5^2}.
    \end{multline*}
\end{lemma}
\setcounter{lemma}{17}

\begin{proof}[Proof of Lemma \ref{lem:martingale_conc}]
    Given $\mathcal{G}_{t-1}$, $\x^{(t)}$ is an Ising model under Dobrushin's uniqueness condition. Let $t$ be fixed.
    We apply Lemma \ref{lem:ising_model_conc} to $S_t$ by setting $M = \tilde{\xi}\Gamma$, $J = \xi^\ast \Gamma$, $b = \tilde{\h}^{(t)}$, and $w = {\h^\ast}^{(t)}$. Then, recalling the definition of $Q_t$, for some constants $c,C > 0$ 
    \begin{equation}
        \mathbb{P}\left(|S_t| \geq u \,|\, \mathcal{G}_{t-1}\right) \leq 
        C\exp\left(-e^{-cB} \min\left(\frac{\rho^{4}u^2}{Q_t^2},\frac{\rho^{2}u}{\|\tilde{\xi}\Gamma\|_2}\right)\right).
        \label{eq:s_t_conc}
    \end{equation}
    Similarly applying Lemma \ref{lem:quadr_conc} to $H_t$, we get for some constants $c,c' > 0$,
    \begin{equation}
        \mathbb{P}\left(|H_t - \mathbb{E}[H_t|\mathcal{G}_{t-1}]| \geq e^{-c'B} \|\tilde{\xi}\Gamma\|_2 \cdot u \,|\, \mathcal{G}_{t-1}\right) \leq
        C\exp\left(-e^{-cB}\min\left(\frac{\rho^4 u^2}{Q_t^2},\frac{\rho^2 u}{\|\tilde{\xi}\Gamma\|_2}\right)\right).
        \label{eq:h_t_conc}
    \end{equation}
    We write $D_t = S_t + H_t - \mathbb{E}[S_t + H_t | \mathcal{G}_{t-1}]$ and recall that $\mathbb{E}[S_t|\mathcal{G}_{t-1}] = 0$ from \eqref{eq:score_mean}. Letting $b = \|\tilde{\xi}\Gamma\|_2$, we combine \eqref{eq:s_t_conc} and \eqref{eq:h_t_conc} to derive
    \begin{equation*}
        \mathbb{P}\left(\left|\frac{D_t}{\rho^{-2}e^{c_0B}}\right| > u | \mathcal{G}_{t-1} \right) \leq C\exp\left(- c\min\left(\frac{u^2}{{Q_t'}^2},\frac{u}{b}\right)\right),
    \end{equation*}
    for a constant $c_0$.
    Since $Q_t \geq \|\tilde{\xi}\Gamma\|_F \geq \|\tilde{\xi}\Gamma\|_2$, we satisfy the condition ${Q}_t^2 \geq b^2$. Lemma \ref{lem:tail_to_mgf} therefore yields
    \begin{equation*}
        \mathbb{E}[e^{\lambda \bar{D}_t} | \mathcal{G}_{t-1}] \leq \exp(C\lambda^2 Q_t^2), \qquad |\lambda| \leq \frac{c}{b}
    \end{equation*}
    for all $t \in [T]$, where we have denoted $\bar{D}_t = D_t/\rho^{-2}e^{c_0B}$. Then, Lemma \ref{lem:martingale_mgf} gives 
    \begin{equation*}
        \mathbb{P}\left(\left|\sum_{t=1}^T \bar{D}_t\right| \geq u, \sum_{t=1}^T {Q}_t^2 \leq v^2\right) \leq 2\exp(-\lambda u + C\lambda^2 v^2),
    \end{equation*}
    for all $ 0 < \lambda \leq \frac{c}{b}$. 
    
    Let $\lambda = \min\{\frac{u}{2Cv^2},\frac{c}{2b}\}$. Then, $C\lambda^2v^2 \leq \lambda u/2$ so
    \begin{equation*}
        \mathbb{P}\left(\left|\sum_{t=1}^T \bar{D}_t\right| \geq u, \sum_{t=1}^T {Q}_t^2 \leq v^2\right) \leq 2\exp\left(- \frac{\lambda u}{2}\right) \leq 2\exp\left(- \min \left(\frac{u^2}{4Cv^2}, \frac{cu}{4b}\right)\right).
    \end{equation*}
    Equivalently, 
    \begin{equation}
        \mathbb{P}\left(\left|\sum_{t=1}^T \bar{D}_t\right| \geq \max\{uv, b u^2\}, \sum_{t=1}^T {Q}_t^2 \leq v^2\right)
        \leq 2\exp(-c u^2).
        \label{eq:pre_peeling_bound}
    \end{equation}
    for a potentially different constant $c$.

    To apply Lemma \ref{lem:martingale_mgf_peeling}, define $A_t = {Q}_t^2 + 1/T$. Since $A_t \leq N + 1/T$ for all $t$, applying Lemma \ref{lem:martingale_mgf_peeling} with $J = C'NT$ yields
    \begin{equation*}
        \mathbb{P}\left(\left|\sum_{t=1}^T \bar{D}_t\right| \geq \max\left\{u\sqrt{\sum_{t=1}^T A_t}, b u^2\right\}\right) \leq 2(1+\log(NT))\exp(-cu^2).
    \end{equation*}
    By definition of $\bar{D}_t$, this is equivalent to
    \begin{equation}
        \mathbb{P}\left(
        \left|\sum_{t=1}^T D_t\right|
        \le
        \rho^{-2}e^{c_5B}\max\left\{
        u_5\sqrt{\sum_{t=1}^T {Q_t}^2},
        u_5^2
        \right\}
        \right) \\
        \geq 1-C_5\log NTe^{-c_5'u_5^2}.
        \label{eq:final_eq_of_lemma5}
    \end{equation}
    up to a renaming of constants if $u_5 \geq 1$ (we use $u_5 \geq 1$ to convert from $A_t$ back to ${Q}_t$). If $u_5 < 1$, then \eqref{eq:final_eq_of_lemma5} is trivial by making $C_5$ large enough. The final bound follows by recalling definition of $D_t$ and the characterization $r(\theta) = \sum_{t=1}^T Q_t^2$.
\end{proof}

We now turn towards Lemma \ref{lem:ext_field_identifiability}.
\setcounter{lemma}{6}
\begin{lemma}
    Let $\theta\in\Theta$ be fixed and recall $A$ is the $N \times T$ dimensional matrix that comprises $\alpha_i^{(t)}$.
    Then there exist constants $c_7,C_7>0$ such that, with probability at least
    \begin{equation*}
        1-C_7\exp\big(-\rho e^{-c_7B} d_h\big),
    \end{equation*}
    we have
    \begin{equation*}
        \sum_{t=1}^T \|\tilde{\h}^{(t)}\|_2^2
        \geq e^{-c_7B}\Big(\|\tilde{A}\|_F^2 + TN(\tilde{\beta}^2 + \tilde{\eta}^2) \Big)
        = e^{-c_7B}d_h \cdot \frac{N}{\|\Gamma\|_F^2}.
    \end{equation*}
\end{lemma}
\setcounter{lemma}{12}

\begin{proof}
We write
\begin{equation*}
    \sum_{t=1}^T \|\tilde{\h}^{(t)}\|_2^2
    = \sum_{t=1}^T L_t,
    \qquad
    L_{t}:=\sum_{i \in N}  \bigl(\tilde \alpha_i^{(t)}+\tilde\beta z_i^{(t)}+\tilde\eta x_i^{(t-1)}\bigr)^2.
    \label{eq:l_t_def}
\end{equation*}
and let $\mathcal{G}'_t = \mathcal{G}_{t-1} = \sigma(\x^{(1)},\dots,\x^{(t-1)},\z)$. Furthermore, write $D'_t = L_t - \mathbb{E}[L_t|\mathcal{G}'_{t-1}]$. Note that since $\x^{(0)}$ is conditioned upon, $L_1$ is known and $D'_1 = 0$.

Fix $t \geq 2$. Since the conditional distribution of $\x^{(t-1)}$ given $\mathcal{G}'_{t-1}$ is an Ising model with local fields bounded by $O(B)$, there exists $c>0$ such that, for both signs $r\in\{-1,1\}$ and every $i\in N$,
\begin{equation*}
    \Pr\bigl(x_i^{(t-1)}=r\mid \mathcal{G}'_{t-1}\bigr) \ge e^{-cB}.
\end{equation*}
It follows that
\begin{align*}
    \mathbb E\left[\bigl(\tilde \alpha_i^{(t)}+\tilde\beta z_i^{(t)}+\tilde\eta x_i^{(t-1)}\bigr)^2\ \Big|\ \mathcal G'_{t-1}\right]
    &\ge e^{-cB}\sum_{r\in\{-1,1\}}\bigl(\tilde \alpha_i^{(t)}+\tilde\beta z_i^{(t)}+\tilde\eta r\bigr)^2 \\
    &= 2e^{-cB}\bigl((\tilde \alpha_i^{(t)}+\tilde\beta z_i^{(t)})^2+\tilde\eta^2\bigr).
\end{align*}
We also have that
\begin{equation*}
    L_1 \geq \frac{1}{2}\sum_{i=1}^N (\tilde \alpha_i^{(1)}+\tilde\beta z_i^{(1)})^2-N\tilde\eta^2
    \geq e^{-cB}\left(\frac{1}{2}\sum_{i=1}^N (\tilde \alpha_i^{(1)}+\tilde\beta z_i^{(1)})^2-N\tilde\eta^2\right),
\end{equation*}
assuming $c$ large enough such that $2e^{-cB} < 1$.
Therefore, by linearity of expectation and the tower law,
\begin{equation}
    \sum_{t=1}^T
    \mathbb E[L_t\mid \mathcal G'_{t-1}]
    \ge
    Ce^{-cB}\Bigl(\|\tilde A+\tilde\beta Z\|_F^2+N(2T-3)\tilde\eta^2\Bigr)
    \ge
    Ce^{-cB}\Bigl(\|\tilde A\|_F^2+NT(\tilde\beta^2+\tilde\eta^2)\Bigr),
    \label{eq:ext_field_predictable_lb}
\end{equation}
where the second inequality follows from Assumption \ref{ass:exc_interv} and possibly increasing $c$.

Since $x_i^{(t-1)} \in \{-1,1\}$, each single-time block contribution is an affine function of the block spin vector $\x^{(t-1)}$:
\begin{equation*}
    L_t = b_{t} + a_{t}^\top \x^{(t-1)},
\end{equation*}
where
\begin{equation*}
    a_{t,i}:=2\tilde \eta(\tilde{\alpha}_i^{(t)} + \tilde{\beta} z_i^{(t)}),
    \qquad
    b_{t}:=\sum_{i\in N}\Bigl((\tilde \alpha_i^{(t)}+\tilde\beta z_i^{(t)})^2+\tilde\eta^2\Bigr).
\end{equation*}
Applying Lemma \ref{lem:affine_conc} with $f(\x^{(t-1)})= L_t$ and conditioning appropriately, we obtain for every $t \geq 2$, and $u>0$,
\begin{equation*}
    \Pr\left(
    |D'_t|
    \ge u
    \,\Big|\, \mathcal{G}'_{t-1}
    \right)
    \le 2\exp\left(-c'\frac{\rho u^2}{\|a_t\|_2^2}\right).
\end{equation*}
For $t=1$, the same concentration bound holds trivially since $D_1' = 0$. By Lemma \ref{lem:tail_to_mgf}, we therefore have
\begin{equation*}
    \mathbb{E}[e^{\lambda D'_t} \,|\, \mathcal{G}'_{t-1} ] \leq \exp({C'}^2\lambda^2\rho^{-1}\|a_t\|_2^2).
\end{equation*}
for all $\lambda \in \mathbb{R}$. Thus, Lemma \ref{lem:martingale_mgf} gives
\begin{equation*}
    \mathbb{P}\left(\left|\sum_{t=1}^T D'_t\right| \geq u, \rho^{-1}\sum_{t=1}^T \|a_t\|_2^2 \leq v^2\right) \leq 2\exp(-\lambda u + C\lambda^2 v^2),
\end{equation*}
for all $\lambda > 0$. Let $\lambda = \frac{u}{2C'v^2}$. Then $C'\lambda^2v^2 \leq \lambda u/2$ so
\begin{equation*}
    \mathbb{P}\left(\left|\sum_{t=1}^T D'_t\right| \geq u, \rho^{-1}\sum_{t=1}^T \|a_t\|_2^2 \leq v^2\right) \leq \exp\left(- \frac{\lambda u}{2}\right) \leq \exp\left(-c' \cdot \frac{u^2}{v^2}\right),
\end{equation*}
for a potentially different constant $c'$.
Equivalently, 
\begin{equation}
    \mathbb{P}\left(\left|\sum_{t=1}^T D'_t\right| \geq uv, \rho^{-1}\sum_{t=1}^T \|a_t\|_2^2 \leq v^2\right)
    \leq \exp(-c' u^2).
    \label{eq:affine_martingale_conc}
\end{equation}
Using $(r+s)^2\le 2r^2+2s^2$ and $z_i^{(t)} \in \{-1,1\}$, we now note that
\begin{align*}
    \sum_{t=1}^T\|a_{t}\|_2^2
    &= 4\tilde{\eta}^2\sum_{t=1}^T\sum_{i=1}^N\big(\tilde{\alpha}_i^{(t)}+\tilde{\beta} z_i^{(t)}\big)^2 \\
    &\leq 8\tilde{\eta}^2\|\tilde{A}\|_F^2 + 8NT\tilde{\eta}^2\tilde{\beta}^2 \\
    &\leq 32B^2(\|\tilde{A}\|_F^2 + NT\tilde{\beta^2})\\
    &\leq e^{c''B}\big(\|\tilde{A}\|_F^2 + NT(\tilde{\beta}^2 + \tilde{\eta}^2)\big),
\end{align*}
since $\tilde{\eta}^2,\tilde{\beta}^2 \leq 4B^2$.
Therefore, \eqref{eq:affine_martingale_conc} implies
\begin{equation*}
    \Pr\left(
    \left|\sum_{t=1}^T D'_t\right|
    \ge u\sqrt{\rho^{-1}e^{c''B}\big(\|\tilde{A}\|_F^2 + NT(\tilde{\beta}^2 + \tilde{\eta}^2)\big)}
    \right)
    \leq \exp\left(-c'u^2\right).
\end{equation*}
On this event, using \eqref{eq:ext_field_predictable_lb},
\begin{equation*}
    \sum_{t=1}^T\|\tilde\h^{(t)}\|_2^2
    \ge
    Ce^{-cB}\big(\|\tilde{A}\|_F^2 + NT(\tilde{\beta}^2 + \tilde{\eta}^2)\big)
    -
    u\sqrt{\rho^{-1}e^{c''B}\big(\|\tilde{A}\|_F^2 + NT(\tilde{\beta}^2 + \tilde{\eta}^2)\big)}.
\end{equation*}
By appropriate choice of $u = e^{-c_7'B}\sqrt{\rho e^{-c''B}\big(\|\tilde{A}\|_F^2 + NT(\tilde{\beta}^2 + \tilde{\eta}^2)\big)}$, this gives
\begin{equation*}
    \sum_{t=1}^T \|\tilde{\h}^{(t)}\|_2^2
    \ge
    e^{-c_7B}\big(\|\tilde{A}\|_F^2 + NT(\tilde{\beta}^2 + \tilde{\eta}^2)\big)
    =
    e^{-c_7B} d_h  \cdot \frac{N}{\|\Gamma\|_F^2},
\end{equation*}
for an appropriate constant $c_7>0$. The probability of the event is
\begin{equation*}
    1 - \exp\left(-\rho e^{-c_7B} d_h  \cdot \frac{N}{\|\Gamma\|_F^2}\right) \geq 1 - C_7\exp\left(-\rho e^{-c_7B} d_h\right),
\end{equation*} 
where the inequality follows since $\|\Gamma\|_F^2 \leq N$ because $\|\Gamma\|_\infty = 1$.
\end{proof}

\subsection{Proof of Lemmas \ref{lem:tail_to_mgf}, \ref{lem:martingale_mgf}, and \ref{lem:martingale_mgf_peeling}} \label{subsec:proof_tail_to_mgf}

We prove Lemmas \ref{lem:tail_to_mgf}, \ref{lem:martingale_mgf}, and \ref{lem:martingale_mgf_peeling}

\setcounter{lemma}{12}
\begin{lemma}
    Let $D$ be a real-valued random variable with $\mathbb E[D\mid\mathcal F]=0$. Suppose that $A$ and $b$ are nonnegative $\mathcal F$-measurable quantities satisfying $A\ge b^2$, and suppose that, for all $u\ge0$,
    \begin{equation}
        \mathbb P\left(|D|\ge u\mid\mathcal F\right)
        \le
        C\exp\left(
        -c\min\left\{\frac{u^2}{A},\frac{u}{b}\right\}
        \right).
        \label{eq:generic_cond_tail_bernstein}
    \end{equation}
    Then there are potentially different constants $c,C>0$ such that for all $|\lambda|\le c/b$,
    \begin{equation}
        \mathbb{E}\left[e^{\lambda D}\mid\mathcal{F}\right]
        \le
        \exp\left(C\lambda^2 A\right).
        \label{eq:cond_mgf_bound_2}
    \end{equation}
    Alternatively, if for all $u \geq 0$ we have
    \begin{equation}
        \mathbb P\left(|D|\ge u\mid\mathcal F\right) \leq C\exp\left(-c \cdot \frac{u^2}{A}\right),
        \label{eq:generic_cond_tail_bound_gauss}
    \end{equation}
    then \eqref{eq:cond_mgf_bound_2} holds for all $\lambda$.
\end{lemma}

\begin{proof}[Proof of Lemma \ref{lem:tail_to_mgf}]
    It suffices to prove the result for $\lambda \geq 0$; the case $\lambda<0$ follows by applying the same argument to $-D$, which satisfies the same tail bound.

    Because $\mathbb E[D\mid\mathcal F]=0$,
    \begin{equation*}
        \mathbb{E}[e^{\lambda D}\mid\mathcal{F}]
        = 1+\mathbb{E}[e^{\lambda D}-1-\lambda D\mid\mathcal F].
    \end{equation*}
    For $y\ge0$, by the fundamental theorem of calculus,
    \begin{equation*}
        e^{\lambda y}-1-\lambda y
        = \int_0^y \lambda(e^{\lambda s}-1)ds
        \leq \int_0^y \lambda^2 s e^{\lambda s}ds
        \leq
        \lambda^2\int_0^{\infty}s e^{\lambda s}{\bf 1}_{|y| \geq s}ds.
    \end{equation*}
    Using similar logic for $y < 0$ and via Fubini's theorem
    \begin{equation}
        \mathbb{E}[e^{\lambda D}\mid\mathcal F]
        \leq
        1+\lambda^2\int_0^\infty
        s e^{\lambda s}
        \mathbb{P}(|D|\ge s\mid\mathcal F)ds.
        \label{eq:mgf_tail_integral_step}
    \end{equation}
    Using the tail bound of \eqref{eq:generic_cond_tail_bernstein}, we rewrite \eqref{eq:mgf_tail_integral_step} as
    \begin{equation}
        \mathbb E[e^{\lambda D}\mid\mathcal F]
        \le 1+C\lambda^2 I(\lambda),
        \qquad
        I(\lambda)
        :=\int_0^\infty
        s\exp\left(\lambda s-c\min\left\{\frac{s^2}{A},\frac{s}{b}\right\}\right)ds.
        \label{eq:mgf_tail_integral_step_I}
    \end{equation}
    Since $s^2/A = s/b$ at $s = A/b$,
    \begin{equation*}
        I(\lambda) \leq I_1(\lambda) + I_2(\lambda),
    \end{equation*}
    where
    \begin{equation*}
        I_1(\lambda)=\int_0^{A/b}s\exp\left(\lambda s-c\frac{s^2}{A}\right)ds 
        \quad\text{and}\quad
        I_2(\lambda)=\int_{A/b}^{\infty}s\exp\left(\lambda s-c\frac{s}{b}\right)ds.
    \end{equation*}
    For $I_1$, completing the square gives
    \begin{equation*}
        \lambda s-c\frac{s^2}{A}
        =
        \frac{\lambda^2 A}{4c}
        -
        \frac{c}{A}
        \left(s-\frac{\lambda A}{2c}\right)^2
        \leq C'\lambda^2A - c'\frac{s^2}{A}.
    \end{equation*}
    Hence
    \begin{equation*}
        I_1(\lambda)
        \leq
        e^{C'\lambda^2A}\int_0^\infty s e^{-c's^2/A}ds
        \leq
        c'' A e^{C'\lambda^2A}.
    \end{equation*}
    For $I_2$, let $0\le\lambda\le c/(2b)$. Then
    \begin{equation*}
        \lambda s-c\frac{s}{b}
        \leq
        -c'\frac{s}{b},
    \end{equation*}
    and therefore
    \begin{equation*}
        I_2(\lambda)
        \leq
        \int_0^\infty s e^{-c's/b}ds
        \leq c''b^2 \leq c''A,
    \end{equation*}
    since $A \geq b^2$. Combining the two estimates gives
    \begin{equation*}
        \mathbb E[e^{\lambda D}\mid\mathcal F]
        \leq
        1+c\lambda^2A e^{C\lambda^2A},
        \qquad 0 \leq \lambda \leq c/2b
    \end{equation*}
    after a renaming of constants. Writing $r=\lambda^2A\ge0$, the inequality $1+cr e^{Cr}\le e^{C'r}$ yields
    \begin{equation*}
        \mathbb E[e^{\lambda D}\mid\mathcal F]
        \le
        \exp(C\lambda^2A), 
        \qquad 0 \leq \lambda \leq c/b,
    \end{equation*}
    following a final renaming of constants.

    For the second part of the Lemma statement, note that if $D$ instead satisfies the tail bound of \eqref{eq:generic_cond_tail_bound_gauss}, we can replace \eqref{eq:mgf_tail_integral_step_I} with
    \begin{equation*}
        \mathbb E[e^{\lambda D}\mid\mathcal F]
        \le 1+C\lambda^2 I(\lambda),
        \qquad
        I(\lambda) := \int_0^\infty s\exp\left(\lambda s - c\frac{s^2}{A}\right)ds.
    \end{equation*}
    The same argument gives
    \begin{equation*}
        \mathbb E[e^{\lambda D}\mid\mathcal F]
        \le
        \exp(C\lambda^2A) 
    \end{equation*}
    for all $\lambda \geq 0$.
\end{proof}

\begin{lemma}
    Let $(D_t,\mathcal{F}_t)_{t=1}^T$ be such that for all $t=1,\dots,T$
    \begin{equation}
        \mathbb{E}[e^{\lambda D_t} \,|\, \mathcal{F}_{t-1}] \leq \exp(C\lambda^2 A_t)
        \label{eq:lemma_tail_to_mgf_conseq}
    \end{equation} 
    for all $\lambda \in \Lambda$ where $C > 0$ is constant, $A_t$ is $\mathcal{F}_{t-1}$-measureable, and $\Lambda$ is a deterministic interval symmetric about zero. Then, for all $u, v, \lambda > 0$ where $\lambda \in \Lambda$, we have
    \begin{equation*}
        \mathbb{P}\left(\left|\sum_{t=1}^T D_t\right| \geq u, \sum_{t=1}^T A_t \leq v^2\right) \leq 2\exp(-\lambda u + C\lambda^2 v^2).
    \end{equation*}
\end{lemma}

\begin{proof}[Proof of Lemma \ref{lem:martingale_mgf}]
    Let $M_k = \sum_{t=1}^k D_t$ and $V_k = \sum_{t=1}^k A_t$. For all $k \in [T]$:
    \begin{multline}
        \mathbb{E}[\exp(\lambda M_k - C\lambda^2 V_k)] 
        = \mathbb{E}\Big[\exp\left(\lambda M_{k-1} - C\lambda^2 V_{k-1}\right)
        \cdot \exp(-C\lambda^2 A_k) \cdot \mathbb{E}[\exp(\lambda D_k) |\mathcal{F}_{k-1}]\Big]
        \\ \leq \mathbb{E}[\exp(\lambda M_{k-1} - C\lambda^2 V_{k-1})] ,
        \label{eq:tower_law_chernoff}
    \end{multline}
    where the first inequality follows by the Tower law and since $A_k$ is $\mathcal{F}_{k-1}$ measureable and the last inequality follows by \eqref{eq:lemma_tail_to_mgf_conseq}. Thus, $\mathbb{E}[\exp(\lambda M_T - C\lambda^2 V_T)] \leq  1$ by repeated application of \eqref{eq:tower_law_chernoff}. Now, let $\lambda \in \Lambda$ with $\lambda \geq 0$. Then, for any $u$ and $v$,
    \begin{align*}
        \mathbb{P}\left(\sum_{t=1}^T D_t \geq u, \sum_{t=1}^T A_t \leq v^2\right) 
        &= \mathbb{P}\left(M_T \geq u, V_T \leq v^2\right) \\
        &\leq \mathbb{P}\left(\lambda M_T - C\lambda^2 V_T \geq \lambda u - C\lambda^2 v^2\right) \\
        &\leq \exp(-\lambda u + C\lambda^2 v^2)\mathbb{E}[\exp(\lambda M_T - C\lambda^2 V_T)] \\
        &\leq \exp(-\lambda u + C\lambda^2 v^2).
    \end{align*}
    Applying the same argument for $-\sum_{t=1}^T D_t$ gives the two sided result.
\end{proof}

\begin{lemma}
    Suppose $A_t$ are random such that $1 \leq \sqrt{\sum_{t=1}^T A_t} \leq J$ for some $J > 1$. Suppose also that for all $u,v > 0$ we have
    \begin{equation}
        \mathbb{P}\left(\left|\sum_{t=1}^T D_t\right| \geq f(u,v), \sum_{t=1}^T A_t \leq v^2\right) \leq \exp(-c u^2),
        \label{eq:general_pre_peeling}
    \end{equation}
    where $f(u,v) = \max\{uv,bu^2\}$ for some $b > 0$ or $f(u,v)=uv$. Then,
    \begin{equation*}
        \mathbb{P}\left(\left|\sum_{t=1}^T D_t \right| \geq f\left(u,\sqrt{\sum_{t=1}^T A_t}\right)\right) \leq (1 + \log J)\exp(-c u^2).
    \end{equation*}
    for a potentially different constant $c$.
\end{lemma}
\begin{proof}[Proof of Lemma \ref{lem:martingale_mgf_peeling}]
    Define $q_j = 2^{j+1}$ and the event $\mathcal{E}_j = \{q_{j-1} \leq \sqrt{\sum_{t=1}^T A_t} \leq q_j\}$ for $j=0,1,\dots,\log J$. Then, applying \eqref{eq:general_pre_peeling},
    \begin{equation*}
        \mathbb{P}\left(\left|\sum_{t=1}^T D_t\right| \geq f(u/2,q_j), \mathcal{E}_j\right)
        \leq \exp(-cu^2/4),
    \end{equation*}
    since $f$ is monotonic in $v$. Since on $\mathcal{E}_j$, we also have $\sqrt{\sum_{t=1}^T A_t} > q_j/2$, changing constants yields
    \begin{equation*}
        \mathbb{P}\left(\left|\sum_{t=1}^T D_t\right| \geq f\left(u,\sqrt{\sum_{t=1}^T A_t}\right),\mathcal{E}_j\right)
        \leq \exp(-c u^2).
    \end{equation*}
    A union bound over $j=0,1,\dots,\log J$ therefore gives
    \begin{equation*}
        \mathbb{P}\left(\left|\sum_{t=1}^T D_t\right| \geq f\left(u,\sqrt{\sum_{t=1}^T A_t}\right)\right) \leq (1 + \log J)\exp(-cu^2).
    \end{equation*}
\end{proof}
\section{Proof of Theorem \ref{thm:causal.est}} \label{app:thm_proof_2}

\subsection{Preliminaries}
Theorem \ref{thm:causal.est} is based on the following, which leverages Dobrushin's uniqueness condition to argue that, due to correlation decay, error does not propagate spatially and temporally. First, for fixed $\z \in \{-1,1\}^{NT}$ define
\begin{equation*}
    M_{\z}(\theta) = 
    \mathbb{E}_\theta\left[ \frac{1}{NT} \sum_{t=1}^T \sum_{i=1}^N x_i^{(t)} \Big| \mathbb{\z}\right],
\end{equation*} 
as the mean outcome under $\z$ with data generated according to $\theta$.
\begin{theorem}
    Let the interventional pattern $\z$ be fixed. If $\theta$ and $\theta'$ are such that $|\eta| + |\xi| \leq 1 - \kappa$ and $|\eta'| + |\xi'| \leq 1- \kappa$ for some $0 < \kappa < 1$, there exists a constant $c>0$ such that
    \begin{equation*}
        \big(M_{\z}(\theta) - M_{\z}(\theta') \big)^2 \leq c\kappa^{-2}\|\theta-\theta'\|_\star.
    \end{equation*}
    \label{thm:av_caus_eff}
\end{theorem}

\begin{proof}[Proof of Theorem \ref{thm:causal.est}]
    Given Theorem \ref{thm:av_caus_eff}, the proof is a simple application of the triangle inequality.
    \begin{align*}
        (\GTE_\theta(\z_1, \z_0) - \GTE_{\theta'}(\z_1,\z_0))^2 
        &\leq [M_{\z_1}(\theta) - M_{\z_0}(\theta) - (M_{\z_1}(\theta') - M_{\z_0}(\theta'))]^2 \\
        &\leq 2(M_{\z_1}(\theta) - M_{\z_1}(\theta'))^2 + 2(M_{\z_0}(\theta) - M_{\z_0}(\theta'))^2 \\
        &\leq 4c'\kappa^{-2}\|\theta'-\theta\|_\star \\
        &= c\kappa^{-2}\|\theta'-\theta\|_\star,
    \end{align*}
    where $c'$ is the constant from Theorem \ref{thm:av_caus_eff}.
\end{proof}

We therefore dedicate the section to proving Theorem \ref{thm:av_caus_eff}. To do so, we invoke results from the literature on weakly dependent random variables satisfying Dobrushin's uniqueness condition \cite{Follmer_1982}.

\begin{lemma}[F\"ollmer Covariance Estimate]
    Let \(\mu\) be a Gibbs measure on \(\{-1,1\}^N\) with Dobrushin matrix \(C\) satisfying
    \(\sum_{n \ge 0} C^n < \infty\), and let \(D = \sum_{n=0}^{\infty} C^n\). Then for any bounded measurable
    functions \(f,g : \{-1,1\}^N \to \mathbb{R}\),
    \begin{equation*}
        \bigl|\operatorname{Cov}_{\mu}(f,g)\bigr|
        \leq
        \frac14 \sum_{i,k=1}^N \delta_i(f) D_{ik} \delta_k(g),
    \end{equation*}
    where
    \begin{equation*}
        \delta_i(f) := \sup_{x_{-i}} |f(1,x_{-i})-f(-1, x_{-i})|
    \end{equation*}
    is the oscillation of \(f\) at site \(i\).
\label{lem:foll_cov_est}
\end{lemma}


\begin{lemma}[F\"ollmer Comparison Result]
    Let \(\mu\) and \(\nu\) be probability measures on \(\{-1,1\}^N\), assume that \(\mu\) is Gibbs with
    Dobrushin matrix \(C\) and resolvent \(D = \sum_{n=0}^{\infty} C^n\), and define
    \[
    b_k := \int d_{\mathrm{TV}}\bigl(\mu_k(\cdot \mid x_{-k}),\, \nu_k(\cdot \mid x_{-k})\bigr)\, \nu(dx),
    \qquad k \in [N].
    \]
    Then, for every bounded measurable \(f : \{-1,1\}^N \to \mathbb{R}\),
    \[
    \left| \int f\, d\mu - \int f\, d\nu \right|
    \leq
    \sum_{i,k=1}^N \delta_i(f) D_{ik} b_k.
    \]
    \label{lem:foll_comp_thm}
\end{lemma}

\subsection{Proof of Theorem \ref{thm:av_caus_eff}}
At a high-level, the result follows by bounding $\sup_{w\in[0,1]} \frac{\partial M_\z(\theta_w)}{\partial \vartheta}$ for all $\vartheta \in \theta$ (e.g., $\vartheta = \beta$) and where $\theta_w = \theta' + w(\theta - \theta')$. We will show that 
\begin{equation*}
    \frac{\partial M_\z(\theta)}{\partial \vartheta} = \frac{1}{T}\sum_{s=1}^T \sum_{t=1}^s \mathbb{E}\Big[
        \cov\bigg(\frac{1}{N}\sum_{i=1}^N x_i^{(s)}, 
                    M_\vartheta^{(t)} \Big|\x^{(t-1)}\bigg)
    \Big],
\end{equation*}
where $M_\vartheta^{(t)}$ is the multiplier of $\vartheta$ in $p(\x^{(t)}|\x^{(t-1)})$ (e.g., $M_\beta^{(t)} = \sum_{i=1}^N x_i^{(t)}z_i^{(t)}$). It therefore suffices to control the covariances of $x_i^{(s)}, x_j^{(t)}$ $t \leq s$. To do so, we use Lemmas \ref{lem:foll_cov_est} and \ref{lem:foll_comp_thm}.

\begin{proof}[Proof of Theorem \ref{thm:av_caus_eff}]
    Let $\theta_w = \theta' + w(\theta - \theta')$ for $w\in[0,1]$. By the mean value theorem, we can write
    \begin{equation}
        |M_{\z}(\theta) - M_{\z}(\theta')| \leq 
        \sum_{i=1}^N \sum_{t=1}^T |\alpha_i^{(t)} - {\alpha'}_i^{(t)}| \sup_{w\in[0,1]} \left|\frac{\partial M_\z(\theta_w)}{\partial \alpha_i^{(t)}}\right| + 
        \sum_{\vartheta \in (\beta,\xi,\eta)} |\vartheta - \vartheta'|\left|\sup_{w\in[0,1]}\frac{\partial M_\z(\theta_w)}{\partial \vartheta}\right|,
        \label{eq:ate_diff}
    \end{equation}
    Note that for any $w\in[0,1]$ we have $|\eta_w| + |\xi_w| \leq 1 - \kappa$ since $\theta_w$ is a convex combination of $\theta$ and $\theta'$. We now focus on bounding the derivative of the mean outcome for all parameters. We overload notation by writing $\theta$ instead of $\theta_w$. 
    
    \textit{Step 1: Gradient Decomposition.} We begin with the scalar parameters. Let $\vartheta \in (\beta, \xi, \eta)$ be fixed. We suppress conditioning on $\z$ and dependence on $\theta$, which are assumed throughout. For arbitrary parametric models, the following holds for any $h(\x)$ not explicitly dependent on $\vartheta$:
    \begin{equation*}
        \frac{\partial}{\partial \vartheta} \mathbb{E}[h(\x)] = \cov\Big(h(\x), \frac{\partial}{\partial \vartheta} \log p(\x)\Big).
    \end{equation*}
    For our Markovian model, the score term decomposes as
    \begin{equation*}
        \log p(\x) = \sum_{t=1}^T \log p(\x^{(t)}|\x^{(t-1)}).
    \end{equation*}
    Thus, also plugging in the definition $M_\z(\theta)$
    \begin{equation*}
        \frac{\partial M_\z(\theta)}{\partial \vartheta} = 
        \frac{1}{T}\sum_{s=1}^T 
        \cov\Bigg(
            \frac{1}{N}\sum_{i=1}^N x_i^{(s)},
            \sum_{t=1}^T \frac{\partial}{\partial \vartheta} \log p(\x^{(t)}|\x^{(t-1)})
        \Bigg).
    \end{equation*}
    Fixing $s \in [T]$, we aim to show that the covariance term is $O(1)$ for all $\vartheta$. Let $M_\vartheta^{(t)}$ be the multiplier of $\vartheta$ in the conditional distribution $p(\x^{(t)}|\x^{(t-1)})$. For example, $M_\beta^{(t)} = \sum_{i=1}^N x_i^{(t)}z_i^{(t)}$. Then,
    \begin{equation*}
        \frac{\partial}{\partial \vartheta} \log p(\x^{(t)}|\x^{(t-1)}) = M_\vartheta^{(t)} - \mathbb{E}[M_\vartheta^{(t)}|\x^{(t-1)}],
    \end{equation*}
    is a well-known property of exponential family distributions. For a fixed $s$, then,
    \begin{equation}
        \cov\Bigg(
            \frac{1}{N}\sum_{i=1}^N x_i^{(s)},
            \sum_{t=1}^T \frac{\partial}{\partial \vartheta} \log p(\x^{(t)}|\x^{(t-1)})
        \Bigg)
        =
        \sum_{t=1}^T \cov\Bigg(\frac{1}{N}\sum_{i=1}^N x_i^{(s)}, 
        M_\vartheta^{(t)} - \mathbb{E}[M_\vartheta^{(t)}|\x^{(t-1)}]\Bigg).
        \label{eq:cov_fixed_s_1}
    \end{equation}
    For each $t$ in the summand, we condition on $t-1$. Since 
    \begin{equation*}
        \mathbb{E}\Big[M_\vartheta^{(t)} - \mathbb{E}[M_\vartheta^{(t)}|\x^{(t-1)}]|\x^{(t-1)}\Big] = 0,
    \end{equation*}
    Equation \eqref{eq:cov_fixed_s_1} thus reduces to
    \begin{equation*}
        \sum_{t=1}^T \mathbb{E}\Bigg[
            \cov\bigg(\frac{1}{N}\sum_{i=1}^N x_i^{(s)}, 
                        M_\vartheta^{(t)} - \mathbb{E}[M_\vartheta^{(t)}|\x^{(t-1)}]\bigg|\x^{(t-1)}\Bigg)
        \bigg]
        =
        \sum_{t=1}^T \mathbb{E}\Big[
            \cov\bigg(\frac{1}{N}\sum_{i=1}^N x_i^{(s)}, 
                        M_\vartheta^{(t)} \Big|\x^{(t-1)}\bigg)
        \Big].
    \end{equation*}
    Clearly, for $t > s$, this expected covariance is zero since $\x^{(t)}$ is independent of $\x^{(s)}$ given $\x^{(t-1)}$. Summarizing,
    \begin{equation}
        \frac{\partial M_\z(\theta)}{\partial \vartheta} = \frac{1}{T}\sum_{s=1}^T \sum_{t=1}^s \mathbb{E}\Big[
            \cov\bigg(\frac{1}{N}\sum_{i=1}^N x_i^{(s)}, 
                        M_\vartheta^{(t)} \Big|\x^{(t-1)}\bigg)
        \Big].
        \label{eq:gradient_decomp}
    \end{equation}
    For the remaining covariances, both across time (when $t < s$) and within-time (when $t=s$), we will use Lemmas \ref{lem:foll_cov_est} and \ref{lem:foll_comp_thm}.

    \textit{Step 2: Bounding Covariances.} 
    
    \hspace{1.5em} \textit{Step 2a:} Let $s=t$. We note our assumptions that $|\xi| < 1$ and $\|\Gamma\|_\infty = 1$ imply that $p(\x^{(s)}|\x^{(s-1)})$ satisfies Dobrushin's uniqueness condition. Thus, 
    \begin{equation}
        \max_{i} \sum_{k=1}^N D_{ik} \leq \frac{1}{1 - |\xi|}.
        \label{eq:d_bound}
    \end{equation}
    It is also true that
    \begin{equation}
        \max_{k} \sum_{i=1}^N D_{ik} \leq \frac{1}{1 - |\xi|}.
        \label{eq:d_bound_flipped}
    \end{equation}
    To see this, note first that $C_{ik} \leq |\xi\gamma_{ik}|$; since $\Gamma$ is symmetric and its absolute row sums are at most one, both the row and column sums of $C$ are at most $|\xi|$. Thus the row and column sums of $D$ are bounded similarly.
    
    We further note that
    \begin{equation*}
        M_{\beta}^{(t)} = \sum_{i=1}^N x_i^{(t)}z_i^{(t)}, \qquad M_{\xi} = \frac{1}{2} \sum_{i=1}^N (\Gamma \x^{(t)})_i x_i^{(t)}, \qquad M_{\eta}^{(t)} = \sum_{i=1}^N x_i^{(t)}x_i^{(t-1)}, 
    \end{equation*}
    so
    \begin{equation}
        \delta_k(M_{\vartheta}^{(t)}) \leq 2
        \label{eq:delta_bound}
    \end{equation}
    for all $\vartheta \in (\beta, \xi, \eta)$. Applying Lemma \ref{lem:foll_cov_est} with $f = \frac{1}{N} \sum_{i=1}^N \x_i^{(s)}$ (so $\delta_i(f) = 2/N$) and $g = M_\vartheta^{(s)}$ thus gives, for any fixed $\x^{(s-1)}$,
    \begin{equation*}
        \bigg|\cov\bigg(\frac{1}{N}\sum_{i=1}^N x_i^{(s)}, 
                        M_\vartheta^{(s)} \Big|\x^{(s-1)}\bigg)\bigg| 
        \leq \frac{1}{4}\sum_{i=1}^N \sum_{k=1}^N \delta_i(f)D_{ik}\delta_k(g) \leq \frac{1}{1-|\xi|}.
    \end{equation*}
    
    \hspace{1.5em} \textit{Step 2b:} Returning now to \eqref{eq:gradient_decomp}, let $t < s$. Since $M_\vartheta^{(t)}$ is deterministic given $\x^{(t)}$, we can write
    \begin{equation*}
        \cov\bigg(\frac{1}{N}\sum_{i=1}^N x_i^{(s)}, 
                        M_\vartheta^{(t)} \Big|\x^{(t-1)}\bigg)
        = 
        \cov\bigg(\frac{1}{N}\mathbb{E}\Big[\sum_{i=1}^N x_i^{(s)}|\x^{(t)}\Big], 
                        M_\vartheta^{(t)} \Big|\x^{(t-1)}\bigg).
    \end{equation*}
    Both \eqref{eq:d_bound} and \eqref{eq:delta_bound} still hold, so we can re-use Lemma \ref{lem:foll_cov_est} given an appropriate bound on $\delta_i(f_t)$, with $f_t = \mathbb{E}[\sum_{i} x_i^{(s)}/N|\x^{(t)}]$. We do recursively via Lemma \ref{lem:foll_comp_thm}.

    Fix $i \in [N]$ and let $y,y' \in \{-1,1\}^N$ differ only at coordinate $i$. Let
    \begin{equation*}
        \mu = p(\x^{(t+1)}|\x^{(t)}=y, \z) \qquad \nu = p(\x^{(t+1)}|\x^{(t)}=y', \z).
    \end{equation*}
    Then,
    \begin{align*}
        f_t(y) - f_t(y') &= \mathbb{E}\left[\frac{1}{N}\sum_{i} x_i^{(s)}|\x^{(t)} = y\right] - \mathbb{E}\left[\frac{1}{N}\sum_{i} x_i^{(s)}|\x^{(t)} = y'\right] \\
        &= \mathbb{E}\left[\mathbb{E}\left[\frac{1}{N}\sum_{i} x_i^{(s)}|\x^{(t+1)}\right] \Big| \x^{(t)}=y\right] - 
        \mathbb{E}\left[\mathbb{E}\left[\frac{1}{N}\sum_{i} x_i^{(s)}|\x^{(t+1)}\right] \Big| \x^{(t)}=y'\right] \\
        &= \int f_{t+1}d\mu - \int f_{t+1}d\nu.
    \end{align*}
    Thus, by Lemma \ref{lem:foll_comp_thm},
    \begin{equation*}
        | f_t(y) - f_t(y')| \leq \sum_{j=1}^N\sum_{k=1}^N \delta_j(f_{t+1}) D_{jk} b_k,
    \end{equation*}
    where we have already established the base case $\delta_j(f_s) = 2/N$. Since $\mu$ and $\nu$ differ only the $i$th coordinate, we have
    \begin{equation*}
        b_k := \begin{cases} 0 &\text{if } k \neq i \\ \tanh(|\eta|) \leq |\eta| &\text{otherwise}. \end{cases}
    \end{equation*}
    Since \eqref{eq:d_bound_flipped} holds for any $p(\x^{(t)}|\x^{(t-1)})$, we have,
    \begin{equation}
        \delta_i(f_t) \leq \max_{y_{-i},y_i, y'_i} |f_t(y) - f_t(y')| \leq |\eta|\sum_{j=1}^N \delta_j(f_{t+1})D_{ji} \leq \frac{|\eta|}{1 - |\xi|} \max_j \delta_j(f_{t+1}) \leq \frac{2}{N}\varsigma^{s-t},
        \label{eq:delta_f_t_bound}
    \end{equation}
    where $\varsigma = |\eta|/(1-|\xi|)$ and the last inequality follows by unraveling the recursion with the established base case. Since $|\eta| + |\xi| < 1$, we have $\varsigma < 1$. Therefore, combining \eqref{eq:delta_f_t_bound} with \eqref{eq:d_bound} and \eqref{eq:delta_bound}, we can use F\"ollmer's covariance estimate (Lemma \ref{lem:foll_cov_est}) with $f = f_t$ and $g = M_\vartheta^{(t)}$, which gives
    \begin{equation*}
        \bigg|\cov\bigg(\frac{1}{N}\mathbb{E}\Big[\sum_{i=1}^N x_i^{(s)}|\x^{(t)}\Big], 
                        M_\vartheta^{(t)} \Big|\x^{(t-1)}\bigg)\bigg|
        \leq \frac{1}{4}\sum_{i=1}^N\sum_{k=1}^N \delta_i(f_t)D_{ik}\delta_k(M_{\vartheta}^{(t)}) \leq \frac{\varsigma^{s-t}}{1 - |\xi|}.
    \end{equation*}
    Substituting into \eqref{eq:gradient_decomp} and summing the geometric series gives
    \begin{multline*}
        \left|\frac{\partial M_\z(\theta)}{\partial \vartheta}\right| \leq \frac{1}{T}\sum_{s=1}^T \sum_{t=1}^s \mathbb{E}\Big[
            \big|\cov\bigg(\frac{1}{N}\sum_{i=1}^N x_i^{(s)}, 
                        M_\vartheta^{(t)} \Big|\x^{(t-1)}\bigg)\big|
        \Big] \leq \frac{1}{T}\sum_{s=1}^T \sum_{t=1}^s \frac{\varsigma^{s-t}}{1 - |\xi|}  \\
        \leq \frac{1}{(1-\varsigma)(1-|\xi|)} \leq \kappa^{-1}.
    \end{multline*}
    Recall, this holds for all $\vartheta \in (\beta, \xi, \eta)$. 
    
    \textit{Step 3: Repeating for $\alpha_i^{(t)}$.} Turning now to the low-rank external field, we can decompose
    \begin{equation*}
        \frac{\partial M_\z(\theta)}{\partial \alpha_i^{(t)}} = \frac{1}{T}\sum_{s=t}^T \mathbb{E}\left[\mathrm{Cov}\left(\frac{1}{N}\sum_{j=1}^N x_j^{(s)}, x_i^{(t)} \Big| \x^{(t-1)}\right)\right]
    \end{equation*}
    similarly to \eqref{eq:gradient_decomp} and by calculating
    \begin{equation*}
         \frac{\partial}{\partial \alpha_i^{(t)}} \log p(\x^{(t')}|\x^{(t'-1)}) = 
         \begin{cases} x_i^{(t)} -\mathbb{E}[x_i^{(t)}|\x^{(t-1)}]&\text{if } t=t' \\ 0&\text{otherwise}.\end{cases}
    \end{equation*}
    Repeating Step 2 with $g = x_i^{(t)}$ and $\delta_j(g) = 2(\mathbb{I}\{j = i\})$ gives
    \begin{equation*}
        \left|\frac{\partial M_\z(\theta)}{\partial \alpha_i^{(t)}}\right| \leq \frac{\kappa^{-1}}{NT}.
    \end{equation*}

    \textit{Step 4: Completing the Argument.}
    We have thus shown that $|\frac{\partial M_\z(\theta)}{\partial \vartheta}| \leq \kappa^{-1}$ for $\vartheta \in (\beta, \xi, \eta)$ and $|\frac{\partial M_\z(\theta)}{\partial \alpha_i^{(t)}}| \leq \kappa^{-1}/NT$. Returning to our expression of the difference in average treatment estimates in \eqref{eq:ate_diff} and using Cauchy-Schwarz,
    \begin{multline*}
        |M_{\z}(\theta) - M_{\z}(\theta')| \leq \frac{\kappa^{-1}}{NT}\sum_{i=1}^N \sum_{t=1}^T |\alpha_i^{(t)} - {\alpha'}_i^{(t)}| + \kappa^{-1}\sum_{\vartheta \in (\beta,\xi,\eta)} |\vartheta - \vartheta'|
        \\ \leq \kappa^{-1}\sqrt{\frac{1}{NT} \sum_{i=1}^N \sum_{t=1}^T (\alpha_i^{(t)} - {\alpha'}_i^{(t)})^2}  + \kappa^{-1}\sum_{\vartheta \in (\beta,\xi,\eta)} \sqrt{(\vartheta - \vartheta')^2}.
    \end{multline*}
    Thus,
    \begin{equation*}
        (M_{\z}(\theta) - M_{\z}(\theta'))^2 \leq c\kappa^{-2}\frac{\|A-A'\|_F^2}{NT} + c\kappa^{-2}\sum_{\vartheta \in (\beta,\xi,\eta)} (\vartheta - \vartheta')^2 \leq c\kappa^{-2}\|\theta-\theta'\|_\star.
    \end{equation*}
\end{proof}

\section{Experimental Settings} \label{app:exp_settings}

\subsection{Synthetic Setting} We detail the synthetic settings of Section \ref{subsec:synth}. As discussed, we set $N=500, T=50, k=3$. 
To generate interventions, we sample $W \in \mathbb{R}^{N \times k}$ and $L \in \mathbb{R}^{T \times k}$ with standard Gaussian entries and define $I = W \Sigma L^\top$, $\Sigma = \mathrm{Diag}(1.0, 0.7, 0.49)$. We normalize $I \in [0,1]^{N\times T}$ and sample each $z_i^{(t)}$ independently with $\mathbb{P}(z_i^{(t)} = 1) = [I]_{it}$. 
We define $A^\ast = W \Sigma' L^\top$, $\Sigma' = \mathrm{Diag}(1.0, 0.8, 0.6)$ and normalize so that $\|A^\ast\|_F/\sqrt{NT}=0.75$. Since $W$ and $L$ are shared between $I$ and $A^\ast$, there exists latent confounding.
We let $\Gamma$ be Erdos-Renyi with $p=0.01$, normalized so $\|\Gamma\|_\infty = 1$. Values of $\beta^\ast$, $\xi^\ast$, and $\eta^\ast$ vary. To generate data, we sample $\x^{(0)}$ with $\mathbb{P}(x_i^{(0)} = 1) = 0.5$ independently across $i \in [N]$. Then, we use sequential Gibbs sampling with $B=100$ local updates for each $i\in[N], t\in [T]$ to generate $\x | \z$. Sequential Gibbs mixes fast since $\xi^\ast < 1$.

\subsection{Hybrid Setting} 

We turn to the hybrid setting of Section \ref{subsec:real-world-results}. As discussed, we take $\z$ and $\Gamma$ from the real-world data.
We do singular value decomposition on $Z \in \mathbb{R}^{N\times T}$ and keep the features associated with the top 1 singular value (77\% of energy). Then, letting $k=5$, we construct $W \in \mathbb{R}^{N \times k}$ and $L \in \mathbb{R}^{T \times k}$ where the first columns are the features identified from $Z$ and the other columns contain random Gaussian entries. We define $A^\ast = W \Sigma L^\top$, $\Sigma = \mathrm{Diag}(1.0,0.9,0.9,0.7,0.6)$ and normalize so $\|A^\ast\|_F / \sqrt{NT} = 0.4$. 
This represents low rank biases with partial confounding.
We again fix $\beta^\ast=-0.3$, $\xi^\ast = 0.8$, and $\eta^\ast=0.3$ and use sequential Gibbs with $B=100$ to produce a single sample from the model. 
Our solution approach is unchanged from the synthetic experiments except we utilize only $t \geq 50$ to calculate the gradient of $\beta$ (since no county is vaccinated before $t=50$).

\section{Supplemental Experimental Results} \label{app:supp_results}

\setcounter{figure}{0}
\renewcommand{\thefigure}{G\arabic{figure}} 
\setcounter{table}{0}
\renewcommand{\thetable}{G\arabic{table}} 

\subsection{Performance under No Interference} \label{app:synth_no_interf}

We conduct a further synthetic experiment to investigate the performance of $\hat{\theta}$ when the underlying distribution satisfies no interference. Our experimental setting is identical to Section \ref{subsec:synth} except for fixing $\xi^\ast = 0$. Table \ref{tab:parameter_gate_recovery_no_interference} summarizes the parameter and causal effect estimation results. Our results show that $\hat{\theta}$ is strictly more general than $\hat{\theta}_{\xi=0}$ since it correctly identifies the lack of interference. 
\begin{table}[hbt]
    \centering
    \begin{tabular}{l|ccccc}
        & $\beta$ & $\xi$ & $\eta$ & $A$ RMSE & $\GTE({\bf 1},{\bf -1})$ \\
        \hline

        $\theta^\ast$
            & $-0.300$
            & $0.000$
            & $0.300$
            & $0.000$
            & $-0.534 \pm 0.002$ \\

        $\hat{\theta}$
            & $-0.280 \pm 0.003$
            & $-0.001 \pm 0.020$
            & $0.296 \pm 0.001$
            & $0.320 \pm 0.003$
            & $-0.516 \pm 0.006$ \\

        $\hat{\theta}_{\xi=0}$
            & $-0.280 \pm 0.003$
            & $0.000 \pm 0.000$
            & $0.296 \pm 0.001$
            & $0.320 \pm 0.003$
            & $-0.515 \pm 0.005$ \\
            
        $\hat{\theta}_{A=0}$
            & $-0.155 \pm 0.004$
            & $-0.005 \pm 0.030$
            & $0.225 \pm 0.005$
            & $0.750 \pm 0.000$
            & $-0.371 \pm 0.010$ \\

        \hline
    \end{tabular}
    \caption{Parameter and $\GTE$ recovery. Entries report means across 10 trials with standard error. For the latent field, the reported value is the root mean square error (RMSE). We estimate $\GTE$ in each trial by averaging eight trajectories each generated with $B=100$ local updates.}
    \label{tab:parameter_gate_recovery_no_interference}
\end{table}

\subsection{Supplement to Real-World Case Study}

The following section contains supplemental results discussed in Section \ref{subsec:real-world-results}.

\textbf{Interventional Distribution.} Figure \ref{fig:app_interv} plots the percentage of US counties above the 30\% threshold over time. The first county crosses the threshold at time $t=50$ and by $t=90$ more than 90\% of counties observe the intervention. 
\begin{figure}[hbt]
    \centering
    \includegraphics[width=0.7\linewidth]{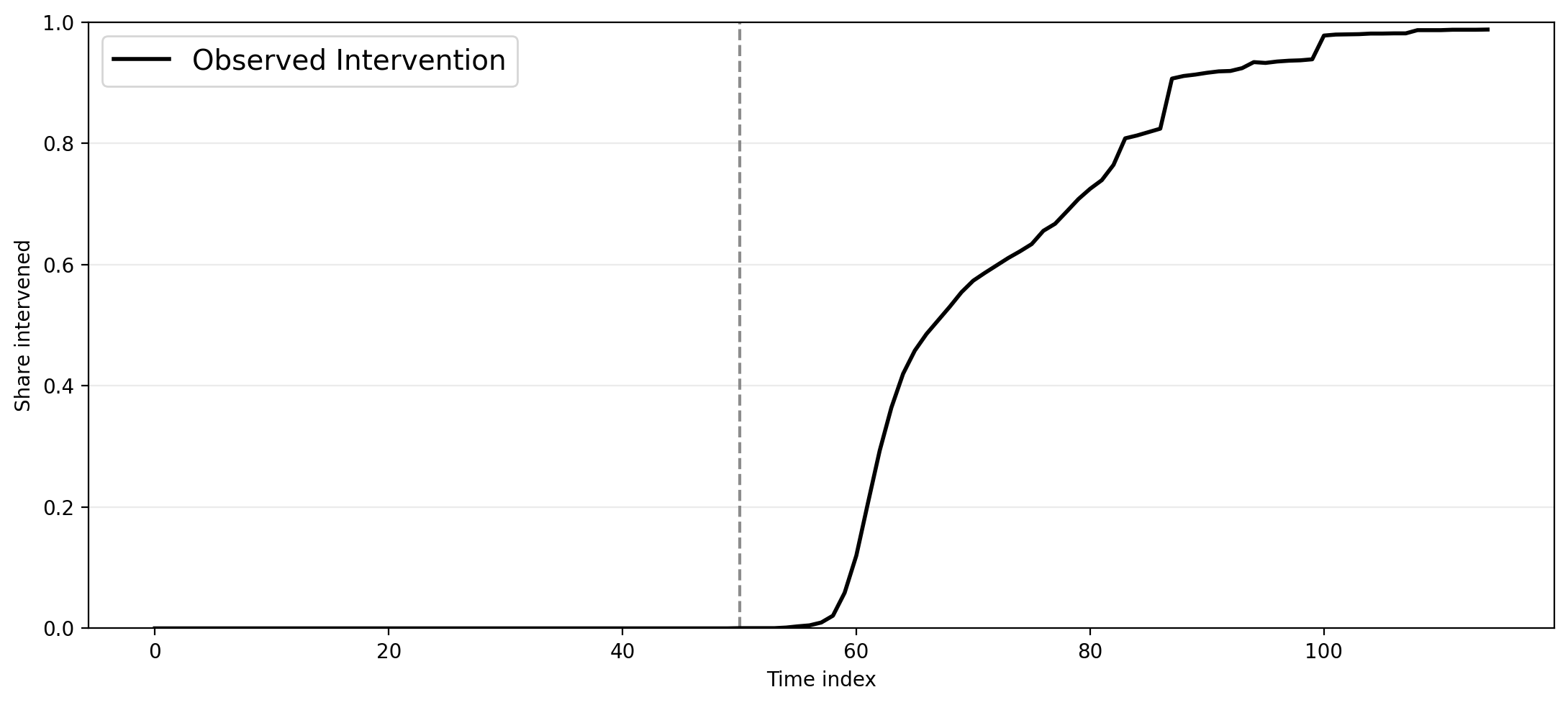}
    \caption{The percentage of US counties above the 30\% vaccination threshold over time.}
    \label{fig:app_interv}
\end{figure}

\textbf{Test Set Recovery.} 

\hspace{1em} \textit{Test Set Construction.}
Figure \ref{fig:test-set-viz} offers a visualization of our test set construction on the grid graph as described in Section \ref{subsec:real-world-results}. We partition the units (y-axis) and time (x-axis) into 6 sets. Nodes in the test set are shaded black. Nodes that are conditioned on to prevent data leakage due to interference are shaded gray; these are nodes that are directly connected spatio-temporally to a node in the test set. Nodes in the training set are white. Appendix \ref{app:cv_details} discusses construction of cross-validation sets, which is by a similar procedure. Note that in the presence of the unseen test set, cross-validation for hyperparameter selection is performed by further partitioning the training set (see Appendix \ref{app:cv_details}).
\begin{figure}[hbt]
    \centering
    \includegraphics[width=0.5\linewidth]{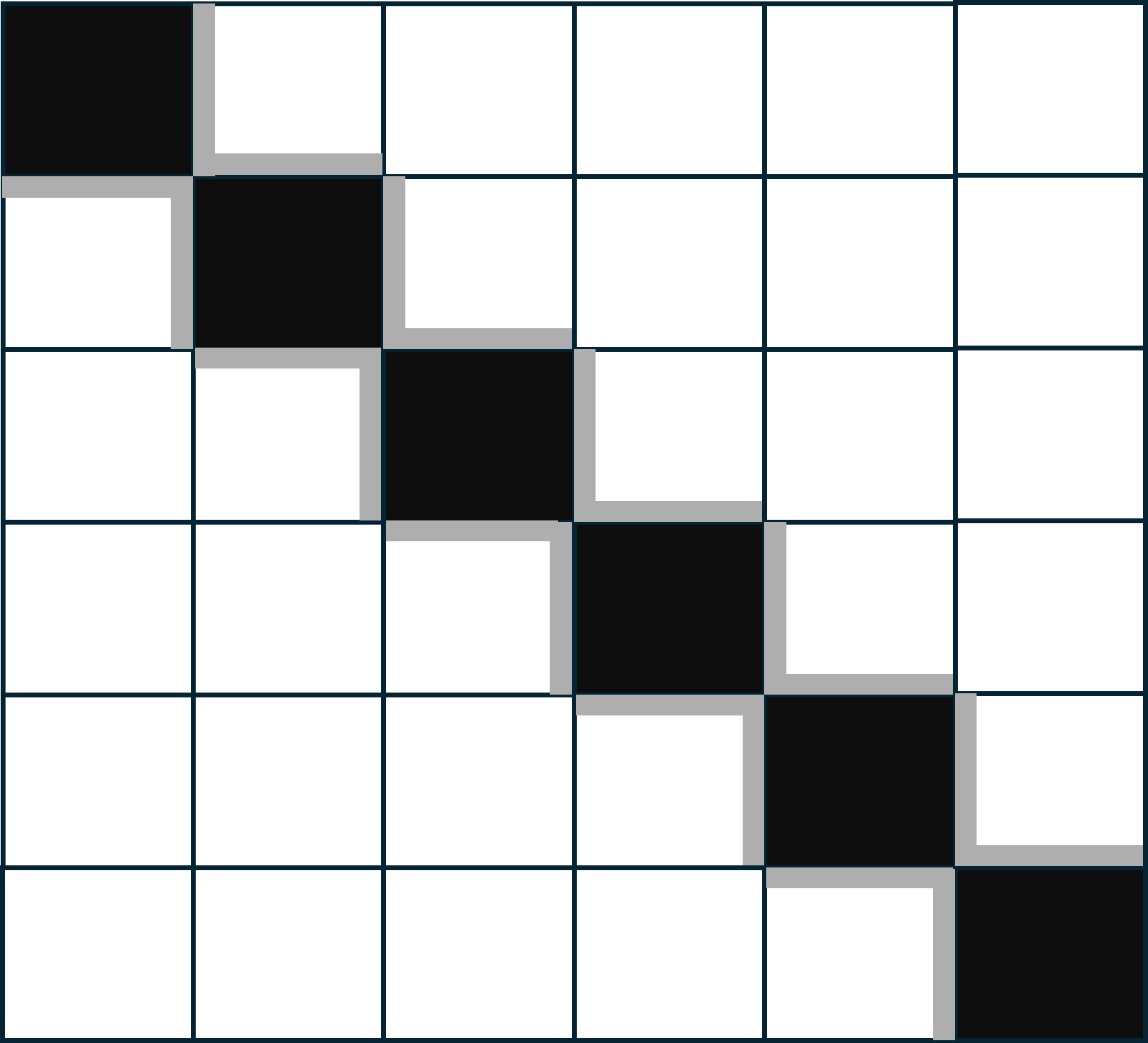}
    \caption{A visualization of our procedure for constructing a held-out test set. We assume the graph is a grid graph so nodes are connected contiguously. Nodes shaded black, gray, and white comprise the test set, the conditioned upon ``separator set'', and the training set, respectively.}
    \label{fig:test-set-viz}
\end{figure}

\hspace{1em} \textit{Marginal Brier Score.} In addition to reporting the absolute error in average outcome estimation, the test set recovery experiments also report the marginal Brier score of the estimators \cite{brier1950verification}. The Brier score over set $S$ is defined as
\begin{equation}
    \frac{1}{|S|}\sum_{(i,t): x_i^{(t)} \in S} \left(\mathbb{P}(x_i^{(t)}=1|\z) - \frac{x_i^{(t)}+1}{2}\right)^2,
    \label{eq:marginal_brier}
\end{equation}
where $\mathbb{P}(x_i^{(t)}=1|\z)$ is the prediction of the estimator and we have accounted for our $\{-1,1\}$ scale. For $\hat{\theta}$ and $\hat{\theta}_{\xi=0}$, we calculate $\mathbb{P}(x_i^{(t)}=1|\z)$ by averaging the 32 sampled trajectories.

\hspace{1em} \textit{Validating MCMC Mixing.} Finally, to validate that sequential Gibbs with $B=100$ is sufficiently mixed for $\hat{\theta}$ fit on COVID-19 data, we re-compute the statistics of Table \ref{tab:test_statistics} using 8 MCMC samples for $B=100$ and $B=500$ local updates.
Table \ref{tab:app_test_stats_B_comp} compares the results for $\hat{\theta}$ and $\hat{\theta}_{\xi=0}$. 
Since the latter satisfies no interference, the sequential Gibbs approach is equivalent to sampling independent Bernoullis and mixes immediately. We expect and observe negligible differences between inference with $B=100$ and $B=500$. We observe similarly negligible differences in estimates for $\hat{\theta}$ at $B=100$ and $B=500$, suggesting that $B=100$ suffices to sample from the Ising model conditional distribution.

\begin{table}[hbt]
    \centering
    \begin{tabular}{l|l|cccccc}
       & B & Train AE & Test AE & Train AE ($t \geq 50$) & Test AE ($t \geq 50$) \\ \hline 
       
        \multirow{2}{*}{$\hat{\theta}$}
        & 100
        & $0.005$
        & $0.015$
        & $0.001$
        & $0.033$
        \\
        & 500
        & $0.004$
        & $0.013$
        & $0.001$
        & $0.034$
        \\ \hline
       
        \multirow{2}{*}{$\hat{\theta}_{\xi=0}$}
        & 100
        & $0.005$
        & $0.011$
        & $0.004$
        & $0.050$
        \\
        & 500
        & $0.005$
        & $0.011$
        & $0.003$
        & $0.047$
       \\
       \hline
    \end{tabular}
    \caption{Absolute error for average outcome estimation for real world data when $B=100$ and $B=500$. Results average 8 MCMC samples.}
    \label{tab:app_test_stats_B_comp}
\end{table}

\textbf{Estimating Causal Effect of COVID-19 Vaccination.} 

\hspace{1em} \textit{Sensitivity Analysis.}
The hyperparameters chosen by our cross-validation approach (see Appendix \ref{app:cv_details}) are $(k,\lambda) = (5,\,0.05)$ for $\hat{\theta}$ and $(k,\lambda) = (5,\,0.01)$ for $\hat{\theta}_{\xi=0}$. To ensure that the conclusions of Table \ref{tab:caus_comparison} are robust to these selections, we calculate $\GTE(\z_1,\z_0)$ for local perturbations of $k$ and $\lambda$ for $\hat{\theta}$ and $\hat{\theta}_{\xi=0}$ in Table \ref{app:sensitivity_analysis}. Results associated with actual choice of $(k,\lambda)$ are bolded.

\begin{table}[hbt]
    \centering
    \begin{tabular}{l|c|c}
       & $(k,\lambda)$ & $\GTE(\z_1,\z_0)$ \\ \hline 
       
        \multirow{5}{*}{$\hat{\theta}$}
        & $\mathbf{(5,\,0.05)}$
        & $\mathbf{-0.106}$
        \\
        & $(3,\,0.05)$
        & $-0.128$
        \\
        & $(8,\,0.05)$
        & $-0.085$
        \\
        & $(5,\,0.01)$
        & $-0.111$
        \\
        & $(5,\,0.1)$
        & $-0.077$
        \\ \hline
       
        \multirow{5}{*}{$\hat{\theta}_{\xi=0}$}
        & $\mathbf{(5,\,0.01)}$
        & $\mathbf{-0.020}$
        \\
        & $(3,\,0.01)$
        & $-0.047$
        \\
        & $(8,\,0.01)$
        & $-0.031$
        \\
        & $(5,\,0.005)$
        & $-0.014$
        \\
        & $(5,\,0.05)$
        & $-0.042$
        \\ \hline
    \end{tabular}
    \caption{$\GTE(\z_1,\z_0)$ for counterfactual scenarios $\z_0$=`No Interventions' and $\z_1$=`All Interventions' for different hyperparameter combinations. Estimates average 8 samples generated using sequential Gibbs with $B=100$ updates.}
    \label{app:sensitivity_analysis}
\end{table}

\hspace{1em} \textit{Additional Visualization of Counterfactual Outcomes.}
In addition to the evolution of average outcome per time-step under counterfactual scenarios plotted by Figure \ref{fig:caus_comparison}, we also compare the average outcome per-unit. Figure \ref{fig:app_unit_comparison} visualizes the average outcomes per-unit over all time steps for the `no interventions' and `all interventions' policies under $\hat{\theta}$ and $\hat{\theta}_{\xi=0}$ and compares to observed data.
The results reinforce that modeling interference estimates a larger causal effect of the COVID-19 vaccine. Results use the original $(k,\lambda)$ chosen via cross-validation as in Figure \ref{fig:caus_comparison} and Table \ref{tab:caus_comparison}.

\begin{figure}[hbt]
    \centering
    \begin{subfigure}{0.48\linewidth}
        \centering
        \includegraphics[width=\linewidth]{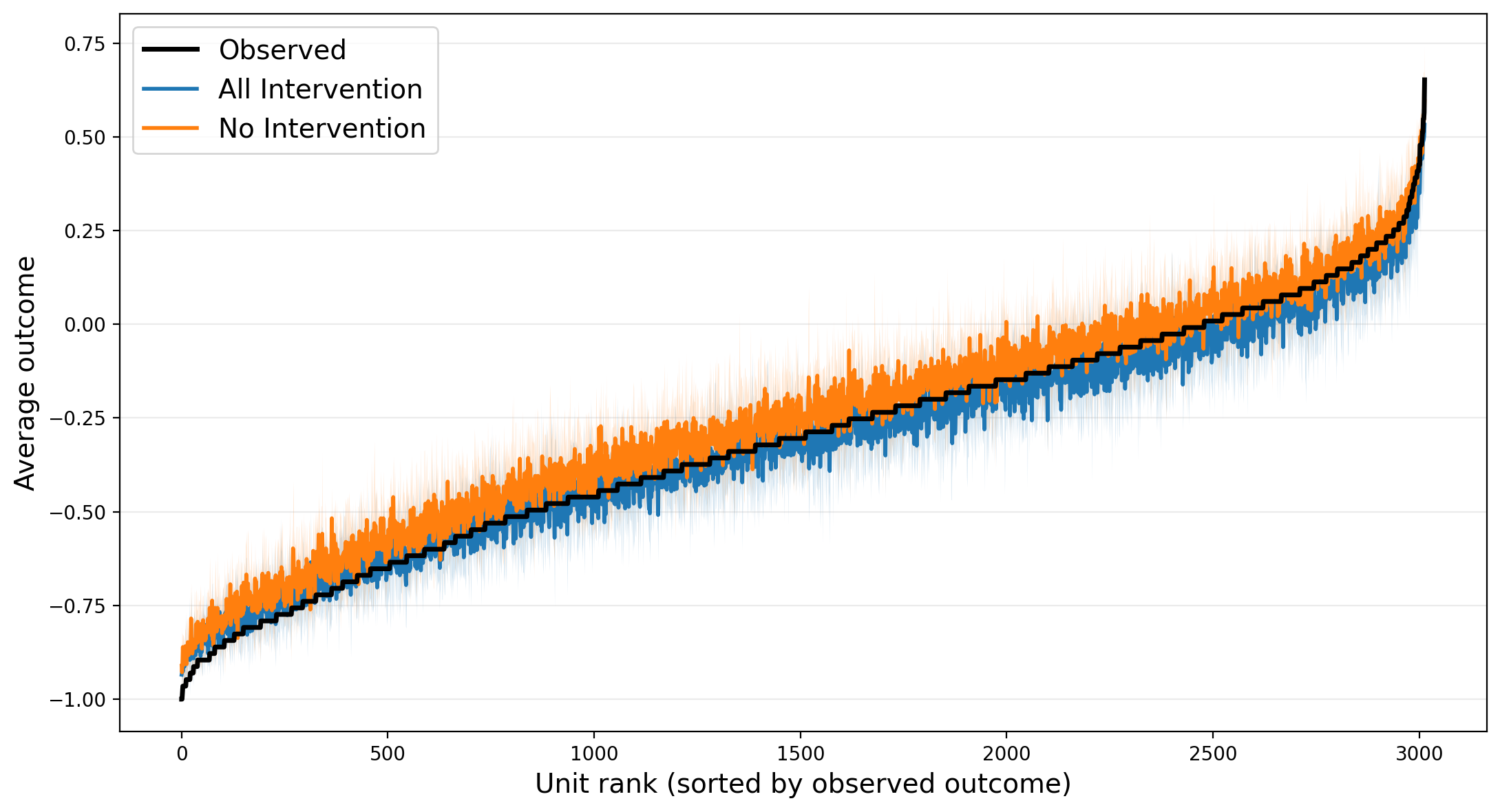}
        \caption{Interference.}
    \end{subfigure}
    \hfill
    \begin{subfigure}{0.48\linewidth}
        \centering
        \includegraphics[width=\linewidth]{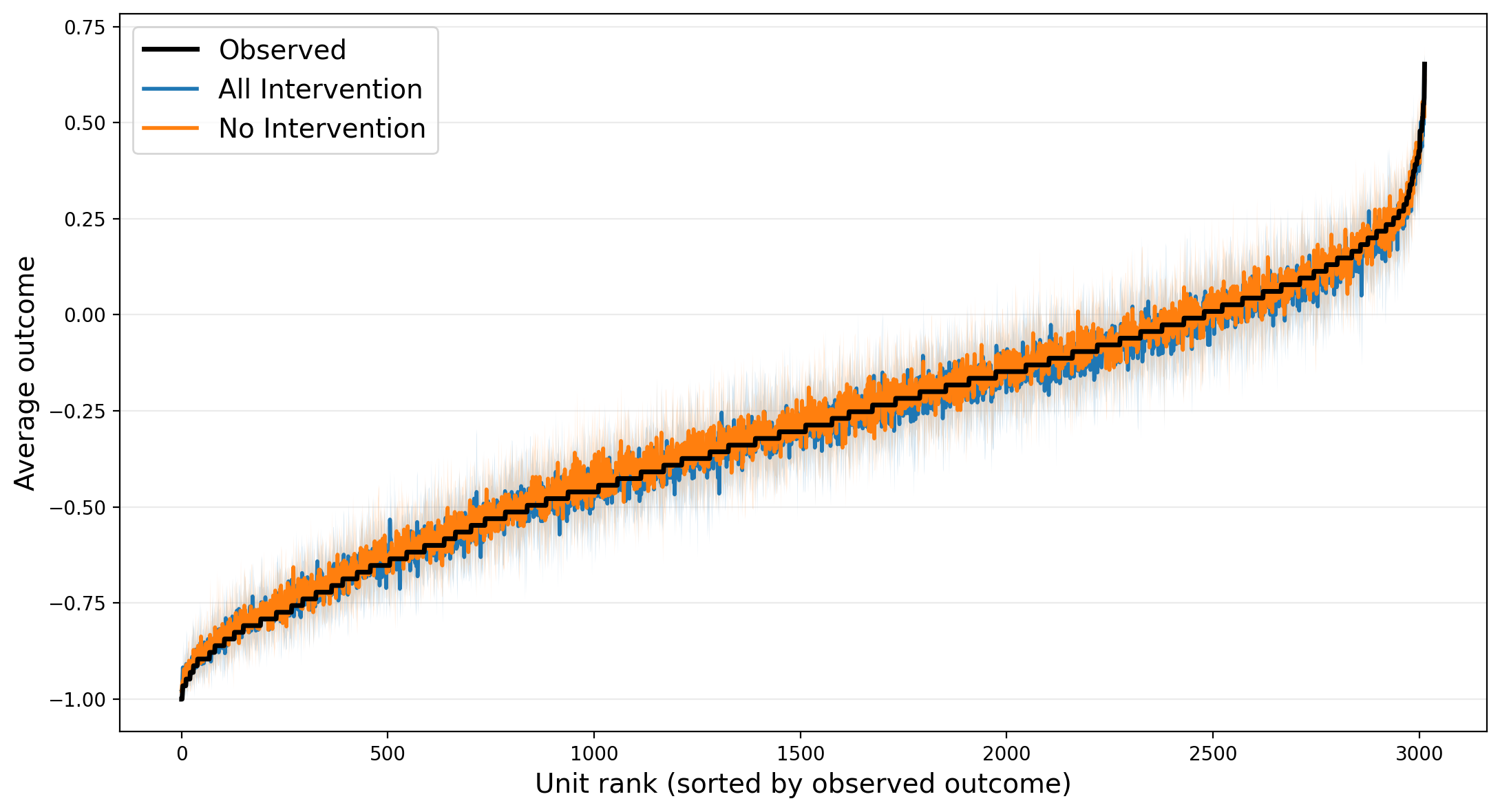}
        \caption{No Interference.}
    \end{subfigure}
    \caption{Average outcomes per-unit under counterfactual scenarios as estimated by $\hat{\theta}$ (interference) and $\hat{\theta}_{\xi=0}$ (no interference).}
    \label{fig:app_unit_comparison}
\end{figure}

\hspace{1em} \textit{Recreating Outcomes under Observed Policy.} As further validation, we also simulate outcomes under the observed policy and compare the estimated per-time and per-unit averages with the observed data for $\hat{\theta}$ and $\hat{\theta}_{\xi=0}$. Figure \ref{fig:app_observed_recreation} reinforces that the model is sufficiently expressive to capture COVID-19 outcomes since the per-time and per-unit outcomes closely match the data.

\begin{figure}[hbt]
    \centering
    \begin{subfigure}{0.48\linewidth}
        \centering
        \includegraphics[width=\linewidth]{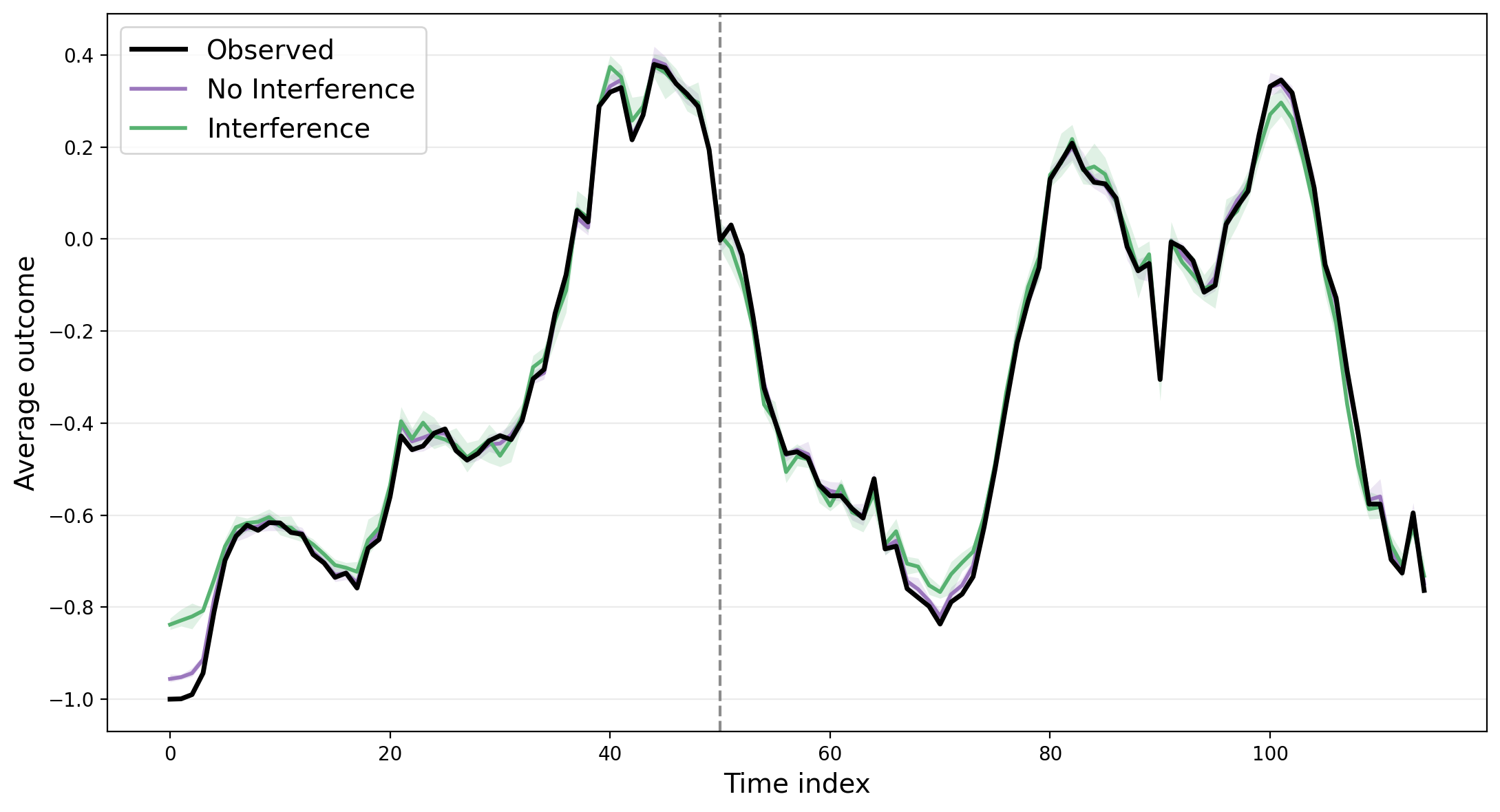}
    \end{subfigure}
    \hfill
    \begin{subfigure}{0.48\linewidth}
        \centering
        \includegraphics[width=\linewidth]{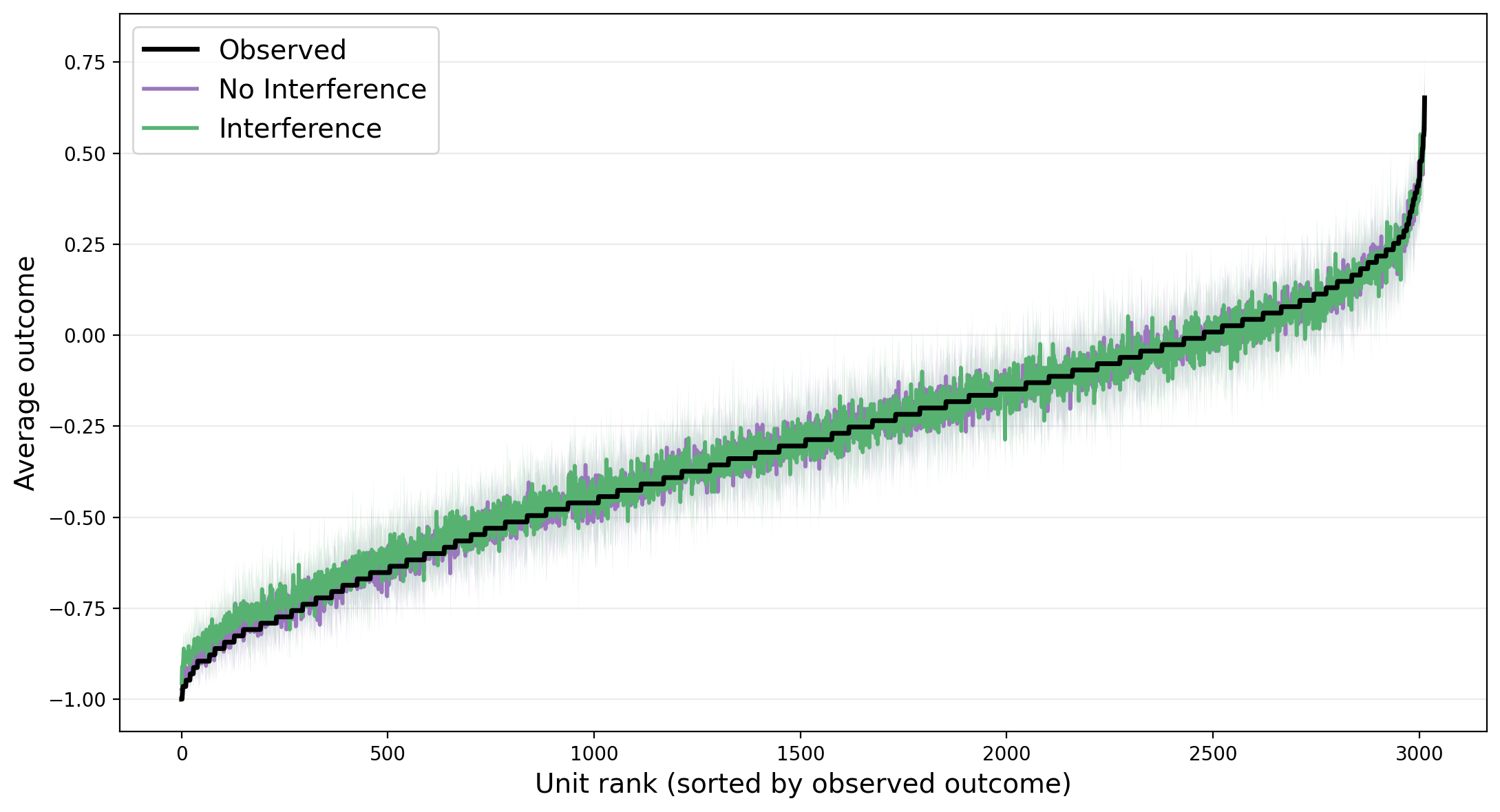}
    \end{subfigure}
    \caption{Estimated average outcomes per-time and per-unit under the observed intervention compared to the observed data. Both $\hat{\theta}$ and $\hat{\theta}_{\xi=0}$ closely follow the observed data.}
    \label{fig:app_observed_recreation}
\end{figure}

\section{Experimental Details} \label{app:exp_details}

\subsection{Cross-Validation} \label{app:cv_details}

We use a cross-validation approach for selecting hyperparameters by creating validation sets similar to our construction of the test set. Figure \ref{fig:test-set-viz} offers a visualization of our construction if we were to partition into five sets (whereas we use 6 for test set construction and 7 for cross validation).

\subsubsection{Set Construction}
Denote by $\mathcal{T} = \{x_i^{(t)}: i \in [N], t\in [T]\}$ the set of all outcomes. There are two construction approaches, based on whether $\mathcal{T}_{train} = \mathcal{T}$. For the synthetic and hybrid experiments and when creating counterfactual estimates on real world data, we train on the full data. For the test set recovery experiments, $\mathcal{T}_{train} \subset \mathcal{T}$.

If $\mathcal{T}_{train} = \mathcal{T}$, we partition the nodes and time horizon into $7$ subsets $C_1,\dots,C_7$ and $T_1,\dots,T_7$. We choose the 7 validation sets as
\begin{equation*}
    \mathcal{V}_k = \bigcup_{j=1}^7 \{x_i^{(t)}: i \in C_j, t \in T_{j+k}\}, \qquad k = 0,\dots, 6,
\end{equation*}
with indices wrapping around. We define by $\mathcal{S}_k$ the nodes that are directly connected spatio-temporally with a node in the validation set. That is,
\begin{equation*}
    \mathcal{S}_k
    =
    \left\{
    x_i^{(t)} \notin \mathcal{V}_k :
    x_i^{(t-1)} \in \mathcal{V}_k
    \,\,\,\, \text{or} \,\,\,\,
    x_i^{(t+1)} \in \mathcal{V}_k
     \,\,\,\, \text{or}  \,\,\,\,
    \exists j \text{  with  } \gamma_{ij} \neq 0 \text{ and } x_j^{(t)} \in \mathcal{V}_k
    \right\}.
\end{equation*}
For fold $k$, we holdout both $\mathcal{V}_k$ and $\mathcal{S}_k$ and fit on $\mathcal{F}_k = \mathcal{T}_{train} \setminus \{\mathcal{V}_k \cup \mathcal{S}_k\}$. $\mathcal{S}_k$ acts as a separator set that prevents data leakage between $\mathcal{V}_k$ and $\mathcal{F}_k$ during training and evaluation since training loss does not depend on validation outcomes.

If $\mathcal{T}_{train} \subset \mathcal{T}$, the construction procedure must also ensure independence of training set loss from test set outcomes. Let $\mathcal{T}_{test}$ denote the test set,
\begin{equation*}
    \mathcal{T}_{sep} = \left\{
    x_i^{(t)} \notin \mathcal{T}_{test} :
    x_i^{(t-1)} \in \mathcal{T}_{test}
    \,\,\, \text{or} \,\,\,
    x_i^{(t+1)} \in \mathcal{T}_{test}
     \,\,\, \text{or}  \,\,\,
    \exists j \text{  with  } \gamma_{ij}\neq 0 \text{ and } x_j^{(t)} \in \mathcal{T}_{test}
    \right\},
\end{equation*}
denote the separator set for the test set, and $\mathcal{T}_{train} = \mathcal{T} \setminus \{\mathcal{T}_{test} \cup \mathcal{T}_{sep}\}$ be the training set. By construction, the nodes whose outcomes comprise the train set are not the same over time; i.e., for $t \neq r$ we may have $x_i^{(t)} \in \mathcal{T}_{train}$ but $x_i^{(r)} \notin \mathcal{T}_{train}$. Therefore, we cannot create a stable partition over the nodes in $\mathcal{T}_{train}$ to repeat the validation construction. 

Instead, we create dynamic partitions of the nodes and combine over time conservatively. We denote by $P_t = \{i \in [N]: x_i^{(t)} \in \mathcal{T}_{train}\}$ the available nodes at time $t$ and partition each $P_t$ into 7 sets $C_{t,1}, \dots, C_{t,7}$. If $P_t = P_r$, we use the same partition. In general, the test set construction is such that $P_{t} = P_{t-1}$ except for a few time steps when the set of nodes comprising the test set changes. We partition the time horizon into $T_1, \dots, T_7$ and exclude the few times $t$ for which $P_{t} \neq P_{t-1}$. Thus, $C_{t,1}, \dots, C_{t,7}$ are dynamic partitions of the nodes and $T_1, \dots, T_7$ build in a buffer when the partitions change. We define the validation sets as 
\begin{equation*}
    \mathcal{V}_k = \bigcup_{j=1}^7 \{x_i^{(t)}: i \in C_{t,j}, t \in T_{j+k}\}, \qquad k=0,\dots,6.
\end{equation*}
Clearly, $\mathcal{V}_k \subseteq \mathcal{T}_{train}$. The separator set is 
\begin{equation*}
    \mathcal{S}_k
    =
    \mathcal{T}_{train} \cap
    \left\{
    x_i^{(t)} \notin \mathcal{V}_k :
    x_i^{(t-1)} \in \mathcal{V}_k
    \,\,\,\, \text{or} \,\,\,\,
    x_i^{(t+1)} \in \mathcal{V}_k
     \,\,\,\, \text{or}  \,\,\,\,
    \exists j \text{  with  } \gamma_{ij}\neq 0 \text{ and } x_j^{(t)} \in \mathcal{V}_k
    \right\}.
\end{equation*}
Finally, we define $\mathcal{F}_k = \mathcal{T}_{train} \setminus \{\mathcal{V}_k \cup \mathcal{S}_k\}$. $\mathcal{S}_k$ ensures no leakage between $\mathcal{V}_k$ and $\mathcal{F}_k$ during training and evaluation. $\mathcal{T}_{sep}$ similarly separates the test set from loss incurred by $\mathcal{T}_{train}$.

\subsubsection{Evaluation} \label{app:subsubsec_eval}
In both cases, we run an MPLE fit for each combination of hyperparameters on each validation fold. Since $\mathcal{S}_k$ separates $\mathcal{F}_k$ and $\mathcal{V}_k$, the pseudo-likelihood of $\mathcal{F}_k$ does not contain terms from $\mathcal{V}_k$.
To evaluate performance, we conditioning on the observed outcomes of $\mathcal{S}_k$ and use sequential Gibbs (16 samples; $B=10$) to estimate the average outcome of the held out validation set. 
We select the candidate with the lowest absolute error in estimating the average outcome of $\mathcal{V}_k$ in the post-interventional period (i.e., considering only $t \geq s$ where $s$ is the first time at which a node observes the intervention). Note $s=1$ for synthetic experiments and $s=50$ for hybrid and real-world experiments. If multiple fits are within a standard error of the best fit, we select among them the candidate with the lowest conditional Brier score \cite{brier1950verification}, defined as the squared difference between the predicted probability of each outcome and its occurrence. Formally, 
\begin{equation*}
    \mathrm{CondBrier}(\theta) = \frac{1}{|\mathcal{V}_k|} \sum_{(i,t): x_i^{(t)} \in \mathcal{V}_k} \left(\mathbb{P}_\theta(x_i^{(t)}=1|\x^{(t)}_{-i},\z^{(t)},\x^{(t-1)}) - \frac{x_i^{(t)} + 1}{2}\right)^2,
\end{equation*}
where we have accounted for our $\{-1,1\}$ scale.
If multiple eligible candidates are within a standard error of the best conditional Brier score, we choose the least regularized option.

\subsubsection{Results}

Let $\Lambda = \{0.001, 0.005, 0.01, 0.05, 0.1, 0.5\}$, which we use for all searches. 
Below, we list the hyperparameters used for all results.
See the \href{https://github.com/phevosp/MPLE-Causal-Inference}{repository} for full cross-validation results. 
\begin{itemize}
    \item For the first synthetic experiment (Table \ref{tab:parameter_gate_recovery}), we fix $k=3$. After evaluation, we select $\lambda=0.05$ for both $\hat{\theta}$ and $\hat{\theta}_{\xi=0}$.
    \item For the synthetic experiment under no interference (Table \ref{tab:parameter_gate_recovery_no_interference}), we fix $k=3$. After evaluation, we select $\lambda=0.05$ for both $\hat{\theta}$ and $\hat{\theta}_{\xi=0}$.
    \item For the comparison to SNN (Table \ref{tab:gte-comparison}), we run a separate search for each configuration of $\beta^\ast$ and $\xi^\ast$. We fix $k=3$. After evaluation, we select $\lambda=0.1$ when $\beta^\ast=0.1; \xi^\ast=0.8$ and $\beta^\ast=0.2; \xi^\ast=0.0$. For all other configurations, we select $\lambda = 0.05$.
    \item For the hybrid experiment (Table \ref{tab:parameter_gate_recovery_hybrid}), we fix $k=5$. After evaluation, we select $\lambda = 0.001$ for $\hat{\theta}$ and $\lambda = 0.05$ for $\hat{\theta}_{\xi=0}$.
    \item For the test set recovery experiment for real-world data (Table \ref{tab:test_statistics}), we perform a grid search over $k \in \{3, 5, 8\}$ and $\lambda \in \Lambda$. We select $k=8$ and $\lambda=0.01$ for both $\hat{\theta}$ and $\hat{\theta}_{\xi=0}$.
    \item For the test set recovery experiment for hybrid data (Table \ref{tab:test_statistics_hybrid}), we fix $k=5$ and select $\lambda=0.005$ for both $\hat{\theta}$ and $\hat{\theta}_{\xi=0}$.
    \item For causal effect estimation on real-world data (Table \ref{tab:caus_comparison}), we re-do the grid search over $k \in \{3, 5, 8\}$ and $\lambda \in \Lambda$ using $\mathcal{T} = \mathcal{T}_{train}$. We select $k=5$ and $\lambda = 0.05$ for $\hat{\theta}$ and $k=5$ and $\lambda =0.01$ for $\hat{\theta}_{\xi=0}$.
\end{itemize}

\subsection{Computational Resources}

All data generation, optimization, and sampling are run on a cluster with multiple CPU cores. Sampling may take several hours. Results can be reproduced on a standard desktop/laptop.

\section*{NeurIPS Paper Checklist}

\begin{enumerate}

\item {\bf Claims}
    \item[] Question: Do the main claims made in the abstract and introduction accurately reflect the paper's contributions and scope?
    \item[] Answer: \answerYes{} 
    \item[] Justification: Claims made in the abstract and introduction are supported by theory (Theorems \ref{thm:param_est} and \ref{thm:causal.est}) and empirical results (Section \ref{sec:empirical_results})

\item {\bf Limitations}
    \item[] Question: Does the paper discuss the limitations of the work performed by the authors?
    \item[] Answer: \answerYes{} 
    \item[] Justification: See Section \ref{sec:limitations}.

\item {\bf Theory assumptions and proofs}
    \item[] Question: For each theoretical result, does the paper provide the full set of assumptions and a complete (and correct) proof?
    \item[] Answer: \answerYes{} 
    \item[] Justification: Theorem statements list assumptions explicitly. Proof details are in the Appendix.

    \item {\bf Experimental result reproducibility}
    \item[] Question: Does the paper fully disclose all the information needed to reproduce the main experimental results of the paper to the extent that it affects the main claims and/or conclusions of the paper (regardless of whether the code and data are provided or not)?
    \item[] Answer: \answerYes{} 
    \item[] Justification: Experimental details are discussed explicitly in Section \ref{sec:empirical_results} and Appendix \ref{app:exp_settings}. Cross validation details are discussed in Appendix \ref{app:cv_details}

\item {\bf Open access to data and code}
    \item[] Question: Does the paper provide open access to the data and code, with sufficient instructions to faithfully reproduce the main experimental results, as described in supplemental material?
    \item[] Answer: \answerYes{} 
    \item[] Justification: Code is shared via the github repository: \url{https://github.com/phevosp/MPLE-Causal-Inference}.

\item {\bf Experimental setting/details}
    \item[] Question: Does the paper specify all the training and test details (e.g., data splits, hyperparameters, how they were chosen, type of optimizer) necessary to understand the results?
    \item[] Answer: \answerYes{} 
    \item[] Justification: Details of the experimental settings, including hyperparameter choices for data generation, optimization hyperparameters, and real data sources are provided in Section \ref{sec:empirical_results}. Further detail is available in relevant Appendix sections \ref{app:exp_settings}, \ref{app:supp_results}, and \ref{app:exp_details}.

\item {\bf Experiment statistical significance}
    \item[] Question: Does the paper report error bars suitably and correctly defined or other appropriate information about the statistical significance of the experiments?
    \item[] Answer: \answerYes{} 
    \item[] Justification: Error bars are reported everywhere that is appropriate. It is the standard error, as discussed.

\item {\bf Experiments compute resources}
    \item[] Question: For each experiment, does the paper provide sufficient information on the computer resources (type of compute workers, memory, time of execution) needed to reproduce the experiments?
    \item[] Answer: \answerYes{} 
    \item[] Justification: We discuss the resources need in Appendix \ref{app:exp_details}.
    
\item {\bf Code of ethics}
    \item[] Question: Does the research conducted in the paper conform, in every respect, with the NeurIPS Code of Ethics \url{https://neurips.cc/public/EthicsGuidelines}?
    \item[] Answer: \answerYes{} 
    \item[] Justification: We conform to the NeurIPS Code of Ethics.

\item {\bf Broader impacts}
    \item[] Question: Does the paper discuss both potential positive societal impacts and negative societal impacts of the work performed?
    \item[] Answer: \answerYes{} 
    \item[] Justification: We discuss this in the conclusions section.
    
\item {\bf Safeguards}
    \item[] Question: Does the paper describe safeguards that have been put in place for responsible release of data or models that have a high risk for misuse (e.g., pre-trained language models, image generators, or scraped datasets)?
    \item[] Answer: \answerNA{} 
    \item[] Justification: We do not release new data or models with high risk of misuse.

\item {\bf Licenses for existing assets}
    \item[] Question: Are the creators or original owners of assets (e.g., code, data, models), used in the paper, properly credited and are the license and terms of use explicitly mentioned and properly respected?
    \item[] Answer: \answerYes{} 
    \item[] Justification: Data are cited. No code or models from other authors were used.

\item {\bf New assets}
    \item[] Question: Are new assets introduced in the paper well documented and is the documentation provided alongside the assets?
    \item[] Answer: \answerNA{} 
    \item[] Justification: We do not release new assets.

\item {\bf Crowdsourcing and research with human subjects}
    \item[] Question: For crowdsourcing experiments and research with human subjects, does the paper include the full text of instructions given to participants and screenshots, if applicable, as well as details about compensation (if any)? 
    \item[] Answer: \answerNA{} 
    \item[] Justification: We do not involve crowdsourcing or research with human subjects.

\item {\bf Institutional review board (IRB) approvals or equivalent for research with human subjects}
    \item[] Question: Does the paper describe potential risks incurred by study participants, whether such risks were disclosed to the subjects, and whether Institutional Review Board (IRB) approvals (or an equivalent approval/review based on the requirements of your country or institution) were obtained?
    \item[] Answer: \answerNA{} 
    \item[] Justification: We do not involve crowdsourcing or research with human subjects.

\item {\bf Declaration of LLM usage}
    \item[] Question: Does the paper describe the usage of LLMs if it is an important, original, or non-standard component of the core methods in this research? Note that if the LLM is used only for writing, editing, or formatting purposes and does \emph{not} impact the core methodology, scientific rigor, or originality of the research, declaration is not required.
    \item[] Answer: \answerNA{} 
    \item[] Justification: The core method development in this research does not involve LLMs as any important, original, or non-standard components.

\end{enumerate}

\end{document}